\documentclass[]{article}
\usepackage{fullpage}
\usepackage[utf8]{inputenc}
\usepackage[T1]{fontenc}
\usepackage{lineno}
\usepackage{amssymb}

\usepackage{soul}
\usepackage{algorithm}
\usepackage[noend]{algpseudocode}
\usepackage{graphicx}
\usepackage{fancyvrb} 
\usepackage{amsmath}

\usepackage{pgfplots}
\pgfplotsset{compat=1.7}
\usepackage[english]{babel}
\usepackage{bbm}
\usepackage{subcaption}
\usepackage{float}
\usepackage{tabularx}
\usepackage{multirow}
\usepackage{xpatch}
\usepackage{mathtools}
\usepackage{bbold} 
\usepackage{booktabs}
\usepackage{url}
\usepackage{enumitem}
\usepackage{hyperref} 
\definecolor[named]{MyLightGreen}{cmyk}{0.22,0,0.10,0.28}

\DeclareMathOperator*{\argmin}{argmin} 
 
\DeclareMathOperator{\E}{\mathbb{E}}
\DeclareMathOperator{\R}{\mathbb{R}}
\DeclareMathOperator{\1}{\mathbb{1}}

\DeclareMathOperator{\X}{\mathcal{X}}
\DeclareMathOperator{\Y}{\mathcal{Y}}
\DeclareMathOperator{\F}{\mathcal{F}}
\DeclareMathOperator{\Z}{\mathcal{Z}}
\DeclareMathOperator{\A}{\mathcal{A}}
\DeclareMathOperator{\V}{\mathcal{V}}

\newcommand{\rulesep}{\unskip\ \vrule\ }
\newcommand{\anc}{anc}
\let\emptyset\varnothing

\newcommand{\fts}{{f}_2^*}

\usepackage{amsthm}
\usepackage{thmtools}
\usepackage{thm-restate}
\newtheorem{theorem}{Theorem}

\newtheorem{proposition}[theorem]{Proposition}

\newtheorem{lemma}[theorem]{Lemma}
\newtheorem{conjecture}[theorem]{Conjecture}
\newtheorem{definition}[theorem]{Definition}

\makeatletter
\xpatchcmd{\proof}{\@addpunct{.}}{\normalfont\,\@addpunct{:}}{}{}
\makeatother

\usepackage{natbib}

\usepackage{authblk}

\title{
On the Existence of Simpler Machine Learning Models
}

\author[]{Lesia Semenova}
\author[]{Cynthia Rudin}
\author[]{Ronald Parr}
\affil[]{\{\textit{lesia, cynthia, parr}\}@cs.duke.edu}
\affil[]{Department of Computer Science, Duke University}
\date{}

\begin{document}

\maketitle
It is almost always easier to find an accurate-but-complex model than an accurate-yet-simple model. Finding optimal, sparse, accurate models of various forms (linear models with integer coefficients, decision sets, rule lists, decision trees) is generally NP-hard. We often do not know whether the search for a simpler model will be worthwhile, and thus we do not go to the trouble of searching for one. In this work, we ask an important practical question: can accurate-yet-simple models be proven to exist, or shown likely to exist, before explicitly searching for them? We hypothesize that there is an important reason that simple-yet-accurate models often do exist. This hypothesis is that the \textit{size of the Rashomon set is often large}, where the Rashomon set is the set of almost-equally-accurate models from a function class. If the Rashomon set is large, it contains numerous accurate models, and perhaps at least one of them is the simple model we desire. In this work, we formally present the \textit{Rashomon ratio} as a new gauge of simplicity for a learning problem, depending on a function class and a data set. The Rashomon ratio is the ratio of the volume of the set of \textit{accurate models} to the volume of the \textit{hypothesis space}, and it is different from standard complexity measures from statistical learning theory. 
Insight from studying the Rashomon ratio provides an easy way to check whether a simpler model might exist for a problem before finding it, namely whether several different machine learning methods achieve similar performance on the data. In that sense, the Rashomon ratio is a powerful tool for understanding why and when an accurate-yet-simple model might exist. If, as we hypothesize in this work, many real-world data sets admit large Rashomon sets, the implications are vast: it means that simple or interpretable models may often be used for high-stakes decisions without losing accuracy.

\section{Introduction}
Following the principle of Occam's Razor, one should use the simplest model that explains the data well. However, finding the simplest model, let alone any simple-yet-accurate model, is hard. As soon as simplicity constraints such as sparsity are introduced, the optimization problem for finding a simpler model typically becomes NP-hard. Thus, practitioners -- who have no assurance of finding a simpler model that achieves the performance level of a black box -- may not see a reason to attempt such potentially difficult optimization problems. Thus, sadly, what was once the holy grail of finding simpler models, has been, for the most part, abandoned in modern machine learning. In this work, we ask a question that is essential, and potentially game-changing, for this discussion: what if we knew, before attempting a computationally expensive search for a simpler-yet-accurate model, that one was likely to exist? Perhaps knowing this would allow us to justify the time and expense of searching for such a model. If it is true that many data sets have large enough Rashomon sets to admit simple models, then there are important implications for society -- it means we may be able to use simpler or interpretable models for many high-stakes problems without losing accuracy. 


Proving the existence of simpler models before aiming to find them differs from the current approach to machine learning in practice. We generally do not think about going from more complicated spaces to simpler ones; in fact, the reverse is true, where typical statistical learning theory and algorithms allowed us to maintain generalization when handling more complicated model classes (e.g., large margins for support vector machines with complex kernels or large margins for boosted trees) \cite{cortes1995support, schapire1998boosting}. We even build neural networks that are so complex that they can achieve zero training error, and try afterwards to determine why they generalize \cite{Belkin15849, nakkiran2021deep}. However, because simple models are essential for many high-stakes decisions \citep{rudin2019stop}, perhaps we should return to the goal of aiming directly for simpler models. We will need new ideas in order to do this.

Decades of study about generalization in machine learning have provided many different mathematical theories. Many of them measure the complexity of classes of functions without considering the data \citep[e.g., VC theory,][]{vapnik1995nature}, or measure properties of specific algorithms \citep[e.g., algorithmic stability, see][]{bousquet2002stability}. However, none of these theories seems to capture directly a phenomenon that occurs throughout practical machine learning. In particular, \textit{there are a vast number of data sets for which many standard machine learning algorithms perform similarly}. In these cases, the machine learning models \textit{tend to generalize well}. Furthermore, in these same cases, \textit{there is often a simpler model that performs similarly and also generalizes well}. 

We hypothesize that these three observations can all be explained by the same phenomenon: the ``Rashomon effect,''  which is the existence of many almost-equally-accurate models \citep{breiman2001statistical}. Firstly, following a key argument in our work, if there is a large \textit{Rashomon set} of almost-equally-accurate models, a simple model may also be contained in it. Secondly, if the Rashomon set is large, many different machine learning algorithms may find different but approximately-equally-well-performing models inside it. An experimenter could then observe similar performance for different types of algorithms that produce very different functions. Thirdly, if the Rashomon set is large enough to contain simpler models, those models are guaranteed to generalize well. As we will show in Section \ref{sec:exis1}, there are mathematical assumptions that allow us to \textit{prove} existence of simpler models within the Rashomon set. If the assumptions are satisfied, a model from a simpler class is approximately as accurate as the most accurate model within the hypothesis space, which consequently leads to better generalization guarantees. The assumptions are based in approximation theory, which models how one class of functions can approximate another.

We quantify the magnitude of the Rashomon effect through the 
\textit{Rashomon ratio}, which is the ratio of the Rashomon set's volume to the volume of the hypothesis space. An illustration of the Rashomon set is shown in Figure \ref{fig:rset_example}; it does not need to be a connected or convex set.
The Rashomon ratio can serve as a gauge of simplicity for a learning problem.\footnote{Such measures are typically called ``complexity'' measures, but the Rashomon ratio measures simplicity, not complexity.}
As a property of both a data set and a hypothesis space, it differs from the VC dimension \citep{vapnik1971uniform} (because the Rashomon ratio is specific to a data set),  it differs from algorithmic stability \citep[see][]{rogers1978finite,kearns1999algorithmic} (as the Rashomon ratio does not rely on robustness of an algorithm with respect to changes in the data), it differs from local Rademacher complexity \citep{bartlett2005local} (as the Rashomon ratio does not measure the ability of the hypothesis space to handle random changes in targets and actually benefits from multiple similar models), and it differs from geometric margins \citep{vapnik1995nature} (as the maximum margin classifier can have a small minimum margin yet the Rashomon ratio can be large, and margins are measured with respect to one model, whereas the Rashomon ratio considers the existence of many). We provide theorems that show simple cases when the Rashomon ratio disagrees with these complexity measures in Section \ref{section:notation} and Appendix \ref{appendix:measures}. The Rashomon set is not just functions within a flat minimum; it could consist of functions from many non-flat local minima as illustrated in  Figure \ref{fig:flatminima} in Appendix \ref{appendix:flat_minima_figure}, and it applies to discrete hypothesis spaces where gradients, and thus ``sharpness'' \citep{dinh2017sharp} do not exist. 
For linear regression, we derive a closed form solution for the volume of the Rashomon set in parameter space in Theorem \ref{th:ridge_regression} in Appendix \ref{appendix:ridge}.

Our theory and empirical results have implications beyond cases where the size of the Rashomon set can be estimated in practice: they suggest computationally inexpensive ways to gauge whether the Rashomon set is large without directly measuring it. \textit{In particular, our results indicate that when many machine learning methods perform similarly on the same data set (without overfitting), it could be because the Rashomon set of the functions these algorithms consider is large.} 
\textit{Thus, after running different machine learning methods and observing similar performance, our results indicate that it may be worthwhile to optimize directly for simpler models within the Rashomon set.} 

	\begin{figure*}[t]
		\centering
		\includegraphics[width=0.9\textwidth]{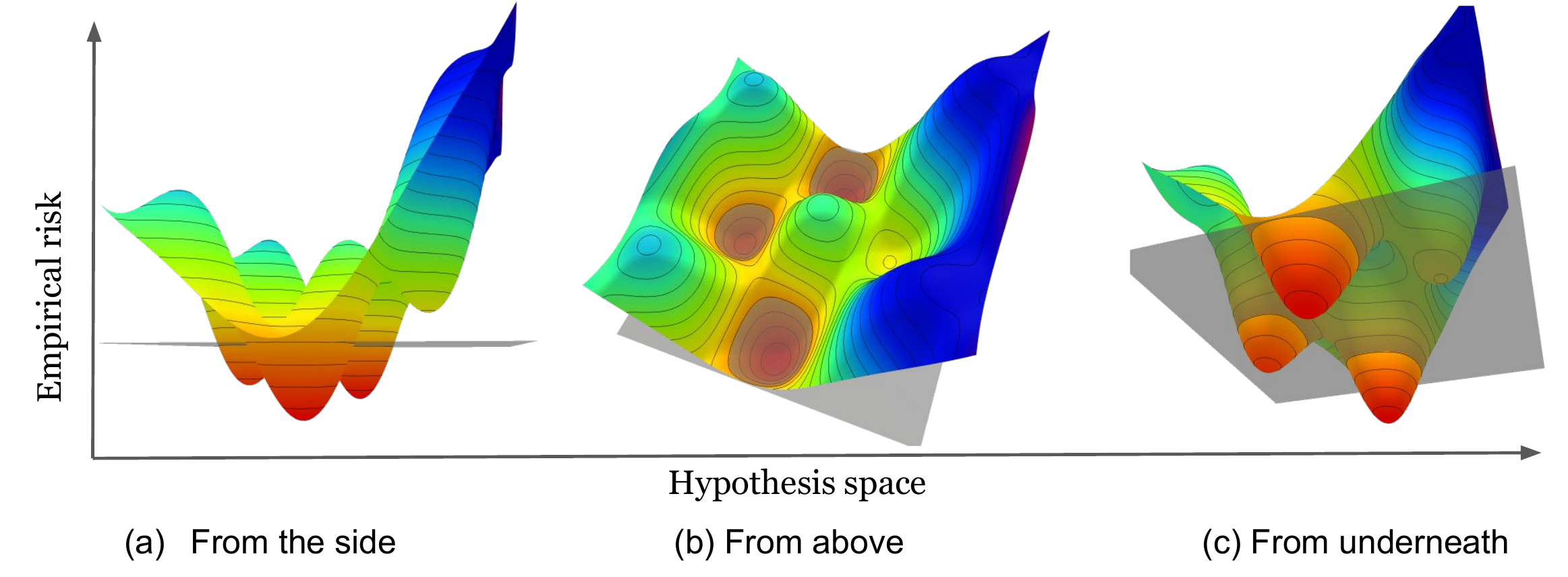}
		\caption{An illustration of a possible Rashomon set in two dimensional hypothesis space $\F$.  Models below the gray plane belong to the Rashomon set $\hat{R}_{set}(\F,\theta)$, where the height of the gray plane is adjusted by the Rashomon parameter $\theta$ defined in Section \ref{section:notation}.}
		\label{fig:rset_example}
	\end{figure*}

We summarize the contributions of this work as follows: (i) We define the \textit{Rashomon ratio} as an important characteristic of the Rashomon set. (ii) We provide generalization bounds for models from the Rashomon set, and show that the size of the Rashomon set serves as a barometer for the existence of accurate-yet-simpler models that generalize well. These are different from standard learning theory bounds that consider the distance between the true and empirical risks for the same function. (iii) We provide several approaches for estimating the size of the Rashomon set. (iv) We show empirically that when a large Rashomon set occurs, most machine learning methods tend to perform similarly, and also in these cases, simple or sparse (yet accurate) models exist. 
(v) We demonstrate that the Rashomon ratio, as a gauge of simplicity of a machine learning problem, is different from other known complexity measures such as VC-dimension, algorithmic stability, geometric margin, and Rademacher complexity. (vi) We show that larger Rashomon sets might occur in the presence of label or feature noise. 

	
\section{Related Work}

There are several bodies of relevant literature as discussed below.

    \paragraph{Rashomon sets: }
	Rashomon sets have been used for various purposes \citep{breiman2001statistical, srebro2010smoothness, fisher2018model, coker2018theory, tulabandhula2014robust,  meinshausen2010stability, LethamLeRuBr16, nevo2017identifying}. For instance, \citet{srebro2010smoothness} consider a loss-restricted class of close-to-optimal models, and with an assumption of H-smoothness of a loss function, they obtain a tighter excess risk bound through local Rademacher complexity \citep{bartlett2005local}. Our bounds do not work the same way and aim to prove a different type of result. Other works aim to search through the Rashomon set to find the most extreme models within it, rather than looking at the size of the Rashomon set, as we do in this work. \citet{fisher2018model} 
	 leverages the Rashomon set in order to understand the spectrum of variable importance and other statistics across the set of good models. Our work considers the existence of models from simpler classes rather than exploring the Rashomon set to find a range of variable importance or other statistics. The work of  \citet*{TulabandhulaRu13, TulabandhulaRu14MLJ, tulabandhula2014robust} uses the Rashomon set to assist with decision making, by finding the range of downstream operational costs associated with the Rashomon set. Rashomon sets are related to p-hacking and robustness of estimation, because the Rashomon set is a set over which one might conduct a sensitivity analysis to choices made by an analyst \citep{coker2018theory}. Large Rashomon sets can occur when the machine learning pipeline is underspecified. \citet{d2020underspecification} provides multiple examples of underspecification in computer vision, natural language processing, and healthcare domains; their work builds off of (an earlier version of) our work.
	 \citet{madras2019detecting} proposed a post-hoc local-ensemble method that measures underspecification for a given test datum. 
	 \citet{pmlr-v119-marx20a} studies conflicting predictions between models within the Rashomon set, while \citet{coston2021characterizing} investigates predictive disparities for algorithmic fairness. 

	\paragraph{Flat minima or wide valleys: } 
	The concept of flat minima (wide valleys) has been explored in the deep learning literature as a possible way to understand convergence properties of the complicated, non-convex loss functions that deep networks traverse during training \citep{hochreiter1997flat, dinh2017sharp, keskar2016large, chaudhari2016entropy}. Based on a minimum-message-length argument \citep{wallace1968information}, several works claim that flat loss functions lead to better generalization due to a robustness to noise around the minimum \citep{hochreiter1997flat, keskar2016large, chaudhari2016entropy}. Following \citet{hochreiter1997flat}, \citet{dinh2017sharp} define volume $\epsilon$-flatness, which constitutes a special case of our Rashomon sets, as shown in  Figure \ref{fig:flatminima} in Appendix \ref{appendix:flat_minima_figure}. In particular, the Rashomon set is defined over the
    hypothesis (functional) space, 
	while the volume $\epsilon$-flatness is defined in a parameter space (though sometimes we use parameter space for ease of computation), and the Rashomon set is not necessarily a single connected component (although it might be in the case of a convex loss over a continuous domain), while volume $\epsilon$-flatness pertains only to a connected set. This means that the Rashomon set can contain models from different local minima, or can be defined on discrete spaces, while volume  $\epsilon$-flatness is relevant only for continuous loss functions. Another way of quantifying flatness is $\sigma$-sharpness \citep{keskar2016large, dinh2017sharp}, which measures the change of the loss function inside a $\sigma$-ball in a parameter space.  In the case of a connected Rashomon set, this loss difference corresponds to the Rashomon parameter $\theta$. 
	
	\paragraph{Statistical learning theory: }
	Numerous works provide generalization bounds based on different complexity measures, and under different assumptions. Some discuss Rademacher \citep{srebro2010smoothness, kakade2009complexity} and Gaussian complexities \citep{kakade2009complexity}, PAC-Bayes theorems \citep{langford2003pac}, covering numbers bounds \citep{zhou2002covering}, and margin bounds \citep{vapnik1971uniform,schapire1998boosting, koltchinskii2002empirical}. In contrast, under assumptions elaborated in Section \ref{section:generalization}, the Rashomon ratio provides a certificate of the existence of a simpler model that generalizes. The use of approximating sets, as used extensively in this paper, is used throughout the literature on learning theory \citep{lecue2011interplay, lugosi2004complexity, schapire1998boosting, mendelson2003few}. An example of this is the classical generalization bound for boosting and margins   \citep{schapire1998boosting}, which uses combinations of several random draws of base classifiers to represent combinations of base classifiers. This is an instance of the so-called ``Maurey's lemma,'' which provides an approximating set for linear model classes.




\section{Rashomon Set Definitions and Notation}\label{section:notation}

	Consider a training set of $n$ data points $S=\{z_1, z_2, ..., z_n\}$,  $z_i = (x_i, y_i)$ drawn i.i.d$.$ from an unknown distribution $\mathcal{D}$ on a bounded set $\Z = \X \times \Y$, where $\X \subset \R^p$ and $\Y \subset \R$  are  an input  and an output space respectively. Our hypothesis space is $\F = \{f: \X \rightarrow \Y\}$. We limit the hypothesis space $\F$ to contain only models that vary within the bounded domain $\Z$ where the data reside. We will assume that the hypothesis space is bounded and that there is a prior distribution $\rho$ over functions in $\mathcal{F}$. 
	To measure the quality of a prediction made by a hypothesis, we use a loss function $\phi: \Y \times \Y \rightarrow \R^+$. Specifically, for each given point $z = (x,y)$ and a hypothesis $f$, the loss function is $\phi(f(x), y)$.  For a given $f$ we will also overload notation by writing $l: \F \times \Z \rightarrow \R^+$ that takes $f$ explicitly as an argument: $l(f, z ) = \phi(f(x), y)$. We are interested in learning a model $f$ that minimizes the \textit{true risk}
	$L(f) = \E_{z \sim \mathcal{D}} [\phi (f(x), y)],$
	which depends on unknown distribution $\mathcal{D}$ and therefore is estimated with an \textit{empirical risk}:
	$\hat{L}(f) = \frac{1}{n} \sum_{i=1}^n \phi (f(x_i), y_i).$

	The \textit{empirical Rashomon set} (or simply \textit{Rashomon set}) is a subset of models of the hypothesis space $\F$ that have training performance close to the best model in the class, according to a loss function \citep{breiman2001statistical, srebro2010smoothness, fisher2018model, coker2018theory, tulabandhula2014robust}. More precisely: 
	\begin{definition}[Rashomon set]\label{def:rset} 
		Given $\theta \geq 0$, a data set $S$, a hypothesis space $\F$, and a loss function $\phi$, the \textit{Rashomon set} $\hat{R}_{set}(\F, \theta)$ is the subspace of the hypothesis space defined as follows:
		\begin{equation*}\label{eq:rset_definition}
		\hat{R}_{set}(\F, \theta) := \{f \in \F: \hat{L}(f) \leq \hat{L}(\hat{f}) + \theta\},
		\end{equation*}
		where $\hat{f}$ is an empirical risk minimizer for the training data $S$ with respect to loss function $\phi$:
		$\hat{f} \in \argmin_{f \in \F} \hat{L}(f).$
	\end{definition}
	 If we want to specify the data set $S$ that is used to compute the Rashomon set, we indicate the data set in the subscript, as $\hat{R}_{{set}_{S}}(\F, \theta)$. \citet{fisher2018model}'s definition of Rashomon set is distinct from ours in that we typically use an empirical risk minimizer to define the Rashomon set instead of a prespecified reference model which is independent of the sample.

    The hypothesis space $\F$ 
    can be a well-defined hypothesis space, such as the space of decision trees of depth $D$ or neural nets with $D$ hidden layers, or it can be a more general space (a meta-hypothesis space) that contains models from different hypothesis spaces (e.g., linear functions, polynomials up to degree $D$, and piecewise constant functions). 

	We call $\theta$  the \textit{Rashomon parameter}. Since hypothesis spaces can vary from one problem to another, we will often normalize the size of the Rashomom set via the \textit{Rashomon ratio} $\hat{R}_{ratio}(\hat{R}_{set}(\F, \theta))$ which takes the Rashomon set as input and outputs a value between 0 and 1. Given a prior, $\rho$, on the hypothesis space, the Rashomon ratio measures the fraction of the hypothesis space contained in the Rashomon set. Unless explicitly specified, $\rho$ is assumed to be uniform. 
	For simplicity, we will denote the Rashomon ratio as $\hat{R}_{ratio}(\F, \theta)$. In general, the Rashomon ratio is $\hat{R}_{ratio}(\F, \theta) = \int_{f\in\mathcal{F}} 1_{f\in \hat{R}_{set}(\mathcal{F},\theta)}  \rho(f) d \rho$.
	If the hypothesis space has a uniform prior, then the Rashomon ratio is the volume of the Rashomon set divided by the volume of the hypothesis space $\hat{R}_{ratio}(\F,\theta) = \frac{\V(\hat{R}_{set}(\F,\theta))}{\V(\F)}$, where $\V(\cdot): \F\rightarrow \R_+$ is the volume function. If the hypothesis space is discrete with a uniform prior, the Rashomon ratio can be computed as $\hat{R}_{ratio}(\F, \theta) = \frac{\left|\hat{R}_{set}(\mathcal{F},\theta)\right|}{\left|\F\right|},$ 
	where $\left|A\right| = \sum_{f\in \mathcal{F}} \mathbb{1}_{f\in A}$.
The Rashomon ratio represents the fraction of models that are good (the fraction of models that fit the data about equally well). A larger Rashomon ratio implies that more models perform about equally well. 
The data set $S$ is denoted in the subscript, as  $\hat{R}_{{ratio}_S}(\F, \theta)$.
	 

   We consider \textit{true Rashomon sets} that contain models with low true risk, relative to the optimal true risk value, with parameter $\gamma > 0$:
    \[R_{set}(\F, \gamma) = \{f\in \F: L(f)\leq L(f^*)+\gamma\},\]
    where $f^*\in\F$ minimizes the true risk. $R_{ratio}(\F, \theta)$ denotes the Rashomon ratio for the true Rashomon set.
    	
    
    A large true Rashomon set, as it turns out, can be a certificate of the existence of a simpler model. Though, since we can never actually explore the true Rashomon set, we would never know whether it will be (or has been) useful for a particular problem. We explain this in Section \ref{section:generalization}, and spend most of our effort considering \textit{empirical} Rashomon sets, which are easier to work with in practice. 
    
    When the hypothesis space has a parameterized representation (denote $\F_{\Omega}$), we assume that we can parameterize each model $f \in \F_{\Omega}$ with a unique parameter vector $\omega \in \Omega$ of finite length and denote $f(z)= f_{\omega}(z)$. 
    In the next section, we discuss 
    properties of the Rashomon ratio as a complexity measure.

\section{Rashomon ratio as a simplicity measure}\label{sec:compelity_measures}

\begin{table}[t]
    \centering
    \small
    \caption{Comparison of Rashomon ratio and other complexity measures. The Rashomon ratio considers the fact that there are multiple good models and is a property both of the hypothesis space and data. 
    }
    \label{tab:complexity_comparison}
    \begin{tabular}{|p{0.33\columnwidth}|p{0.17\columnwidth}|p{0.2\columnwidth}|p{0.15\columnwidth}|}
    \hline
        Complexity measure & Property of & Depends on data & Considers set of good models \\\hline
        VC Dimension & hypothesis space & no&no\\\hline
        Algorithmic stability (Hypothesis stability \cite{bousquet2002stability}) & algorithm, hypothesis space & no&no\\\hline
        Empirical algorithmic stability (Algorithmic hypothesis stability \cite{bousquet2002stability}) & algorithm, hypothesis space & 
        yes&no\\\hline
        Geometric margins & one function & yes&no\\\hline
        Empirical Local Rademacher Complexity \citep{bartlett2005local} & hypothesis space & depends on features, not on labels &no\\\hline
        Rashomon ratio & hypothesis space & yes, but not always on labels (see Theorem \ref{th:ridge_regression}) &yes\\\hline
   \end{tabular}
\end{table}

    The Rashomon ratio, as a \textit{property of a data set and a hypothesis space}, serves as gauge of simplicity of the learning problem. 
    If the Rashomon set is large, many different reasonable optimization procedures could lead to a model from the Rashomon set. Therefore, for large Rashomon sets, accurate models tend to be easier to find (since optimization procedures can find them). 
    \textit{In other words, if the Rashomon ratio is large, the Rashomon set could contain many accurate and simple models, and the learning problem becomes simpler.}
    On the other hand, smaller Rashomon ratios might imply a harder learning problem, especially in the case of few deep and narrow local minima. 
    
    The Rashomon ratio can give insight into the simplicity of a learning problem, though 
    it was designed for a fundamentally different goal than well-known complexity measures from learning theory (see Table \ref{tab:complexity_comparison}). While those complexity measures were designed to help us understand generalization, the Rashomon ratio (with additional assumptions) helps us understand whether simpler functions might exist with the same level of accuracy as complex functions.
    The Rashomon ratio depends on a loss function, the hypothesis space, and a data set, while the majority of other measures are either data-agnostic or focus on properties of a specific model in the space.


    \paragraph{The Rashomon ratio is different from VC dimension.} 
    The VC dimension \citep{vapnik1971uniform} shows the expressive power of a hypothesis space for \textit{any} data set including \textit{an extreme} arrangement of data points and labels. On the contrary, the Rashomon set depends on an empirical risk minimizer that we compute directly for a specific data set, which may not be extreme.
    
    \paragraph{The Rashomon ratio is different from algorithmic stability.}
    Algorithmic stability \citep{bousquet2002stability} (see Definition \ref{def:algorithmic_stability}) depends on a change to a data set, whereas Rashomon Ratio uses a fixed data set. As we will show in Theorem \ref{th:ridge_regression} in Appendix \ref{appendix:ridge}, in the case of linear least squares regression, the Rashomon ratio depends on features ($X$) only, and does not depend on regression targets $Y$. In contrast, hypothesis stability depends heavily on $Y$. In fact, if we can control how we change the set of targets, hypothesis stability (a form of algorithmic stability) can be made to change by an arbitrarily large amount. This is formalized in Theorem \ref{th:ratio_stability} with proof in Appendix \ref{appendix:alg_stability}. 
    
	\newcommand{\ThRatioStability}
	{Consider a distribution $P_{X}$ over a discrete domain $\X = \{x_1,...x_N\}$ and a learning algorithm $A$ that minimizes the sum of squares loss 
	: $\|X\omega-Y\|^2_2$.
	for a linear hypothesis space $\F_{\Omega}$.
	For any $\lambda >0$, there exist joint distributions $P_{X, Y_1}$ and $P_{X, Y_2}$ where for $\mathbf{X}$ drawn i.i.d$.$ from $P_{X}$, $\mathbf{Y}_1$ drawn from  $P_{Y_1|\mathbf{X}}$ over $\Y|\X$, and $\mathbf{Y}_2$ drawn from $P_{Y_2|\mathbf{X}}$ over $\Y|\X$,  the expected Rashomon ratios are the same:
        \begin{equation*}\E_{P_{X,Y_1}}[ \hat{R}_{ratio_{ \mathbf{S}_1}}(\F_{\Omega}, \theta)] = \E_{P_{X, Y_2}}[ \hat{R}_{ratio_{ \mathbf{S}_2}}(\F_{\Omega}, \theta)],\end{equation*}
		yet hypothesis stability constants are different by our arbitrarily chosen value of $\lambda$:
		$\tilde{\beta}_2 - \tilde{\beta}_1 \geq \lambda,$
		where $\mathbf{S}_1$ and $\mathbf{S}_2$ denote data sets $\mathbf{S}_1 = [\mathbf{X}, \mathbf{Y_1}]$ and $\mathbf{S}_2 = [\mathbf{X}, \mathbf{Y_2}]$, $\tilde{\beta}_1$ is the hypothesis stability coefficient of algorithm $\A$ for distribution $P_{X, Y_1}$ and $\tilde{\beta}_2$ is the hypothesis stability coefficient for distribution $P_{X, Y_2}$.}
	
	\begin{theorem}[Rashomon ratio is not algorithmic stability]\label{th:ratio_stability}
	\ThRatioStability
	\end{theorem} 

    \paragraph{The Rashomon ratio is different from geometric margins.}
    The margin (i.e., the minimum margin of the maximum margin classifier) \citep{schapire1998boosting, burges1998tutorial} depends on points closest to the decision boundary (support vectors), while the Rashomon set does not necessarily rely on the support vectors and may depend on the full data set. Theorem \ref{th:ratio_margin}  shows this with proof in Appendix \ref{appendix:margin}.
    
    	\newcommand{\ThRatioMargin}{For any fixed $0 < \lambda < 1$, there exists a fixed hypothesis space $\F_{\Omega}$, a Rashomon parameter $\theta$, and there exist two data sets $S_1$ and $S_2$ with the same empirical risk minimizer $\hat{f}\in \F_{\Omega}$ such that the geometric margin $d$ is the same for both data sets,
		yet the Rashomon ratios are different:
		\begin{equation*}|\hat{R}_{ratio_{S_1}}(\F_{\Omega}, \theta) - \hat{R}_{ratio_{S_2}}(\F_{\Omega}, \theta)| > \lambda.\end{equation*}}
	
	\begin{theorem}[Rashomon ratio is not the geometric margin]\label{th:ratio_margin} 
	\ThRatioMargin
	\end{theorem}
	
\paragraph{The Rashomon ratio is different from empirical local Rademacher complexity.}
Empirical Rademacher complexity \citep{bartlett2005local} (see Definitions \ref{def:rademacher}) measures how well the hypothesis space can fit random assignments of the labels. The Rashomon ratio uses fixed labels. It measures the number of models that are close to optimal. In other words, the Rashomon set benefits from having multiple similar models, while Rademacher complexity treats them as equivalent. Please see Theorem \ref{th:ratio_rademacher} with definition and proof in Appendix \ref{appendix:rademacher}.
   
   	\newcommand{\ThRatioRademacher}
	{
	For $0 < \lambda < 1$, there exist two data sets $S_1$ and $S_2$, a hypothesis space $\F_{\Omega}$, and a Rashomon parameter $\theta$ such that the local Rademacher complexities defined on the Rashomon sets for $S_1$ and $S_2$ are the same:
		$\hat{R}_n^{S_1}\left(\hat{R}_{set}(\F_{\Omega}, \theta)\right) = \hat{R}_n^{S_2}\left(\hat{R}_{set}(\F_{\Omega}, \theta)\right),$
		yet the Rashomon ratios are different:
		$
		\left|\hat{R}_{ratio_{S_1}}(\F_{\Omega}, \theta) - \hat{R}_{ratio_{S_2}}(\F_{\Omega}, \theta)\right| > \lambda.
	$
	}
	
	\begin{theorem}[Rashomon ratio is not local Rademacher complexity]\label{th:ratio_rademacher}
	\ThRatioRademacher
	\end{theorem}
   

Now that we have established that the Rashomon ratio is not the same as other simplicity measures, we can now shift our focus to proving simplicity and generalization properties of models in the Rashomon set. This is critical to our thesis that simple-yet-accurate models exist.

\section{Rashomon Set Models: Simplicity and Generalization}\label{section:generalization}

    Consider two hypothesis (functional) spaces with different levels of complexity, where the lower-complexity space serves as a good \textit{approximating set} (i.e., a good \textit{cover}) for the higher-complexity space. The hypothesis spaces are called $\F_1$, for the simpler space, and $\F_2$, for the more complex space, where $\F_1 \subset \F_2$. Here, to determine the complexity of a hypothesis space, we use traditional notions of complexity (conversely, simplicity) such as covering numbers or VC dimension. For a useful example of a simple and a more complex space, consider $\F_2$ to be the space of linear models with real-valued coefficients in a space of $d$ dimensions, and consider $\F_1$ to be the space of scoring systems \citep{UstunRu2016SLIM}, which are sparse linear models, with at most $d'$ nonzero integer coefficients, $d'\ll d$. Another example is if the more complex space $\F_2$ consists of boosted decision trees, and $\F_1$ consists of single trees. 
     Generalization bounds would be tighter if we could use the lower complexity space $\F_1$, but as we are considering functions from $\F_2$,  learning theory often has us include the complexity of $\F_2$ in the bound. Given this setup, we have several questions to answer: 
\begin{enumerate}
     \item What if the higher-complexity hypothesis space we chose were more complex than necessary for modeling the data? In that case, if we had instead used the simpler model class $\F_1$, would we still get a model that is (almost) as good as we could have obtained using the more complex class $\F_2$? If so, perhaps we can leverage the complexity of the simpler model class $\F_1$ for generalization bounds on our model rather than the more complex class $\F_2$. We answer this question in  Section \ref{sec:exis0}, where a property on the complex space that will help us is that \textbf{the true Rashomon set of $\F_2$ is large enough to admit a simpler model}. We do not need to know what this model is and we may never discover it (we would likely discover a different model using data).

    \item Under what conditions on the complex and simpler model classes does the property we mentioned above (that the Rashomon set includes simpler models) hold? 
    Does it hold often? As it turns out, under natural conditions on the function class and loss function, a large Rashomon set in the complex class does imply the existence of simple-yet-accurate models. We identify these conditions in Section \ref{sec:exis1}, namely that the loss function is smooth, and that $\F_1$ serves as a cover for $\F_2$. Thus, \textbf{under these natural conditions that occur in practice, a large Rashomon set for a complex class of functions implies the existence of a simple-yet-accurate model}.

    \end{enumerate}

    The bounds we present in Section \ref{sec:exis0} do not serve the same purpose as standard statistical learning theoretic bounds, as they do not aim to bound generalization error for a single function (that is, the difference between training and test loss for a function). Rather, we are interested in bounding train loss of one function (a \textit{simpler} function) with test loss of another (the optimal model in a \textit{more complex} function class). 
    Standard learning theory analysis handles the single function case nicely; we are concerned with other questions here.

\subsection{The True Rashomon Set Can Be Very Helpful... But You Might Not Know When}\label{sec:exis0} 

    As in classic Occam's razor bounds, we start with finite hypothesis spaces. Consider finite hypothesis spaces $\F_1$ and $\F_2$, where  $\F_1 \subset \F_2$. Consider the first question discussed above: Given $\F_1$ and $\F_2$, can we have a guarantee that a model we produce using a simpler function class $\F_1$ on our data could be approximately as good as the test performance of the best model from $\F_2$? 
    In the following theorem, we will make a key assumption that allows us to do this: we assume that the Rashomon set of $\F_2$ includes a member of the simpler class of functions, $\F_1$, even if we do not know which function it is. 
     Later, in Section \ref{sec:exis1}, we show conditions under which simple models from $\F_1$ are proven to exist in the Rashomon set of $\F_2$, which depends on the size of $\F_2$'s Rashomon set.
    Here, $|\mathcal{F}|$ denotes the cardinality of the finite space $\mathcal{F}$. These bounds can be generalized to infinite hypothesis spaces with a simple extension to covering numbers, but they are designed for intuition, which works nicely with finite hypothesis spaces. Again, this is different from a regular learning theory bound as it does not consider generalization of just one function.
    
    \newcommand{\ThApproximatingFiniteOne}
    {
        Consider finite hypothesis spaces $\F_1$ and $\F_2$, such that $\F_1\subset\F_2$. Let the loss $l$ be bounded by $b$, $l(f_2,z)\in[0,b] \;\;\forall f_2\in\F_2, \forall z\in\mathcal{Z}$.
        Define an optimal function $\fts\in \textrm{\rm argmin}_{f_2\in \F_2} L(f_2)$. Assume that the true Rashomon set includes a function from $\F_1$, so there exists a model $\tilde{f}_1\in\F_1$ such that $\tilde{f}_1\in R_{set}(\F_2,\gamma)$. (Note that we do not know $\tilde{f}_1$.)
        In that case, for any $\epsilon > 0$ with probability at least $1 - \epsilon$ with respect to the random draw of data:
        \begin{equation}\label{eq:theorem43}
        L(f_2^*) - b\sqrt{\frac{\log |\F_1| + \log 2/\epsilon}{2n}} \leq \hat{L}(\hat{f}_1) \leq L(f_2^*) + \gamma+b\sqrt{\frac{ \log 1/\epsilon}{2n}},
        \end{equation}
        where $\hat{f}_1 \in \argmin_{f_1\in\F_1}\hat{L}(f_1)$. (Unlike $\tilde{f}_1$, we do know $\hat{f}_1$ because we can calculate it.)
    }

    \begin{theorem}[The advantage of a true  Rashomon set]\label{th:approximating_finite_1}
    \ThApproximatingFiniteOne
    \end{theorem} 

    That is, we can bound the best empirical model from $\F_1$ with the true risk of the best model within $\F_2$. 
    Thus, if the Rashomon set is large enough to include a single model from $\F_1$, we can work with the simpler class $\F_1$ in practice and achieve strong performance guarantees. 

    The main assumption in Theorem \ref{th:approximating_finite_1} is about the population, and does not rely on the sample. It relies only on the existence of one special function in the true  Rashomon set. There are no smoothness assumptions on the loss function.
    If the main assumption of this theorem holds, then we gain the benefit of guarantees on $\F_2$ from looking only at $\F_1$ empirically. We cannot check whether the assumption holds since it involves the true risk, but practitioners can reap the benefits of it anyway: The possibility of a large Rashomon set may embolden the practitioner to minimize over $\F_1$, achieving test error close to the best of $\F_2$ if the conditions of Theorem \ref{th:approximating_finite_1} are indeed satisfied. 

    To make the connection of this result to Rashomon sets more explicit, we will choose a specific relationship between $\F_1$ and $\F_2$, specifically, $\F_1$ will be a random sample of $\F_2$ that is chosen prior to, and separately from, learning. This is an artificial example in that $\F_1$ would never actually be chosen as a random sample from $\F_2$ in reality. However, the random sampling assumption permits $\F_1$ to be distributed fairly evenly within $\F_2$, which, arguably, could approximate the way some simpler spaces are embedded in more complex spaces.  


    If $\F_1$ is a random sample of functions from $\F_2$, 
     and if $\F_2$ has a large true Rashomon set, then the true Rashomon set is likely to include at least one model from $\F_1$.
    In that case, Theorem \ref{th:approximating_finite_1} applies. This is formalized below. 

    \newcommand{\ThApproximatingFiniteThree}
    {
    Consider finite hypothesis spaces $\F_1$ and $\F_2$, such that $\F_1 \subset \F_2$ and $\F_1$ is uniformly drawn from $\F_2$ without replacement. For loss $l$ bounded by $b$, if the Rashomon ratio is at least 
    $$R_{ratio}(\F_2,\gamma)\geq  1 - \epsilon^{\frac{1}{|\F_1|}} $$
    then for  any $\epsilon > 0$, with probability at least $\left(1-\epsilon\right)^2$ with respect to the random draw of functions from $\F_2$ to form $\F_1$ and with respect to the random draw of data, the assumptions of Theorem \ref{th:approximating_finite_1} hold and thus the bound \eqref{eq:theorem43} holds.
    }
    
    \begin{theorem}[Example of the advantage of a large true Rashomon set]\label{th:approximating_finite_3}
    \ThApproximatingFiniteThree
    \end{theorem}


    
     Table \ref{table:th43examples} shows possible values of the lower bound on the Rashomon ratio, given $|\F_1|$ and $\epsilon$. For example, the first line of the table states that if at least a tiny fraction (0.0053\%) of the complex function space $\F_2$ consists of good models, and there exists at least 100,000 simple functions in $\F_1$, then the chance that we will find an accurate-but-simple model on our data set is over 99\%.
    
    
    The intuition for Theorem \ref{th:approximating_finite_3} holds beyond the case when $\F_1$ is randomly sampled from $\F_2$, it holds whenever $\F_1$ covers $\F_2$ sufficiently well. This intuition is that as the true Rashomon ratio increases, it is more likely that the empirical risk minimum of $\F_1$ will be close to the minimum of the true risk of $\F_2$.
    
     \begin{table}[t]
    \centering
    \caption {Examples of the possible usage of Theorem \ref{th:approximating_finite_3}. }
    		\label{table:th43examples}
    \begin{tabular}{|p{\columnwidth}|} 
    				\hline
    			If $|F_1|= 100000$ then to get the bound \eqref{eq:theorem43} to hold with probability at least 99\% \\\hspace{0.5cm}
    			the Rashomon ratio should be $R_{ratio}(\F_2,\gamma) \geq 0.0053\%$.\\\hline
    			If $|F_1|= 10000$ then to get the bound \eqref{eq:theorem43} to hold with probability at least 99\% \\\hspace{0.5cm}
    			the Rashomon ratio should be $R_{ratio}(\F_2,\gamma) \geq 0.053\%$.\\\hline
    			If $|F_1|= 1000$ then to get the bound \eqref{eq:theorem43} to hold with probability at least 99\% \\\hspace{0.5cm}
    			the Rashomon ratio should be $R_{ratio}(\F_2,\gamma) \geq 0.53\%$.\\\hline
    			\end{tabular}

    \end{table}

\subsection{Proving the Existence of Simple-yet-Accurate Models with Good Generalization}\label{sec:exis1}

 Theorems \ref{th:approximating_finite_1} and \ref{th:approximating_finite_3} do not take advantage of the fact that we can investigate $\F_2$ empirically, and \textit{more easily than we can investigate} $\F_1$; these theorems instead only discuss exploration of $\F_1$. Thus, the next analysis makes two improvements: (1) it studies empirical Rashomon sets instead of true Rashomon sets, (2) it substitutes the unrealistic random draw assumption for a realistic smoothness assumption. We now  assume smoothness of the loss over the function space.
   

    
    The field of Approximation Theory provides general conditions under which classes of functions can approximate each other. Given a target function from one class, we want to know whether a sequence of functions from another class can converge to the target. Table \ref{table:approximation} in Appendix \ref{appendix:approximatingtable} shows classes of functions $\F_2$ that can be approximated by classes $\F_1$. For instance, piecewise constant functions, such as decision trees, can approximate smooth functions.

    For a hypothesis space $\F$ and some $f' \in \F$, define the $\delta$-ball of functions centered at $f'$ as $B_{\delta}(f')=\{f \in \F: \|f' - f\|_p \leq \delta\}.$  A loss $l: \F\times\X \rightarrow \Y$  is said to be K-\textit{Lipschitz}, $K \geq 0$, if for all $f_1, f_2 \in \F$ and for all $z \in \mathcal{Z}$:
    $|l(f_1, z) -l(f_2, z)| \leq K \|f_1 - f_2\|_p.$ The $p$-norm can be defined, for example, as $\|f\|_p = \left(\int_{\X}|f|^pd\mu\right)^{1/p}$, where $\mu$ is a measure on $\X$. Define a $\delta$-packing as a  finite set $\Xi =\{\xi_1, ..., \xi_k | \xi_i \in \F\}$ such that $\|\xi_i - \xi_j\|_p > \delta$, meaning that $B_{\delta/2}(\xi_i) \cap B_{\delta/2}(\xi_j) = \emptyset$ for all $i \neq j$. The \textit{packing number} $\mathcal{B}(\F, \delta)$ is the largest $\delta$-packing. 
	
	  Theorem \ref{th:existence_multiple} below uses the approximating set argument from the previous subsection, but now requires the Rashomon set to be large enough to include balls of functions rather than using the random draw assumption. As long as the set of simpler functions is distributed well among the full hypothesis space, each ball contains at least one function from the simpler class.
    
    \newcommand{\ThExistenceMultiple}
    {
    For $K$-Lipschitz  loss $l$ bounded by $b$, consider hypothesis spaces $\F_1$ and $\F_2$, $\F_1 \subset \F_2$. With probability greater than $1 - \epsilon$ w.r.t$.$ the random draw of training data, if for every model $f_2\in \hat{R}_{set}(\F_2,\theta)$ there exists $f_1 \in \F_1$ such that $\|f_2 - f_1\|_p \leq \delta$, then there exists at least $B = \mathcal{B}(\hat{R}_{set}(\F_2, \theta), 2\delta)$ functions $\bar{f_1^1}, \bar{f_1^2}...,\bar{f_1^B} \in \hat{R}_{set}(\F,\theta)$ such that:
    
    \begin{enumerate}[leftmargin=0.055\textwidth]
        \item They are from the simpler space: $\bar{f_1^1}, \bar{f_1^2}...,\bar{f_1^B} \in \F_1$.
        \item $\left|L(\bar{f_1^i}) - \hat{L}(\bar{f_1^i})\right| \leq 2K R_n(\F_1) + b\sqrt{\frac{\log(2/\epsilon)}{2n}},$ for all $i \in [1,..,B]$,
        where  $R_n(\F)$ is the Rademacher complexity of a hypothesis space $\F$. (This is from standard learning theory.)
    \end{enumerate}
    }
    
    \begin{theorem}[Existence of multiple simpler models]\label{th:existence_multiple}
    \ThExistenceMultiple
    \end{theorem}
    
    
    From Theorem \ref{th:existence_multiple}, we see that since larger Rashomon sets have larger packing numbers, they contain more simpler models with good generalization guarantees. 
    Note that in Theorem 
    \ref{th:existence_multiple}, other complexity measures  from learning theory could be used. 
    We chose Rademacher complexity as it provides the tightest bound among standard complexity measures.
    

 Theorem \ref{th:existence_multiple} has practical implications. \textit{If the Rashomon set is large, and the smoothness conditions are obeyed, Theorem \ref{th:existence_multiple} shows that many simple-yet-accurate models would exist, prior to actually finding them.} Knowledge that simple models exist implies it will be worthwhile to actually solve the difficult optimization problem to find a simple model.

Thus, \textit{if} the Rashomon set is large, we have a guarantee. But how will we know when is the Rashomon set large? This is what we answer in the next section.



\section{Larger Rashomon Ratios Correlate with Similar Performance of Machine Learning Algorithms, and Good Generalization}\label{section:experiments}


We expect that in many real-world applications of machine learning, properties similar to the assumptions behind our theorems hold, i.e., that large enough Rashomon sets intersect simpler hypothesis spaces in ways that lead to or explain good performance. This conjecture is difficult to verify theoretically because it is not a mathematical conjecture about the structure of two specific function spaces, but a statement about many function spaces, and how they interact with commonly occurring data sets. Thus, we consider this question empirically.

Our experiments will demonstrate that, in the case where Rashomon sets are large, two conclusions follow that are consistent with our theoretical development. First, \textit{training} performance in \textit{simpler} hypothesis spaces is correlated with \textit{test} performance in the more \textit{complex} hypothesis spaces (Theorem \ref{th:approximating_finite_1}), and second, that good \textit{training} performance in a \textit{simpler} space $\F_1$ correlates with \textit{good generalization} performance of other models in 
the more \textit{complex} space $\F_2$. 
Most importantly, our experiments suggest an intriguing alternative to the often difficult computational problem of directly estimating the size of the Rashomon set, namely that {\em similar performance across a range of algorithms with different hypothesis spaces is strongly correlated with a large Rashomon set}.

Now we will describe our experimental setup for arriving at these conclusions.

\subsection{Experimental Design}

\noindent \paragraph{Data sets.}
	We used 38 machine learning classification data sets from the UCI Machine Learning Repository \citep{Dua:2017}, among which 16 have categorical features and 22 have real-valued features. The majority of the data sets are binary classification data sets and we adapted the rest to binary classification (as shown in  Table \ref{table:datasets} in Appendix \ref{apendix:more_data}) to make importance sampling easier (as discussed in Appendix \ref{appendix:perfomance_figures}). The number of features varies from 3 to 784, with the majority of the data sets being in the 15--25 feature range. Appendix \ref{apendix:more_data} contains a description of the data sets we considered.

\noindent \paragraph{Definition of complex hypothesis space.} 
For these experiments, we will consider $\F_2$ to be the union ($\F_{union}$) of the hypothesis spaces of five popular machine learning algorithms: logistic regression (LR), CART, random forests (RF), gradient boosted trees (GBT), and support vector machines with RBF kernels (SVM). CART, RF and GBT were regularized by varying the tree depth, the minimum number of samples required to split a node, the minimum number of samples required to create a leaf node, and the number of trees in the ensemble. SVMs were tuned by varying the regularization parameter and the kernel coefficient and LR by varying the regularization parameter. Appendix \ref{appendix:perfomance_figures} discusses the effect of regularization on the model class. We chose algorithms that search hypothesis spaces of different complexity to ensure that these algorithms produce diverse models.
The notion of $\F_2$ as a union of hypothesis spaces may seem surprising at first, but it is consistent with how many machine learning practitioners approach problems by running a collection of machine learning techniques in parallel and comparing the results, creating a {\em de facto} union space. Our experiment has three steps, as follows.

\noindent \paragraph{Step 1: Run all machine learning algorithms.}
We obtain training and generalization performance from all algorithms (logistic regression, CART, random forests, gradient boosted trees, and SVM with RBF kernels) on all data sets.

\noindent \paragraph{Step 2: Estimate the size of the Rashomon set.}
It is not possible to measure the Rashomon set of such a complex model space, so we will estimate its size by sampling from an approximating set, which is decision trees of bounded depth. 
Decision trees are easy to sample and can refine an input space arbitrarily finely as tree depth increases.
With sufficient depth they can approximate many other types of hypothesis spaces, including those used by other machine learning methods. Thus, we will measure the size of the Rashomon set and Rashomon ratio in decision trees of depth seven as a surrogate for measuring these quantities in $\F_2$. The suitability of these trees for this role is an empirical observation about the data sets we have used; they may not be a suitable surrogate for some other data sets, e.g., imagery data. 
We measure the size of the empirical Rashomon ratio as a surrogate for the true Rashomon ratio when referring to Theorem \ref{th:approximating_finite_1}. 
To estimate the Rashomon ratio of depth seven decision trees, we used importance sampling. 
The proposal distribution assigns the correct labels to the leaves of the tree based on the training data. Since the data are populated on a bounded domain, to grow a tree up to a depth $D$ fully, we make $2^{D-1}$ splits.
	For each data set and each depth, we average our results over ten folds for data sets with less than 200 points and over five folds for data sets with more than 200 points, and we sample 250,000 decision trees per fold. We choose the Rashomon parameter $\theta$ to be 5\%, and, therefore, all the models in the Rashomon set have empirical risk not more than $\hat{L}(\hat{f})+ 0.05$, where $\hat{L}(\hat{f})$ is the lowest achievable empirical risk across all algorithms we considered. 
	We further discuss experimental setup in Appendix \ref{appendix:perfomance_figures}.

\noindent \paragraph{Step 3: See if a large Rashomon Set in Step 2 correlates with performance differences in Step 1.}
By construction, the hypothesis spaces of each of the machine learning algorithms we consider are embedded in $\F_2$.
RF and GBT both enjoy extremely rich hypothesis spaces that are likely close in size to $\F_2$ itself. 
LR and CART are less expressive than these others, so we will view LR and CART as simpler, $\F_1$ type, hypothesis spaces. Our question to answer is whether a large Rashomon set measured in Step 2 correlates with the functions from $\F_1$ (CART, LR) having performance as good as that of $\F_2$ (GBT, RF, SVM) as our theory predicts it will. 

\subsection{Experimental Results}\label{subsec:simper}

	\begin{figure*}[t]
		\centering
		\begin{tabular}{cc}
		
		\begin{subfigure}[b]{0.44\textwidth}
		\includegraphics[width=0.48\textwidth]{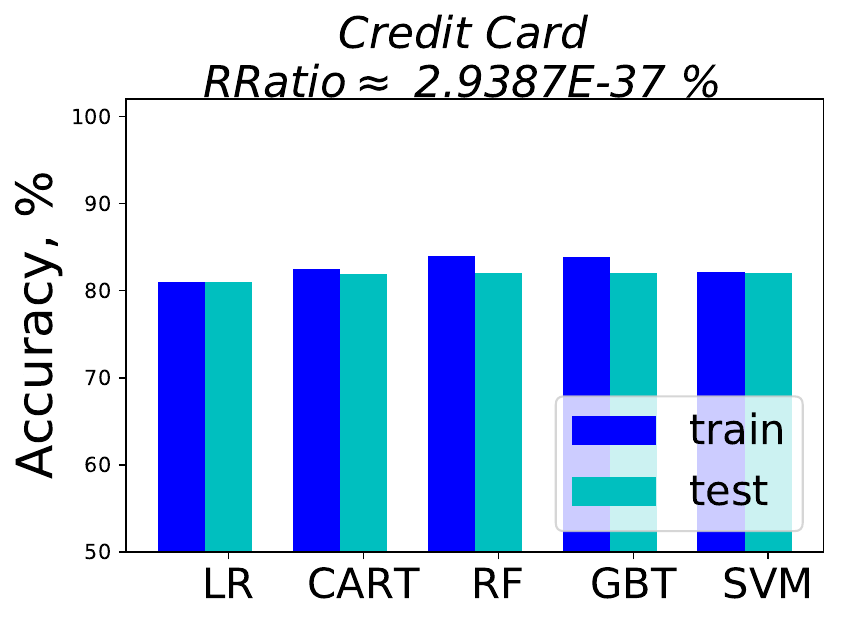}
		\includegraphics[width=0.48\textwidth]{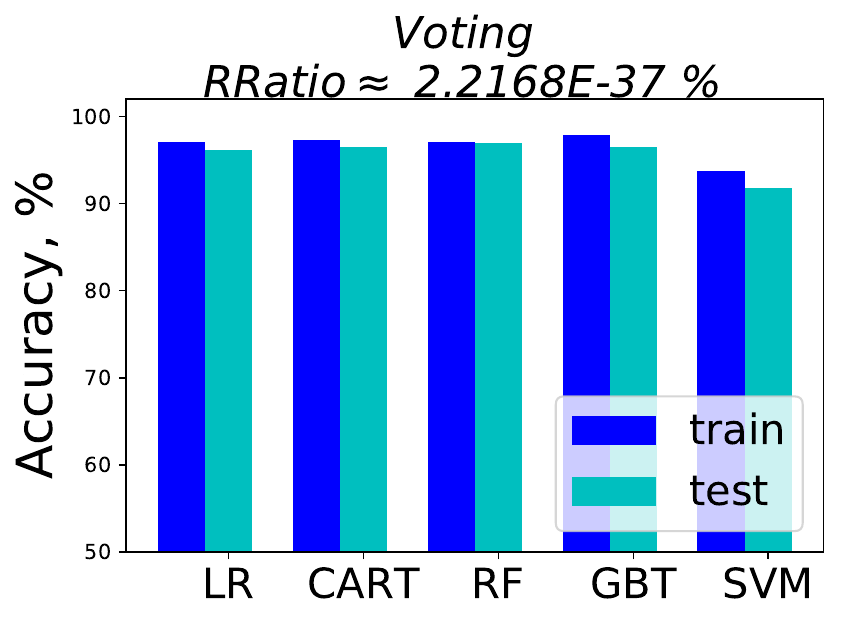}\\
		\includegraphics[width=0.48\textwidth]{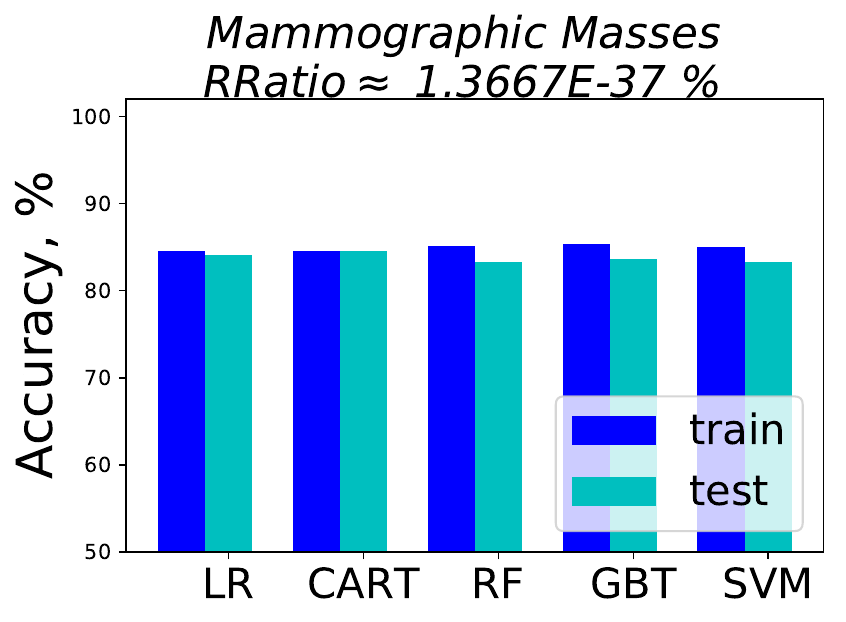}
		\includegraphics[width=0.48\textwidth]{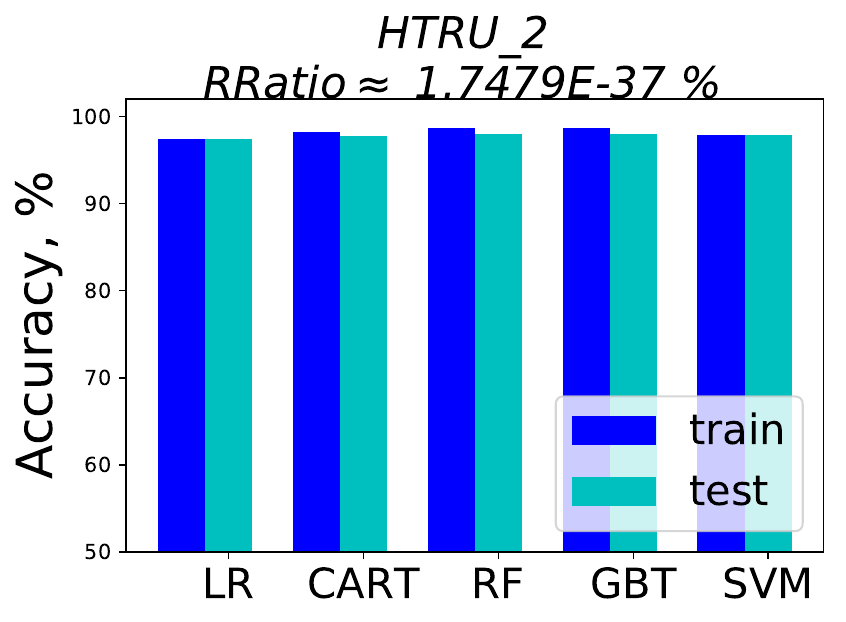}
		\caption{ }
		\end{subfigure}
		\rulesep
		
		\begin{subfigure}[b]{0.22\textwidth}
		\includegraphics[width=0.96\textwidth]{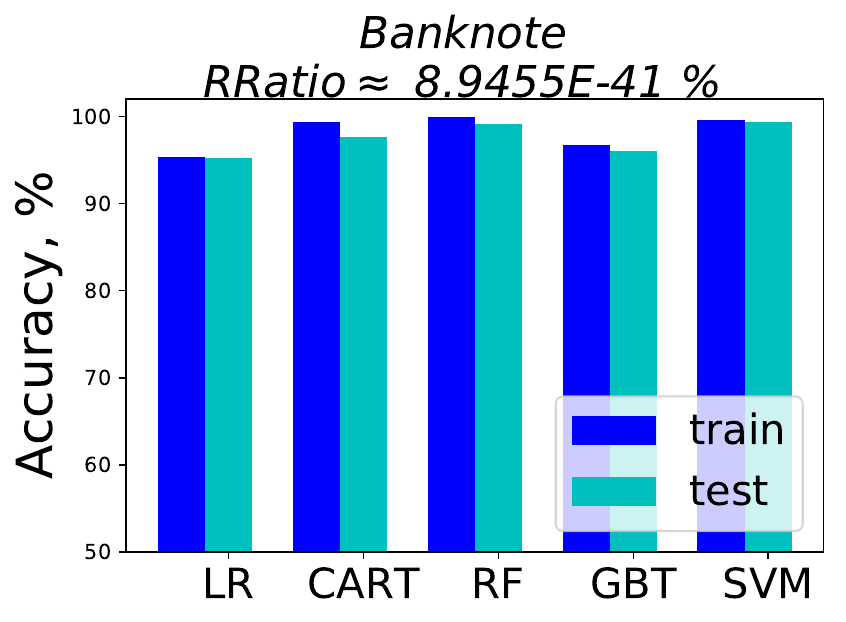}\\
		\includegraphics[width=0.96\textwidth]{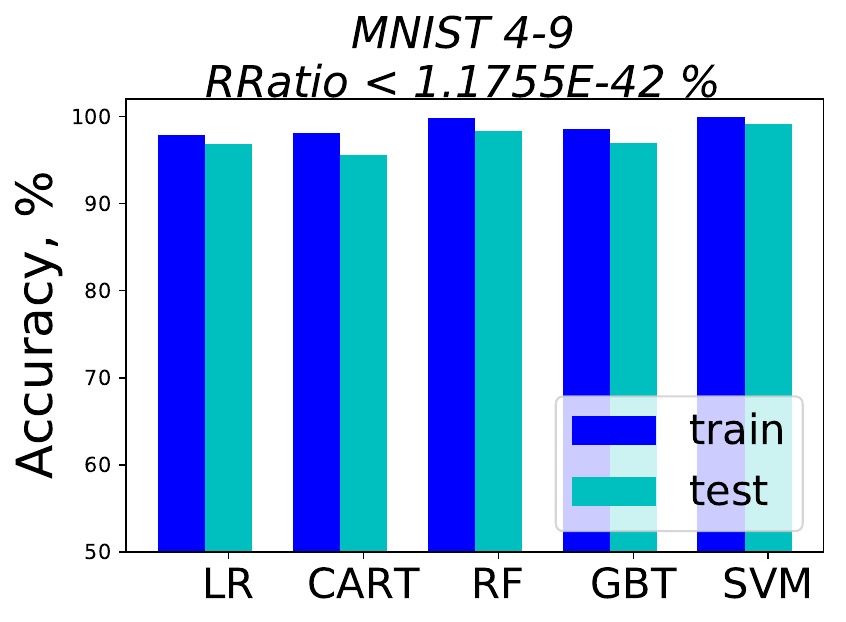}
		\caption{ }
		\end{subfigure}
		
		 \rulesep
		
		\begin{subfigure}[b]{0.22\textwidth}
		\includegraphics[width=0.96\textwidth]{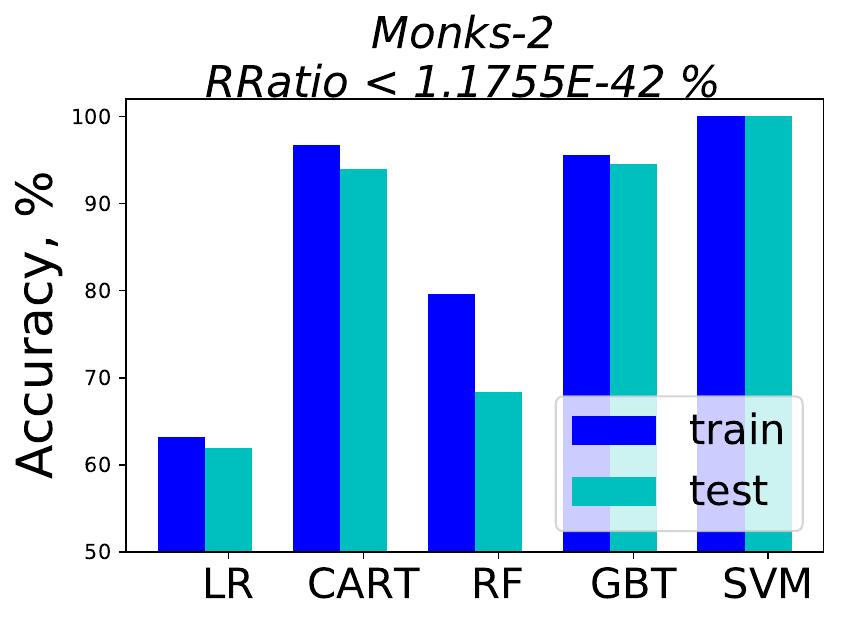}\\
		\includegraphics[width=0.96\textwidth]{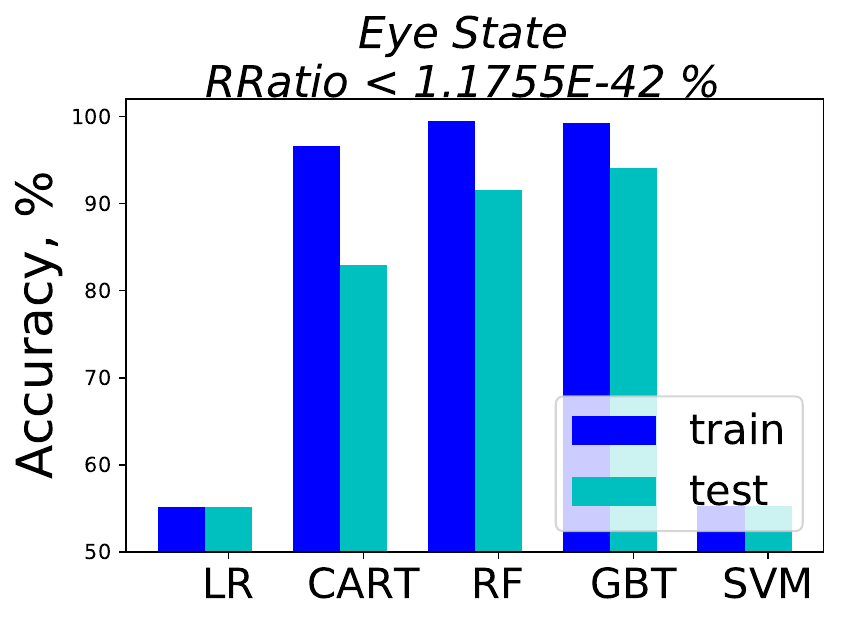}
		\caption{ }
		\end{subfigure}
		
		\end{tabular}
		\caption{ (a) Examples of experiments on four data sets showing that larger Rashomon ratios lead to similar performance of five machine learning algorithms with regularization. All the algorithms generalize well and have similar test accuracy. (b)-(c): Examples showing that smaller Rashomon ratios do not necessarily imply a performance difference between machine learning algorithms. Even with low Rashomon ratios, algorithms can be highly accurate and generalize well, as shown in Figure (b). On the other hand, when the Rashomon ratio 
		is small, sometimes algorithms can perform differently or fail to generalize, as shown in Figure (c). In the figure, test accuracies, training accuracies and the Rashomon ratio are averaged over ten folds. We show all 38 data sets in the Appendix \ref{appendix:perfomance_figures}.}
		\label{fig:exp_bar}
	\end{figure*}
	
Figure \ref{fig:exp_bar}(a) shows the performance of the five machine learning algorithms on data sets for which the Rashomon ratio was largest, as measured in the space of decision trees of depth 7.
	Performance for all data sets is shown in  Figures \ref{fig:bar_plots_1} and \ref{fig:bar_plots_2} in Appendix \ref{appendix:perfomance_figures}.
	Across the 38 data sets considered, we observe \textit{\textbf{larger Rashomon ratios led to approximately similar training results across all algorithms}} (within $\sim 5\%$ difference between algorithms). Here, large Rashomon ratios are on the order of $10^{-37}\%$ or $10^{-38}\%$, whereas small Rashomon ratios are $10^{-40}\%$ or less\footnote{For other data sets and other metrics of measuring the Rashomon set, the results might be different.}.
	Moreover, all of the models chosen by the algorithms, including simpler $\F_1$ type models, generalized well
	(the differences between training and test errors are within $\sim 5\%$). These results are consistent with our thesis that larger Rashomon sets lead to the existence of accurate-yet-simpler models (in agreement with the theory in Section \ref{sec:exis0}), and that larger Rashomon sets lead to better generalization.
	\textit{The results also imply that large Rashomon sets do occur in many data sets, with the Rashomon effect being large enough to include simpler models in practice 
	(in agreement with Section \ref{sec:exis1})}.

	Interestingly, the converse statement, that similar performance across different algorithms should lead to large Rashomon sets, does not always hold; sometimes, generalization occurs with small Rashomon ratios (see Figure \ref{fig:exp_bar}(b)). This observation could be explained in several different ways. Mainly, the Rashomon ratio is not the only driver of good generalization performance. The amount of data is one obvious additional driver. Appendix \ref{sec:sm_rsetmeasure} discusses this further. Quality of features is another driver, as discussed in Appendix \ref{appendix:ratios_and_features}.

	
    Our second main result is that in \textit{\textbf{all}} cases where large Rashomon ratios were observed, 
    \textbf{\textit{ test performance was consistent with training performance across algorithms of varying complexity.}}
    This correlation between the size of the Rashomon ratio and consistent generalization performance suggests an indirect means of assessing the size of the Rashomon ratio as an alternative to the computationally intensive approach of sampling. When consistent training and test performance across algorithms is observed, this {\em may} indicate a large Rashomon ratio.
   
    One thing we notably did not observe were cases where algorithms did not generalize, performance differed across algorithms, and the Rashomon set was large. Across all 38 data sets, we did not observe cases where the Rashomon set was large and performance differed among algorithms.

    Figure \ref{fig:exp_bar}(c) shows small Rashomon sets, where we observe wildly different performance across algorithms, where sometimes the models generalize and sometimes they do not. We show one example of each of these cases in Figure \ref{fig:exp_bar}(c). Our theory does not apply to the case of small Rashomon sets, and thus there is no guarantee for such data sets.

\section{Rashomon Sets in the Presence of Noise}\label{sec:noise}

We have seen that we can empirically determine whether a data set is likely to have a large Rashomon set: as we showed, we simply run many algorithms, and if they all perform similarly and generalize, there could be a large Rashomon set. But what about before examining the data? Could we know, just from understanding what kind of data set it is, whether it is likely to have a large Rashomon set? We aim to answer this now.

A typical reason given for ``underspecification'' \cite{d2020underspecification} (i.e., a large number of approximately-equally good models, a large Rashomon set) is the presence of substantial noise in the data. Intuitively, for data whose outcomes are essentially a random guess, it makes sense that no model would perform well, and many models would be equally poor. But what about more interesting cases? Does this intuition still hold? First, as we have shown above, large Rashomon sets exists for non-noisy data sets as well  (see Figure \ref{fig:exp_bar} Voting and HTRU 2 data sets); \textit{it is not just noise that determines the Rashomon set}. Second, it is not true that the Rashomon set always gets much larger under noise. In fact, if we add random classification noise \cite{angluin1988learning, Natarajan13} (each label is flipped independently with some small probability), it is possible that the Rashomon set does not change at all. This is because the error rate of all models in the Rashomon set (assuming they are all better than random guessing) increases by the same amount in expectation. At least, as we show in Theorem \ref{th:noise_labels}, the size of the true Rashomon set does not decrease if we add random classification noise.

\newcommand{\ThNoiseLabels}
    {Consider hypothesis space $\F$, data distribution $\mathcal{D} = \mathcal{X}\times \mathcal{Y}$, where, as before, $\mathcal{X}\in\R^p$, and $\mathcal{Y}\in\{-1,1\}$. Let $\rho \in (0,\frac{1}{2})$ be a probability with which each label $y_i$ is flipped independently, and $\mathcal{D}_{\rho}$ denotes the noisy version of $D$. If the loss function is $\phi(f(x), y) = \mathbb{1}_{[f(x)\neq y]}$, then in expectation, the true Rashomon set over $\mathcal{D}$ is a subset of the true Rashomon set over $\mathcal{D}_{\rho}$, $R_{{set}_{\mathcal{D}}}(\F,\gamma)  \subseteq R_{{set}_{\mathcal{D}_{\rho}}}(\F,\gamma)$.
    }

    \begin{theorem}[Expected size of the true Rashomon set cannot decrease under random classification noise]\label{th:noise_labels}
    \ThNoiseLabels
    \end{theorem}

In Theorem \ref{th:noise_labels} we have shown that the size of the true Rashomon set does not decrease when adding random classification noise, but to prove that $R_{{set}_{\mathcal{D}}}(\F,\gamma)\subset R_{{set}_{\mathcal{D}_{\rho}}}(\F,\gamma)$, we would need at least one model $\bar{f}$ such that $\bar{f} \not\in R_{{set}_{\mathcal{D}}}(\F,\gamma)$, yet $ \bar{f}\in R_{{set}_{\mathcal{D}_{\rho}}}(\F,\gamma)$, and such a model may not actually exist; in fact, if all models have increased error rates when noise is added, it does not.

We still are left with finding a scenario where noise does impact the size of the Rashomon set in order to provide some proof to the intuition. Let us consider the setting of linear (Gaussian) discriminant analysis, where the data arise from two Gaussians, one with positive labels and one with negative labels. Instead of increasing label noise, we will increase feature noise by increasing the variances or changing the means of the Gaussians, making the two distributions overlap. In this case, will the Rashomon set increase in size? The answer to this question is conjectured in Conjecture \ref{th:noise_gaussian}.

\begin{figure*}[t]
\centering
\includegraphics[width=0.7\textwidth]{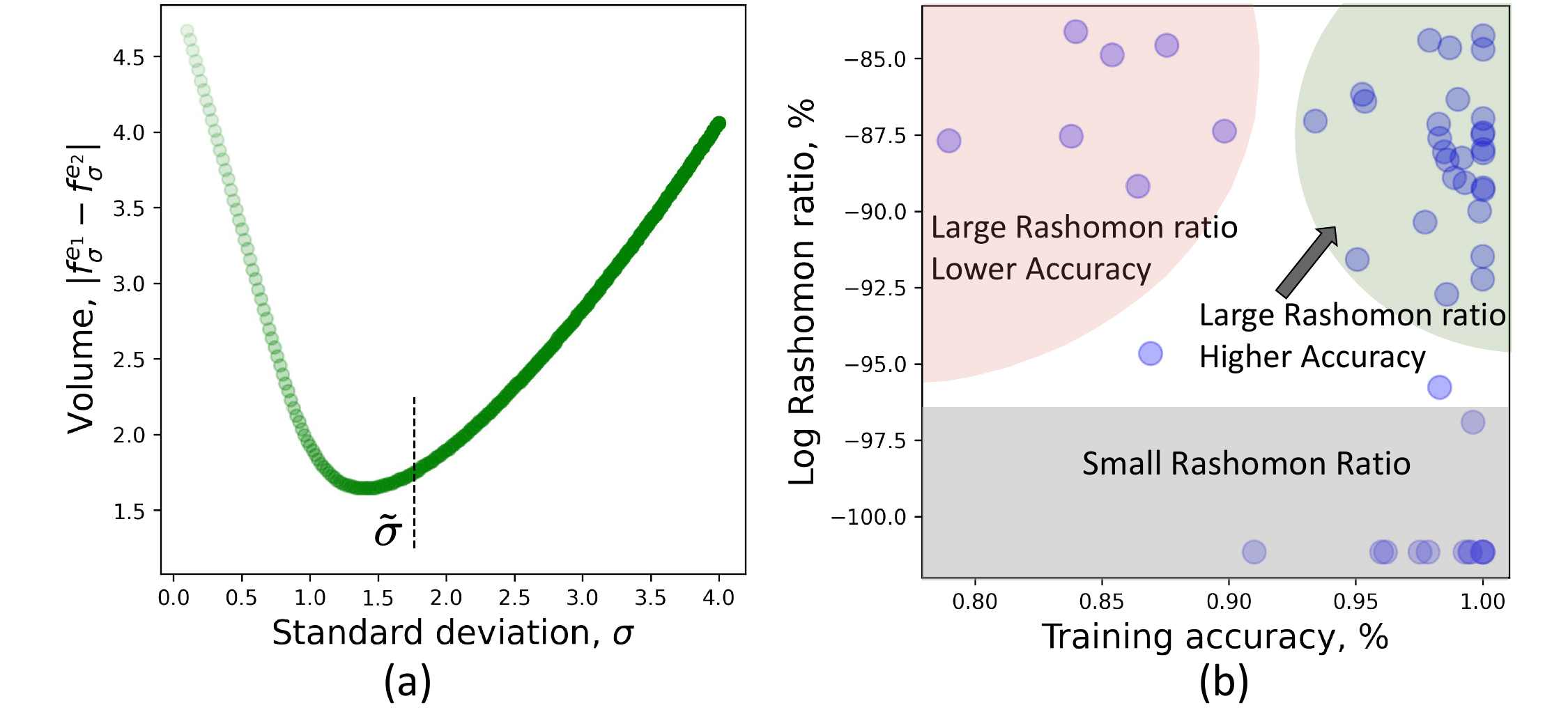}
        \caption{(a) Dependence of the Rashomon ratio on noise $\sigma$ for the two Gaussians example in Conjecture \ref{th:noise_gaussian}.  When $\sigma > \tilde{\sigma}$, as we add more noise, the size of the Rashomon set increases. (2). 
        (b) Scatter plot of the log of Rashomon ratio versus the maximum accuracy across five different algorithms for 38 data sets in Section \ref{section:experiments}. 
        We observe larger Rashomon ratios in both noisy and non-noisy data sets. 
        }
        \label{fig:noise}
    \end{figure*}

\newcommand{\ThNoiseGaussianOneD}
    {Consider data distribution $\mathcal{D} = \mathcal{X}\times \mathcal{Y}$, where, $\mathcal{X}\in\R$, $\mathcal{Y}\in\{-1,1\}$, and classes are balanced $P(Y=-1)=P(Y=1)$ and generated by Gaussian distributions $P(X|Y=-1)=\mathcal{N}(\mu_1,\sigma^2)$,  $P(X|Y=1)=\mathcal{N}(\mu_2,\sigma^2)$, where $0 \leq \mu_1 <\mu_2$. For the hypothesis space $\F = \{f: f\in (\beta_1, \beta_2) \}$, where $ (\mu_1, \mu_2) \subset (\beta_1, \beta_2)$,  $\beta_1 \ll \mu_1$, and $\mu_2 \ll \beta_2$, 
    and the Rashomon parameter $\gamma > 0$:
    \begin{enumerate}
        \item[(I)] The volume of the Rashomon set is $\V(R_{{set}_{\sigma}}(\F, \gamma)) = |f^{e_1}_\sigma- f^{e_2}_\sigma|$, where $f^{e_1}_\sigma$ and $f^{e_2}_\sigma$ are the two solutions to Eqn$.$ \eqref{eq:gaus_obg}, where $\Phi$ is the CDF of the standard normal:
    \begin{equation}\label{eq:gaus_obg}
        2\Phi \left(\frac{\mu_2-\mu_1}{2\sigma}\right) - \Phi \left(\frac{\mu_2-f}{\sigma}\right) - \Phi \left(\frac{f-\mu_1}{\sigma}\right) = \gamma.
    \end{equation}
    \item[(II)] We conjecture that\footnote{The hypothesis space for Part II conservatively includes all reasonable candidates for the empirical risk minimizer. In other words, we assume that decision boundary can be anywhere between the means of the two distributions.} for $\F = \{f: f\in (\mu_1, \mu_2) \}$, as we add feature noise to the data set by increasing the standard deviation $\sigma$, for all $\sigma$ such that $\sigma > \tilde{\sigma} = \frac{\mu_2-\mu_1}{2\sqrt{2}}$, the volume of the Rashomon set 
    increases as a function of $\sigma$.
    \item[(III)] Consider the setting where $\sigma = 1$ for both Gaussians, and we add or remove noise by moving the means $\mu_1$ and $\mu_2$ of the Gaussians towards or away from each other. For any $\gamma>0$, the volume of the Rashomon set is minimized when $\mu_2\approx \mu_1+2$. Moving the Gaussians either away from or towards each other increases the volume of the Rashomon set.
    \end{enumerate}
}

    \begin{conjecture}[The Rashomon set can increase with feature noise]\label{th:noise_gaussian}
    \ThNoiseGaussianOneD
    \end{conjecture}

This conjecture is not a theorem because there is no analytical solution to the minimizer of the volume of the Rashomon set; the calculations are quite complex, involving differences of the CDF values of different Gaussians. However, \textit{all parts of the conjecture have been fully checked numerically}. In part (II), we use an analytical derivation and exhaustive numerical computations to show that the derivatives of the left side of Eqn$.$ \eqref{eq:gaus_obg} are either positive or negative sign.
For part (III), we transpose the left Gaussian to $N(0,1)$ to form a canonical problem in which all possible solutions can be computed numerically. We exhaustively search over the range of $\gamma$, finding the optimal $\mu_2$ and volume of the Rashomon set for each $\gamma$. We find that $\mu_2$ is very close to 2 for all $\gamma$.
We discuss this in Appendix \ref{appendix:gaus_proof}.


This conjecture suggests that \textbf{\textit{data that are approximately distributed according to two normal distributions, where the positive and negative normal distributions substantially overlap, will have a large Rashomon set}}. Figure \ref{fig:noise}(a) shows the dependence of the Rashomon set on the noise level $\sigma$ for $\mu_1 = 1, \mu_2 = 6$ and $\sigma \in [0.2, 4]$. 
Figure \ref{fig:noise}(b) plots maximum accuracy versus Rashomon ratio for 38 data sets considered in Section \ref{section:experiments}. These figures indicate that large Rashomon sets occur both in noisy and non-noisy data. 

There can be many data sets with characteristics as in Conjecture \ref{th:noise_gaussian}. For example, let us consider criminal recidivism data, whose Rashomon sets have been studied \cite{fisher2018model,DongRu2020} and that admit simple-yet-accurate models \cite{ZengUsRu2017,RudinWaCo2020}. Each data point is generated based on a set of random events happening in the world; whether someone enters a job training program, whether someone associates with criminal associates after release, and whether someone commits a crime each day are all random variables whose random effects are cumulative over time, and thus could be modeled by Gaussians by the central limit theorem. By this logic, we would expect many criminal recidivism prediction problems to admit large Rashomon sets. Other high-stakes predictions such as loan defaults may have similar characteristics. 

In a sense, this full analysis paints a much clearer picture as to why such problems admit simple yet similarly accurate models: their distributions are approximately Gaussian with significant overlap, such overlap leads to large Rashomon sets, and large Rashomon sets lead to the existence of simple yet similarly accurate models.

\section{Conclusion and Implications}\label{section:practical_implication_2}
	
We have proposed Rashomon sets and ratios as another perspective on the relationship between hypothesis spaces and data sets, and we have provided initial theoretical and experimental results showing that this is a unique perspective that may help explain some phenomena observed in practice. 
More specifically, the main conclusions include: (1) Large Rashomon sets can embed models from simpler hypothesis spaces (Section \ref{section:generalization}); (2) Similar performance across different machine learning algorithms may correlate with large Rashomon sets (Section \ref{section:experiments}); (3)  Large Rashomon sets correlate with existence of models that have good generalization performance (Section \ref{section:experiments}); (4) The Rashomon ratio is a measure of a learning problem's complexity (Section \ref{sec:compelity_measures}), and that data that approximately arise from overlapping Gaussian distributions tend to have large Rashomon sets (Section \ref{sec:noise}). 
    
    Consider a researcher conducting a standard set of machine learning experiments in which the performance of several different algorithms are compared, and generalization is assessed. 
    In the possible scenario where \textit{all algorithms perform similarly}, and when their models tend to generalize well on validation data, the learning problem is likely to have a large Rashomon set. 
   Based on the result in Section \ref{section:generalization}, simpler models are likely to exist in a large Rashomon set. If the researcher is interested in simpler models, they can search the simpler function class $\F_1$, a subset of the larger class $\F_2$, to locate simpler models within it. While optimizing for simplicity or interpretability constraints is usually much more computationally expensive than running standard machine learning algorithms, our thesis is that this search would be likely to succeed in the presence of a large Rashomon set. 
    In the converse case, if the researcher's algorithms perform \textit{differently from each other}, the researcher might then select a more complex model class that achieves better performance yet does not overfit. Further, if the researcher knows that the data are likely to have arisen from overlapping Gaussian distributions, the researcher could assume that it is worthwhile to search for a simple model that performs well. 

\bibliographystyle{ACM-Reference-Format}
\bibliography{bibliography}


\newpage

\appendix


\section{Difference between flat minima and Rashomon set}\label{appendix:flat_minima_figure}

 Figure \ref{fig:flatminima} illustrates differences between volume $\epsilon$-flatness and the Rashomon set.

	\begin{figure*}[h]
		\centering
		\begin{subfigure}[b]{0.3\textwidth}
			\includegraphics[width=\textwidth]{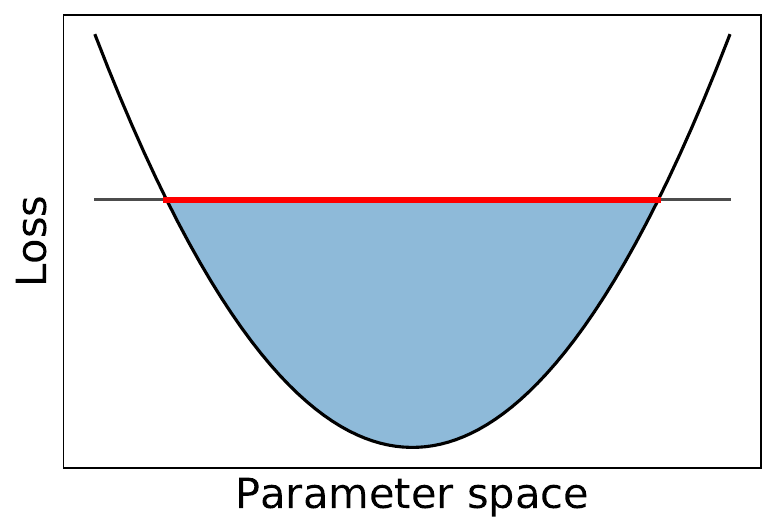}
			\caption{Volume $\epsilon$-flatness}
			\label{fig:flatminima1}
		\end{subfigure}
		~ 
		\begin{subfigure}[b]{0.3\textwidth}
			\includegraphics[width=\textwidth]{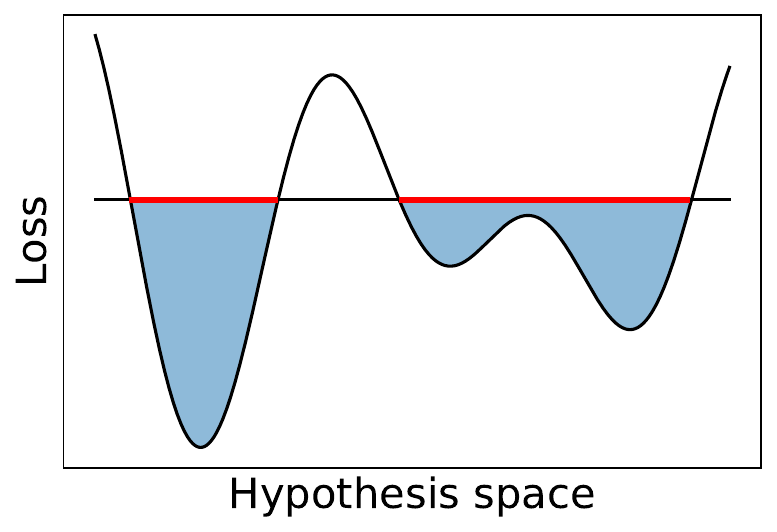}
			\caption{The Rashomon set 	}
			\label{fig:flatminima2}
		\end{subfigure}
		
		\caption{Difference between volume $\epsilon$-flatness as defined in \citet{dinh2017sharp} and the Rashomon set. The red line represents the volume $\epsilon$-flatness in (a), and the Rashomon set in (b). The volume of the Rashomon set is the sum of lengths of red lines in (b). The height of the shaded area represents (a) the parameter $\epsilon$ or the  $2\sigma$-sharpness, and (b) the Rashomon parameter $\theta$.  Volume $\epsilon$-flatness is defined by a connected component in a parameter space for a given local minimum, while the Rashomon set is defined with respect to an empirical risk minimizer over the full hypothesis space $\F$ and may contain models from multiple local minima. Rashomon sets are also defined for discrete spaces. }\label{fig:flatminima}
	\end{figure*}
	
	
\section{Connection to Simplicity Measures}\label{appendix:measures}

    We will use demonstrations to show the differences between the Rashomon ratio and other complexity measures mentioned in Section \ref{sec:compelity_measures}. However, first we discuss how to analytically compute the Rashomon ratio for ridge regression since the demonstration of the difference between algorithmic stability and Rashomon ratio relies on it.

\subsection{Analytical Calculation of Rashomon Ratio for Ridge Regression}\label{appendix:ridge}

	
	A special case of when the Rashomon ratio can be computed in  closed form in a parameter space is ridge regression. For a space of linear models $\F_{\Omega} = \{\omega^Tx, \omega \in \R^p\}$,  ridge regression chooses a parameter vector by minimizing the penalized sum of squared errors for a training data set $S = [X, Y]$:
	\begin{equation}\label{eq:ridge}
	\min_{\omega}\hat{L} (\omega) = \min_{\omega}(X\omega - Y)^T(X\omega - Y) + C\omega^T\omega,
	\end{equation}
	where the optimal solution of the ridge regression estimator is $\hat{\omega} = (X^TX + CI_p)^{-1} X^T Y.$

    Geometrically, the optimal solution to ridge regression will be a parameter vector that corresponds to the intersection of ellipsoidal isosurfaces of the sum of squares term and a hypersphere centered at the origin, with the regularization parameter $C$ determining the trade off between the loss and the radius of the sphere. More generally, isosurfaces of the ridge regression loss function are ellipsoids, and the volume of such an ellipsoid corresponds to the volume of the Rashomon set. For a hypothesis space with uniform prior and volume function $\V$,  the Rashomon ratio is $\frac{\V(\hat{R}_{set}(\F_{\Omega}, \theta))}{\V(\F_{\Omega})}$. Using the geometric intuition above, we compute the Rashomon ratio in the parameter space by the following theorem: 
    
    \newcommand{\ThRidgeRegression}
    {
    For a parametric hypothesis space of linear models $\F_{\Omega} = \{f_{\omega}(x) = \omega^Tx, \omega \in \R^p\}$ with uniform prior and a data set $S = X\times Y$, the Rashomon set $\hat{R}_{set}(\F_{\Omega}, \theta)$ of ridge regression is an ellipsoid, containing vectors $\omega$ such that:
		\begin{equation*}
		(\omega - \hat{\omega})^T \frac{X^TX + CI_p}{\theta}(\omega - \hat{\omega}) \leq 1,
		\end{equation*}
		and the Rashomon ratio can be computed as:
		\begin{equation}\label{eq:ridge_volume}
		\hat{R}_{ratio}(\F_{\Omega}, \theta) = \frac{J(\theta, p)}{\V(\F_{\Omega})}\prod_{i=1}^p\frac{1}{\sqrt{\sigma_i^2 + C}},
		\end{equation}
		where $\sigma_i$ are singular values of matrix $X$, $ J(\theta, p) = \frac{\pi^{p/2}\theta^{p/2}}{\Gamma(p/2+1)}$ and $\Gamma(\cdot)$ is the gamma function.
    }
    
    \newcommand{\ThRidgeRegressionCopy}
    {
    For a parametric hypothesis space of linear models $\F_{\Omega} = \{f_{\omega}(x) = \omega^Tx, \omega \in \R^p\}$ and a data set $S = X\times Y$, the Rashomon set $\hat{R}_{set}(\F_{\Omega}, \theta)$ of ridge regression is an ellipsoid, containing vectors $\omega$ such that:
		\begin{equation*}
		(\omega - \hat{\omega})^T \frac{X^TX + CI_p}{\theta}(\omega - \hat{\omega}) \leq 1,
		\end{equation*}
		and the Rashomon ratio can be computed as:
		\begin{equation*}
		\V(\hat{R}_{set}(\F_{\Omega}, \theta)) = J(\theta, p)\prod_{i=1}^p\frac{1}{\sqrt{\sigma_i^2 + C}},
		\end{equation*}
		where $\sigma_i$ are singular values of matrix $X$, $ J(\theta, p) = \frac{\pi^{p/2}\theta^{p/2}}{\Gamma(p/2+1)}$ and $\Gamma(\cdot)$ is the gamma function.
    }
	
	\begin{theorem}[Rashomon ratio for ridge regression]\label{th:ridge_regression}
	\ThRidgeRegression
	\end{theorem}
		\begin{proof}
		Consider all models $f_{\omega} \in \F_{\Omega}$ from the Rashomon set $\hat{R}_{set}(\F_{\Omega}, \theta)$. Then by Definition \ref{def:rset} we get:
		
		\begin{equation}\label{eq:ridgeRashomon}
		\hat{L} (X, Y, \omega) \leq \hat{L} (X, Y, \hat{\omega}) + \theta.
		\end{equation}
		
		Using $X^TY = (X^TX + CI_p)\hat{\omega}$ from the optimal solution of the ridge regression estimator $\hat{\omega} = (X^TX + CI_p)^{-1} X^T Y$, and expanding the difference between empirical risks we have:
		\begin{equation*} 
		\begin{split}
	\theta \geq &\hat{L} (X, Y, \omega) - \hat{L} (X, Y, \hat{\omega})\\
	       &= (X\omega - Y)^T(X\omega - Y) + C\omega^T\omega - (X\hat{\omega} - Y)^T(X\hat{\omega} - Y) - C\hat{\omega}^{T}\hat{\omega}\\
	       &= \omega^TX^TX\omega - 2\omega^TX^TY + C\omega^T\omega - \hat{\omega}^{T}X^TX\hat{\omega}+ 2\hat{\omega}^{T}X^TY - C\hat{\omega}^{T}\hat{\omega}\\
	       &= \omega^TX^TX\omega - 2\omega^T(X^TX + CI_p)\hat{\omega} + C\omega^T\omega - \hat{\omega}^{T}X^TX\hat{\omega} \\
	       & +  2\hat{\omega}^{T}(X^TX + CI_p)\hat{\omega} - C\hat{\omega}^{T}\hat{\omega}\\
	       &= 	\omega^TX^TX\omega  + C\omega^T\omega - 2\omega^T(X^TX + CI_p)\hat{\omega}  + \hat{\omega}^{T} X^T X \hat{\omega} + C\hat{\omega}^{T}\hat{\omega}\\
	       &=  \omega^T(X^TX + CI_p)\omega - 2\omega^T(X^TX + CI_p)\hat{\omega} + \hat{\omega}^{T}(X^TX + CI_p)\hat{\omega}\\
	       &=  	(\omega - \hat{\omega})^T(X^TX + CI_p)(\omega-\hat{\omega}).
	   \end{split}
		\end{equation*}

	Therefore the Rashomon set is an ellipsoid centered at $\hat{\omega}$:
		\begin{equation*}
		(\omega-\hat{\omega})^T \frac{X^TX+CI_p}{\theta} (\omega - \hat{\omega}) \leq 1.
		\end{equation*}
		By the formula of the volume of a p-dimensional ellipsoid, 
		the volume of the Rashomon set can be computed as:
		\begin{equation*}
		\V(\hat{R}_{set}(\F_{\Omega}, \theta)) = \frac{\pi^{p/2}\theta^{p/2}}{\Gamma(p/2+1)} \prod_{i=1}^p \frac{1}{\sqrt{\sigma_i^2 + C}},
		\end{equation*}
		where $\sigma_i$ are singular values of $X$. 
		
		Since we assume a uniform prior on $\F_{\Omega}$, $\V(F_{\Omega})$ is the volume of a box (or other closed region) containing the plausible values of $\Omega$.
		Therefore, the Rashomon ratio is $\hat{R}_{ratio}(\F_{\Omega}, \theta) =\frac{\V(\hat{R}_{set}(\F_{\Omega}, \theta))}{\V(\F_{\Omega})} =  \frac{J(\theta, p)}{\V(\F_{\Omega})}\prod_{i=1}^p\frac{1}{\sqrt{\sigma_i^2 + C}}$, where $ J(\theta, p) = \frac{\pi^{p/2}\theta^{p/2}}{\Gamma(p/2+1)}$.
	\end{proof}	
	
	
    Interestingly, from Theorem \ref{th:ridge_regression}, it follows that for ridge regression, \textit{the Rashomon ratio depends on the feature space only and does not depend on the regression targets $Y$}. Indeed, assume that every parameter vector $\omega$ such that $f_{\omega} \in \hat{R}_{set}(\F_{\Omega}, \theta)$ can be represented as $\omega = \hat{\omega} + \delta$. By a simple transformation, we have that $\hat{L}(f_{\omega}) - \hat{L}(f_{\hat{\omega}}) = \delta^TX^TX\delta$, meaning that if we take a step in parameter space, the empirical risk difference will depend only on the feature space and the step itself, and not on the targets of the problem. This observation can help us choose the parameter $\theta$ as $\theta = \delta^TX^TX\delta$ if we want to ensure some dependence between the optimal model $\hat{\omega}$ and a model of interest $\omega$. Then, by choosing the direction as $\delta = \omega - \hat{\omega}$, we can compute the Rashomon parameter $\theta$.
	
    For other algorithms, the Rashomon ratio generally depends on the targets; in that sense, ridge regression is unusual.

	\subsection{Algorithmic Stability}\label{appendix:alg_stability}
     The main motivation for algorithmic stability theory is to ensure robustness of a learning algorithm. Following \citet {bousquet2002stability}, we define the hypothesis stability of a learning algorithm as follows.	
	\begin{definition}[Hypothesis stability]
		A learning algorithm $\A$
		has $\beta$ hypothesis stability  with respect to the  loss $l$
		if  for all $i \in \{1, . . . , n\}$,
		\begin{equation*} \E_{S, z}[|l(f_S, z) - l(f_{S^{ \setminus i}} , z)|] \leq \beta,\end{equation*}
		where $\beta \in R_+$, hypothesis $f_S$ is learned by an algorithm $\A$ on a data set $S$,  loss $l(f_S, z ) = \phi(f_S(x), y)$ for $z=(x,y)$, data set $S = \{z_1, ... z_n\}$,  and $S^{\setminus i}$ is modified from the training data by removing the $i$\textsuperscript{th} element of the data set: $S^{\setminus i} = \{z_1, ..., z_{i-1}, z_{i+1}, ... z_n\}$.
	\end{definition}\label{def:algorithmic_stability}

	The Rashomon ratio is fundamentally different from  hypothesis stability, in case of linear least squares regression (which is discussed in Section \ref{appendix:ridge}). This is formalized in Theorem \ref{th:ratio_stability}.
	
	

\begingroup
\def\thetheorem{\ref{th:ratio_stability}}
\begin{theorem}[Rashomon ratio is not algorithmic stability]
\ThRatioStability
\end{theorem}
\addtocounter{theorem}{-1}
\endgroup
	
		\begin{proof} Let us create our distribution.
		Consider the least squares regression  $\min_{\omega} \sum_{i=1}^n l(\omega,\mathbf{z}_i)^2$, where $\omega \in \R^p$, and  loss $l(\omega, \mathbf{z}) = \phi(\omega^T\mathbf{x}, \mathbf{y})$ for $\mathbf{z} = (\mathbf{x},\mathbf{y})$. For the marginal distribution $P_X$ and $\mathbf{X = [x_1, ..., x_n]}$  drawn i.i.d$.$ from $P_X$, we design distributions $P_{Y_1|\mathbf{X}}$ and $P_{Y_2|\mathbf{X}}$ as:
		\begin{equation*}P_{Y_1|\mathbf{X}}(y=\mathbf{0}|\mathbf{x}) = 1\;\; \forall \mathbf{x} \in \mathbf{X},\end{equation*}
		\begin{equation*}
		P_{Y_2|\mathbf{X}}(y = \mathbf{0}|\mathbf{x} \neq \mathbf{x_0}) = 1, \text{ }P_{Y_2|\mathbf{X}}(y= \mathbf{0}| \mathbf{x} = \mathbf{x_0}) = 0.5,
		\end{equation*}
		\begin{equation*}\text{ }P_{Y_2|\mathbf{X}}(y = \mathbf{H}| \mathbf{x} = \mathbf{x_0}) = 0.5,
		\end{equation*}
		where $\mathbf{x_0} \in \{x_1,...,x_N\}$ is some fixed point with a positive probability $P_{X}(\mathbf{x_0})$ and we define $\mathbf{H} \in \R$ later. That is, the two conditional distributions have $y=0$ except when $\mathbf{x}=\mathbf{x}_0$ for $Y_2$, when it is $H$ with probability 1/2.
		
		As a first part of the proof, we show that the algorithmic stability constants are different. According to the definition of algorithmic stability, for $P_{X, Y_1}$ we have:
		\begin{equation*} \E_{S_1, z}[|l(f_{\mathbf{S}_1}, \mathbf{z}) - l(f_{\mathbf{S}_1^{ \setminus i}} , \mathbf{z})|] =0 = \tilde{\beta_1},\end{equation*}
		and for distribution $P_{X,Y_2}$:
		\[ 
		\begin{split}
		 \E_{S_2, z}&\left[\left|l(f_{\mathbf{S}_2}, \mathbf{z}) - l(f_{\mathbf{S}_2^{ \setminus i}} , \mathbf{z})\right|\right] =
		\sum_{\mathbf{S}_2, \mathbf{z} \sim P_{X,Y_2}} P_{X,Y_2}(\mathbf{S}_2) P_{X,Y_2}(\mathbf{z})\\
		& \times \left|l(f_{\mathbf{S}_2}, \mathbf{z}) - l(f_{\mathbf{S}_2^{ \setminus i}} , \mathbf{z})\right| \\ 
		&\geq P_{X,Y_2}(\mathbf{S}^s_2) P_{X,Y_2}(\mathbf{z}^s) \left|l(f_{\mathbf{S}_2}^s, \mathbf{z}^s) - l(f_{\mathbf{S}_2^{s, \setminus i}} , \mathbf{z}^s)\right|,
		\end{split}  
		\]
		where $\mathbf{S}_2^s, \mathbf{z}^s$ is a special draw such that $\mathbf{z}^s =  (\mathbf{x_0}, \mathbf{H})$, and where $\mathbf{S}_2^s$ includes one point at $(\mathbf{x_0}, \mathbf{H})$, one point at $(\mathbf{x_0}, \mathbf{0})$, and the rest at other values $(\mathbf{x, 0})$.
		Since the domain $\X$ is discrete, the probabilities of a special draw are:
		\begin{equation*}
		P_{X,Y_2}(\mathbf{z}^s) = \frac{1}{2} Bin(1, n, P_X(\mathbf{x_0})),
		\end{equation*}
		\begin{equation*}
		P_{X,Y_2}(\mathbf{S}^s_2) = \frac{1}{4} Bin(1, n, P_X(\mathbf{x_0}))^2  Bin(n-2, n, 1 - P_X(\mathbf{x_0})),
		\end{equation*}
		where $Bin(k, n, p_k) = \binom{n}{k} p_k^k (1-p_k)^{(n-k)}$ is a binomial coefficient, namely the probability of getting exactly $k$ successes from $n$ trials, where each trial has a probability of success $p_k$. Denote $P_{(\mathbf{S}_2^s, \mathbf{z}^s) }$ as the probability of getting a special draw, then $P_{(\mathbf{S}_2^s, \mathbf{z}^s) } = P_{X,Y_2}(\mathbf{S}^s_2) P_{X,Y_2}(\mathbf{z}^s)$.
		
		If $\mathbf{S}_2^s$ contains only two points $\mathbf{z_1} = (\mathbf{x_0}, \mathbf{H})$ and $\mathbf{z_2} = (\mathbf{x_0}, \mathbf{0})$, the  loss difference $|l(f_{\mathbf{S}_2^s}, \mathbf{z}^s) - l(f_{\mathbf{S}_2^{ s,\setminus i}} , \mathbf{z}^s)|$  evaluated at $\mathbf{z}^s$ for all $i$ will be at least $\frac{\mathbf{H}^2}{4}$. To see this, note that the optimal function's value at $\mathbf{x_0}$ is: $f_{\mathbf{S}_2^s}(\mathbf{x_0}) = \frac{\mathbf{H}}{2}$, the optimal function's value at $\mathbf{x_0}$ after we remove the first point is $f_{\mathbf{S}_2^{ s,\setminus 1}}(\mathbf{x_0}) = \mathbf{0}$, and the optimal function's value at $\mathbf{x_0}$ after removing the second point is $f_{\mathbf{S}_2^{ s,\setminus 2}}(\mathbf{x_0}) = \mathbf{H}$. Therefore, $l(f_{\mathbf{S}_2^s}, \mathbf{z}^s) = \frac{\mathbf{H}^2}{4}$, $l(f_{\mathbf{S}_2^{ s,\setminus 1}} , \mathbf{z}^s) = \mathbf{H}^2$, $l(f_{\mathbf{S}_2^{ s,\setminus 2}} , \mathbf{z}^s) = \mathbf{0}$. And we get that $|l(f_{\mathbf{S}_2^s}, \mathbf{z}^s) - l(f_{\mathbf{S}_2^{ s,\setminus 1}} , \mathbf{z}^s)| = \frac{3\mathbf{H}^2}{4}$, $|l(f_{\mathbf{S}_2^s}, \mathbf{z}^s) - l(f_{\mathbf{S}_2^{ s,\setminus 2}} , \mathbf{z}^s)| = \frac{\mathbf{H}^2}{4}$.
		As we add the rest of the points $(\mathbf{x}_i, \mathbf{0})$ to the data set $\mathbf{S}_2^s$,  the loss difference (from changing $f_{\mathbf{S}_2^s}(\mathbf{z}^s)$ to $f_{\mathbf{S}_2^{ s, \setminus i}}(\mathbf{z}^s)$) in the special draw case will only increase. Therefore for all $i$:
		\begin{equation*}|l(f_{\mathbf{S}_2^s}, \mathbf{z}^s) - l(f_{\mathbf{S}_2^{ s, \setminus i}} , \mathbf{z}^s)| \geq \frac{\mathbf{H}^2}{4}.\end{equation*}
		If we choose $\mathbf{H}$ such that $\mathbf{H} > 2 \sqrt{\lambda}\left(P_{(\mathbf{S}_2^s, \mathbf{z}^s)}\right)^{-1/2}$, then from the definition of algorithmic stability we have:
		\begin{equation*}\tilde{\beta_2} \geq \E_{S_2, z}\left[\left|l(f_{\mathbf{S}_2}, \mathbf{z}) - l(f_{\mathbf{S}_2^{ \setminus i}} , \mathbf{z})\right|\right]	\geq P_{(\mathbf{S}_2^s, \mathbf{z}^s) } \frac{\mathbf{H}^2}{4} > \lambda.\end{equation*}
		Therefore for any given $\lambda$ we get that $\left|\tilde{\beta_1} - \tilde{\beta_2} \right| > \lambda$. This proves that the hypothesis stability constants are different and completes the first part of the proof.
		
		We now need to prove that the expected Rashomon ratios are the same, which will constitute the second part of the proof.
		The Rashomon ratio for the hypothesis  space $\F_{\Omega}$ of linear models does not depend on targets and can be calculated as in (\ref{eq:ridge_volume}) for both $\mathbf{S}_1$ and $\mathbf{S}_2$. 	Therefore the expected Rashomon ratios are the same:	
		\begin{equation*}\E_{P_{X,Y_1}}[ \hat{R}_{ratio_{ \mathbf{S}_1}}(\F_{\Omega}, \theta)] = \E_{P_{X, Y_2}}[ \hat{R}_{ratio_{ \mathbf{S}_2}}(\F_{\Omega}, \theta)].\end{equation*}
		
		Thus, both halves of our proof are complete.
	\end{proof}	
	
	
	\subsection{Geometric Margin}\label{appendix:margin}
	For the parametric hypothesis space of linear models $\F_{\Omega} = \{f: f(x) = \omega^T x, \omega \in \R^p\}$ and binary classification, denote $d_{+}$ and $d_{-}$ as the shortest distances from a decision boundary to the closest points with targets $y = 1$ and $y = -1$ respectively. Then the margin $d$ is a sum of these distances $d = d_+ + d_-$ \citep{burges1998tutorial}. Moreover, for the model $f_{\hat{\omega}}$ that maximizes the margin, the margin width is $\frac{2}{\|\hat{\omega}\|_2}$.
	
	Intuitively both the Rashomon ratio and the width of the geometric margin are data-dependent and show how expressive the hypothesis space is with respect to a given data set. However, the margin depends on support vectors while the Rashomon set depends on the full data set. Theorem \ref{th:ratio_margin} summarizes this idea.

\begingroup
\def\thetheorem{\ref{th:ratio_margin}}
\begin{theorem}[Rashomon ratio is not the geometric margin]
\ThRatioMargin
\end{theorem}
\addtocounter{theorem}{-1}
\endgroup
	
	
		\begin{figure*}[t]
\hspace{0.125\textwidth}
            \begin{subfigure}[t]{0.25\textwidth}
			\includegraphics[width=\textwidth]{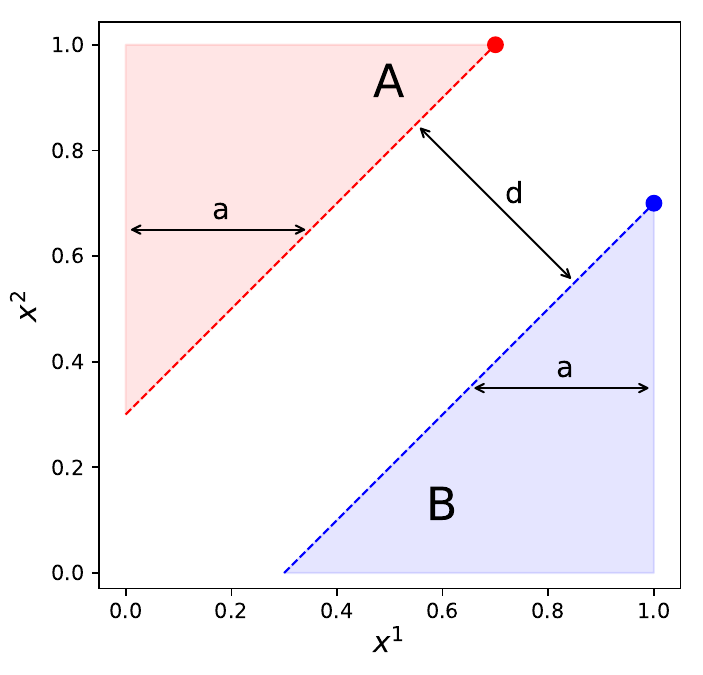}
			\caption{Structure of $S_1$}
			\label{fig:margin1}
		\end{subfigure}
		\begin{subfigure}[t]{0.25\textwidth}
			\includegraphics[width=\textwidth]{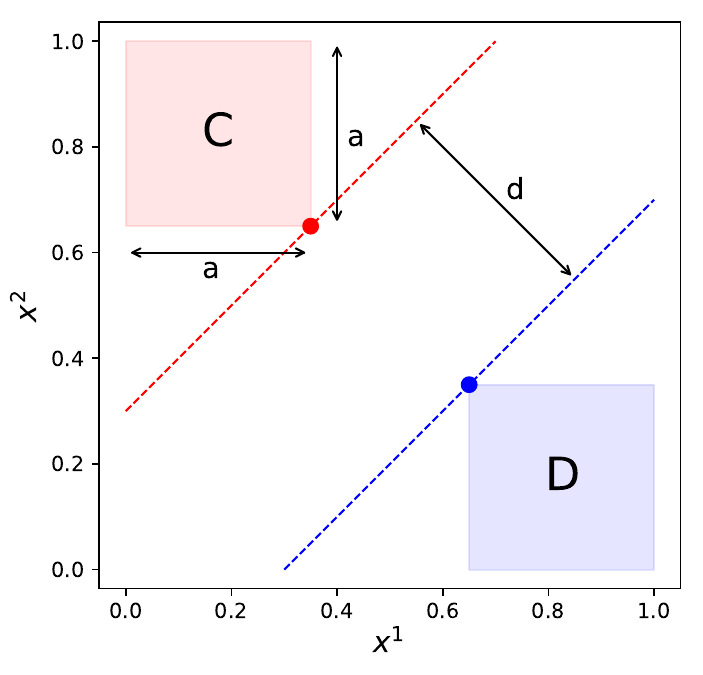}
			\caption{Structure of $S_2$}
			\label{fig:margin2}
		\end{subfigure}
		\begin{subfigure}[t]{0.25\textwidth}
			\includegraphics[width=\textwidth]{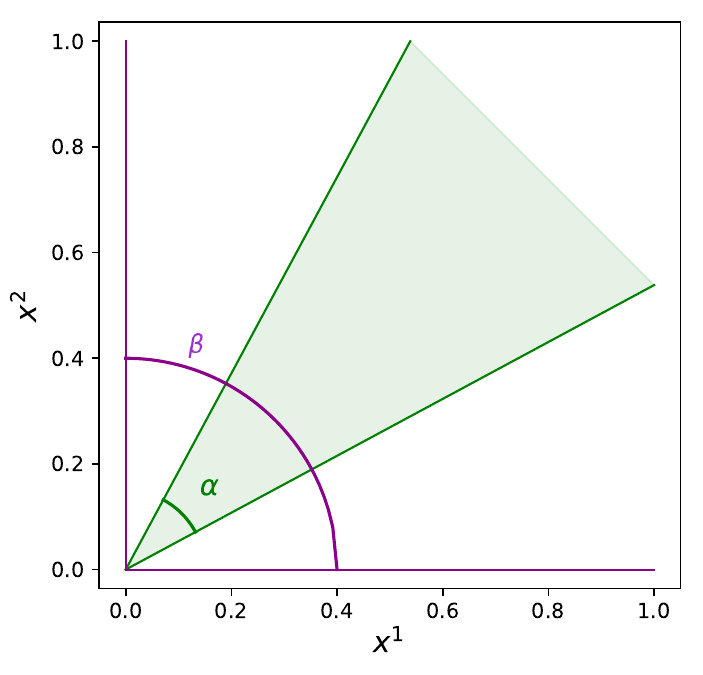}
			\caption{$\hat{R}_{ratio}(\F_{\Omega}, \theta) = \frac{\alpha}{\beta}$}
			\label{fig:margin3}
		\end{subfigure}
		
		\bigskip
		\hspace{0.125\textwidth}
		\begin{subfigure}[t]{0.25\textwidth}
			\includegraphics[width=\textwidth]{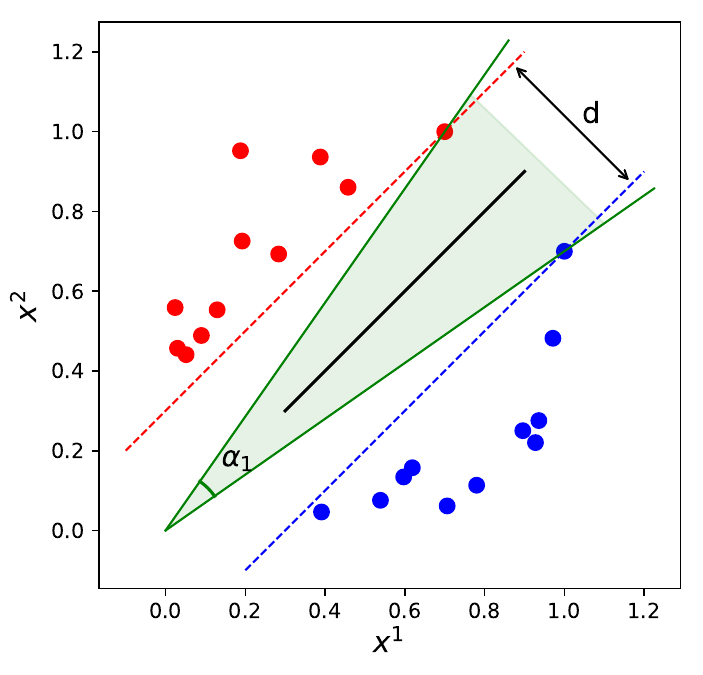}
			\caption{$ \hat{R}_{{ratio}_{S_1}}(\F_{\Omega}, 0) = \frac{\alpha_1}{\pi/2}$}
			\label{fig:margin4}
		\end{subfigure}
		\begin{subfigure}[t]{0.25\textwidth}
			\includegraphics[width=\textwidth]{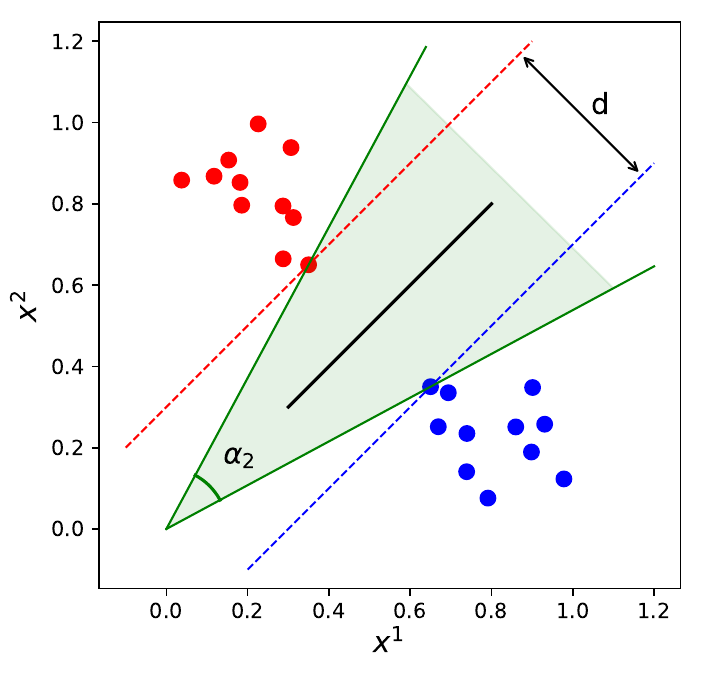}
			\caption{$ \hat{R}_{{ratio}_{S_2}}(\F_{\Omega}, 0) = \frac{\alpha_2}{\pi/2}$}
			\label{fig:margin5}
		\end{subfigure}

		\caption{An illustration of different Rashomon ratios with identical  geometric margins. (a) and (b) show the data sets $S_1$ and $S_2$ with identical margin $d$. The black line in (d) and (e) shows the optimal model, the shaded region in (c), (d), and (e) indicates the Rashomon set $\hat{R}_{set}(\F_{\Omega}, 0)$ with its boundaries represented by green lines. The hypothesis space consists of all origin-centered linear models that intersect the zero-one hypercube, where data reside. 
		(c) shows that the Rashomon ratio can be computed as a ratio of angles $\alpha$ (represents the Rashomon set) and $\beta$ (represents the hypothesis space). (d) and (e) illustrate that data sets $S_1$ and $S_2$ are represented by different angles $\alpha_1$ and $\alpha_2$ and therefore have different Rashomon ratios. Figure is best seen in color.} 
			\label{fig:margin}
	\end{figure*}

	\begin{proof}
		
		Consider two-dimensional separable data, $\X \in [0,1]^2$, and a parametrized hypothesis space of origin-centered linear models: $\F = \{\omega^T x, \omega = (k, -1), x \in \R^2, k \in \R\}$. Consider also 0-1 loss $\phi_{\omega}(x,y) = \1_{[y = sign (\omega^Tx)]}$ and an empirical risk minimizer $\hat{f} = f_{\hat{\omega}} $ that maximizes the geometric margin. Since the data are populated in a $[0,1]^2$ hypercube, as a hypothesis space we will consider all models that intersect the unit-hypercube.
		
		For some positive constant $a \in(0,1)$ that we choose later, consider the following regions of the feature space:
		\[A = \{ x^1 \in[0, 1-a), x^2 > x^1 + (1-2a) \},\]
		\[B = \{ x^1 \in(a, 1], x^2 < x^1 - (1-2a) \},\]
		\[C = \{ x^1 \in[0, a), x^2 \in (1-a, 1] \},\]
		\[D = \{ x^1 \in(1-a, 1], x^2 \in [0, a) \}.\]
		
		Construct data set $S_1$, such that $S_1 = \{(x_A,1) \cup (x_B, -1) \cup (x_{S_1}^{s_1}, 1) \cup (x_{S_1}^{s_2}, -1)\}$, where $x_A \in A$ is any sample from the region $A$, $x_B \in B$ is any sample from the region $B$, $x_{S_1}^{s_1}$ and $s_{S_1}^{s^2}$ are special points for the data set $S_1$ such that $x_{S_1}^{s_1} = [1 - 2a, 1]$ and $x_{S_1}^{s_2} = [1, 1 -2a]$. Please see  Figure \ref{fig:margin1} for details.
		
		Construct data set $S_2$, such that $S_2 = \{(x_C,1) \cup (x_D, -1) \cup (x_{S_2}^{s_1}, 1) \cup (x_{S_2}^{s_2}, -1)\}$, where $x_C \in C$ is any sample from the region $C$, $x_D \in D$ is any sample from the region $D$, $x_{S_2}^{s_1}$ and $x_{S_2}^{s_2}$ are special points for the data set $S_2$ such that $x_{S_2}^{s_1} = [a, 1-a]$ and $x_{S_2}^{s_2} = [1-a, a]$. Please see  Figure \ref{fig:margin2} for details.
		
		Note that the data sets we considered have the same width for the geometrical margin $d = \sqrt{2}(2a - 1)$ (see  Figures \ref{fig:margin1}, \ref{fig:margin2}). Now, we are left to show that the Rashomon ratios are different.
		
		For the hypothesis space of origin-centered lines we have a unique parameterization and a one-to-one correspondence between an actual model and its parameterization. Therefore, if the Rashomon set is a single connected component, an angle $\alpha$ between the two most distant models in the Rashomon set gives us some information about the size of the Rashomon set. In particular, we can compute the Rashomon ratio as a ratio of the angle $\alpha$ that represents the Rashomon set and the angle $\beta$ that corresponds to the hypothesis space as shown on  Figure \ref{fig:margin3}.  Since the hypothesis space is defined on the unit-hypercube, $\beta = \pi/2$ and for the Rashomon parameter $\theta = 0$ the Rashomon ratio is:
		\begin{equation*} \hat{R}_{ratio}(\F, 0))  = \frac{\alpha}{\beta} = \frac{ 2  \max_{f \in \hat{R}_{set}(\F_{\Omega}, 0)}\left|\arctan (f_{\hat{\omega}}) - \arctan(f_{\omega})\right|}{\pi/2}.\end{equation*}
		
		For data sets $S_1$ and $S_2$  Figures \ref{fig:margin4} and \ref{fig:margin5}   show the Rashomon set and angles $\alpha_1$ and $\alpha_2$ that represent the volume of the Rashomon set. Given the special points in the data sets we can compute $\alpha_1$ and $\alpha_2$ exactly: $\alpha_1 = 2\left(\arctan(1) - \arctan(1-2a)\right) = \frac{\pi}{2} - 2\arctan(1-2a)$ and $\alpha_2 = 2\left(\arctan(1) - \arctan\left(\frac{a}{1 - a}\right)\right) = \frac{\pi}{2} - 2\arctan\left(\frac{a}{1 - a}\right)$.  Then the Rashomon ratios difference is:
		
		\begin{equation*}
		\begin{split}
		|\hat{R}_{ratio_{S_1}}(\F, 0) &- \hat{R}_{ratio_{S_2}}(\F, 0)|  = \left|\frac{\alpha_1 - \alpha_2}{\pi/2}\right| \\
		& =\left| \frac{4}{\pi}\left(\arctan(1 - 2a) - \arctan\left(\frac{a}{1 -a}\right)\right)\right|\\
		&= \left| \frac{4}{\pi}\arctan\left(1 - \frac{4a-2}{2a^2-1}\right)\right|.
		\end{split}
		\end{equation*}
		
		Now if we choose $a \in (0,1)$ and such that $\left| \frac{4}{\pi}\arctan\left(1 - \frac{4a-2}{2a^2-1}\right)\right| > \lambda$, then the Rashomon ratio difference $|\hat{R}_{ratio_{S_1}}(\F, 0) - \hat{R}_{ratio_{S_2}}(\F, 0)|$ is at least $\lambda$.
		

	\end{proof}	
	
	
\subsection{Empirical Local Rademacher Complexity}\label{appendix:rademacher}
	The empirical Rademacher complexity is another complexity measure of the hypothesis space. Following \citet{bartlett2005local}, for binary classification we define it as follows.
	\begin{definition}[Empirical Rademacher complexity]	
	Given a data set $S$, and a hypothesis space $\F$ of real-valued functions, the empirical Rademacher complexity of $\F$ is defined as:
	\begin{equation*} \hat{R}_n^{S}(\F)={\frac {1}{n}}\mathbb {E} _{\sigma }\left[\sup _{f\in F}\sum _{i=1}^{n}\sigma _{i}f(z_{i})\right],\end{equation*}
	where $ \sigma _{1},\sigma _{2},\dots ,\sigma _{n}$ are independent random variables drawn from the Rademacher distribution i.e. $P(\sigma _{i}=+1)=P(\sigma _{i}=-1)=1/2$ for $ i=1,2,\dots ,n$.
	\end{definition}\label{def:rademacher}
	
	Since we are interested only in models that are inside the Rashomon set, we will consider local empirical Rademacher complexity \citep{bartlett2005local}, which is defined using the Rashomon set $\hat{R}_{set}(\F, \theta)$. In the following theorem, we provide a simple example to show the discrepancy between the two measures.
	
	

\begingroup
\def\thetheorem{\ref{th:ratio_rademacher}}
\begin{theorem}[Rashomon ratio is not the local Rademacher complexity]
\ThRatioRademacher
\end{theorem}
\addtocounter{theorem}{-1}
\endgroup
	
			\begin{figure*}[t]
		\centering
		\begin{subfigure}[t]{0.25\textwidth}
			\includegraphics[width=\textwidth]{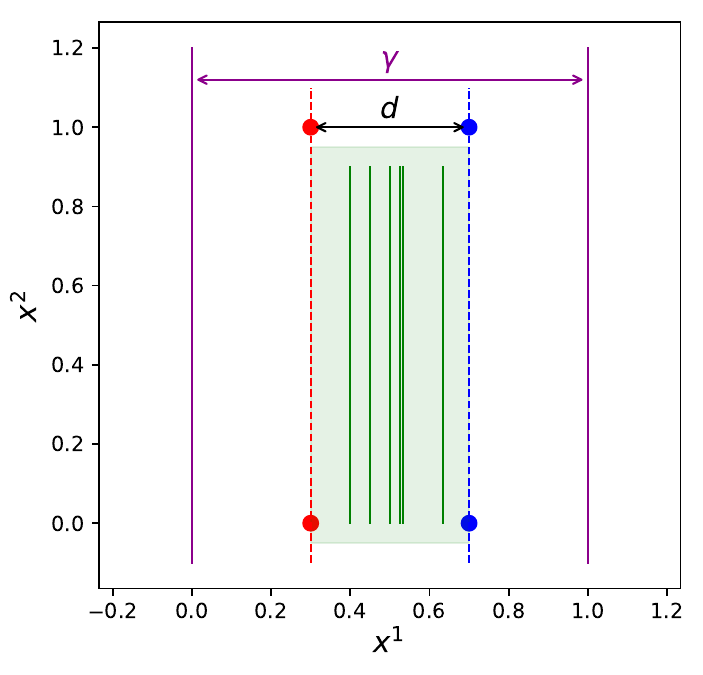}
			\caption{$\hat{R}_{ratio}(\F_{\Omega}, \theta)=\frac{d}{\gamma}$}
			\label{fig:rademacher1}
		\end{subfigure}
		\begin{subfigure}[t]{0.25\textwidth}
			\includegraphics[width=\textwidth]{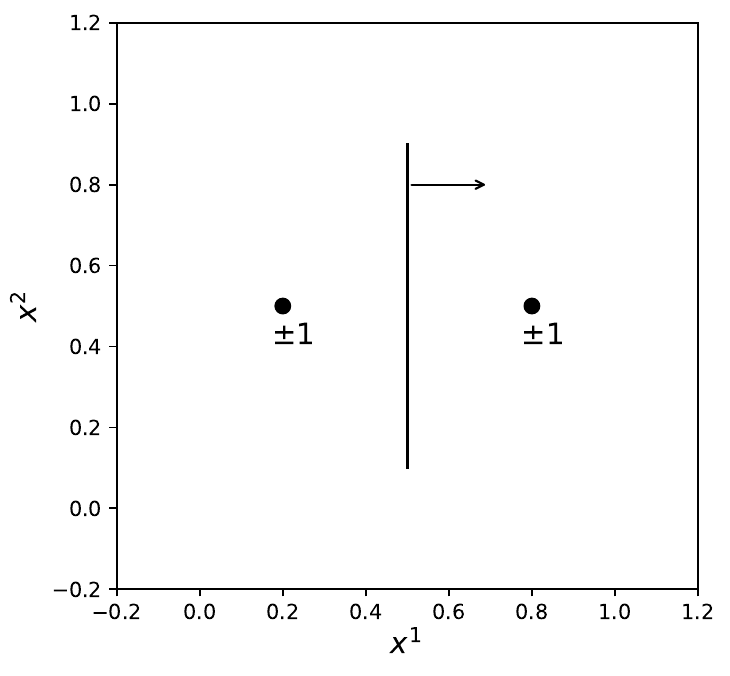}
			\caption{Toy data set}
			\label{fig:rademacher4}
		\end{subfigure}

		\begin{subfigure}[t]{0.25\textwidth}
			\includegraphics[width=\textwidth]{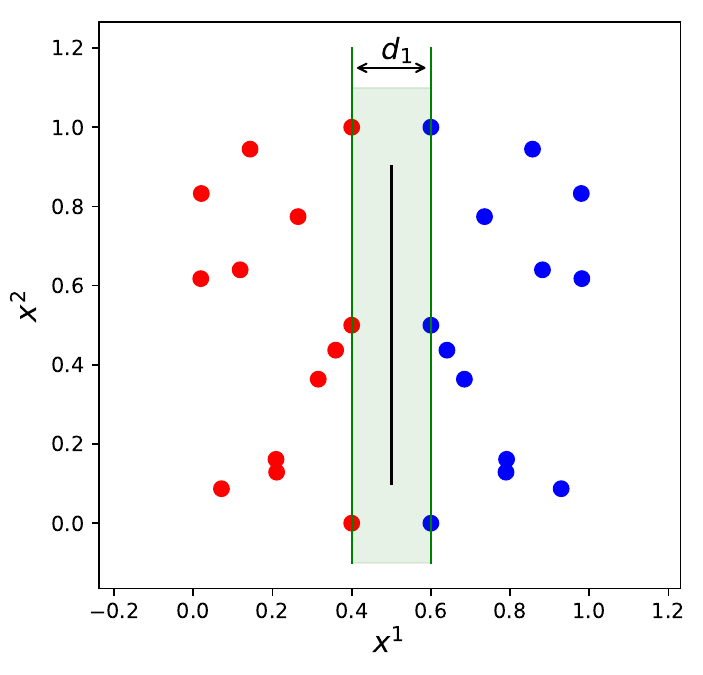}
			\caption{$\hat{R}_{{ratio}_{S_1}}(\F_{\Omega}, 0)=d_1$}
			\label{fig:rademacher2}
		\end{subfigure}
		\begin{subfigure}[t]{0.25\textwidth}
			\includegraphics[width=\textwidth]{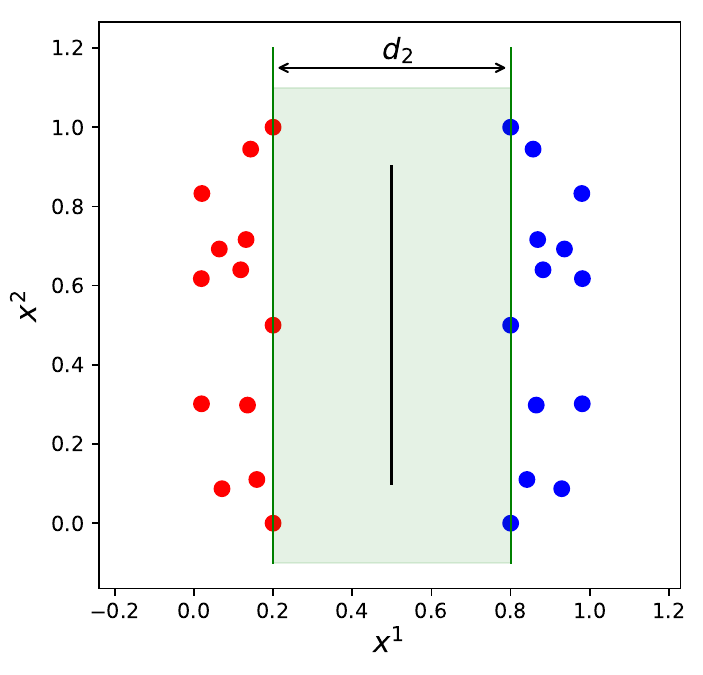}
			\caption{$\hat{R}_{{ratio}_{S_2}}(\F_{\Omega}, 0)=d_2$}
			\label{fig:rademacher3}
		\end{subfigure}
		
		\caption{An illustration of different Rashomon ratios with equivalent empirical local Rademacher complexities. Black line shows the optimal model, shaded region indicates the Rashomon set $\hat{R}_{set}(\F_{\Omega}, 0)$ with its models represented by green lines, the magenta color indicates boundaries of the hypothesis space. (a) The projected minimal distance $d$ is equivalent to the volume of the Rashomon set. (b) A toy data set that illustrates that the empirical local Rademacher complexity is zero for models in the Rashomon set. (c) data set $S_1$, and (d) data set $S_2$ illustrate symmetric separable data sets with different Rashomon ratios. Best seen in color.}\label{fig:rademacher}
	\end{figure*}
	
	    \begin{proof}
		
		Consider two-dimensional separable symmetric data, $\X \in [0,1]^2$, $\Y=\{0,1\}$, 0-1 loss $\phi_f(x,y) = \1_{[y = sign f(x)]}$ with empirical risk minimizer $\hat{f}$, and a hypothesis space $\F_{\Omega}$ of decision stumps based on the first feature, where for $f\in \F_{\Omega}$: $f=1$ if $x^1>\omega$, $\omega\in \R$, $f=0$ otherwise. We have a one-to-one correspondence between a function and its threshold parameter $\omega$. Therefore, if the Rashomon set is a single connected component, we can compute the volume of the Rashomon set in  parameter space by computing the difference between the largest and smallest threshold values of models within the Rashomon set, as illustrated in  Figure \ref{fig:rademacher1}. For $\theta = 0$, the difference between the largest and the smallest threshold values will be equivalent to the minimal distance between points of opposite classes projected onto the first feature $d = \min_{x_i, x_j: y_i\neq y_j}|PR_1(x_{i}) - PR_1(x_{j})|$, where $PR_1$ is the projection of point $x$ onto first feature.
		
		For the hypothesis space, we consider all decision stumps in the first dimension that are in the segment $[0,1]$, where data are populated. The difference in thresholds for the hypothesis space is $\beta = 1$ and therefore  $\V(\F_{\Omega}) = 1$. For $\theta = 0$, the volume of the Rashomon set will be equivalent to $d$---the projected minimal distance between points of opposite classes, and have that $\V(\hat{R}_{set}(\F_{\Omega}, 0)) = d$ and $\hat{R}_{ratio}(\hat{R}_{set}(\F_{\Omega}, 0)) = \frac{d}{1}=d$. Now consider any two separable symmetric data sets  $S_1$, $S_2$ with different projected minimal distances $d_1$ and $d_2$, such that $|d_1 - d_2| > \lambda$. (Please see  Figure \ref{fig:rademacher2} and \ref{fig:rademacher3} for details of the data sets $S_1$ and $S_2$.) Consequently we get that: 
		\[\left|\hat{R}_{ratio_{S_1}}(\F_{\Omega}, 0) - \hat{R}_{ratio_{S_2}}(\F_{\Omega}, 0)\right| = |d_1 - d_2| > \lambda.\]
		
		For a separable symmetric data $S$ and 0-1 loss function, the Rashomon set $\hat{R}_{set}(\F_{\Omega}, 0)$ contains all models that separate data in the same way. Therefore the Rademacher complexity of the Rashomon set is $\hat{R}_n^{S}\left(\hat{R}_{set}(\F_{\Omega}\right)$ is:
		\begin{equation*}
		    \begin{split}
		        \hat{R}_n^{S}\left(\hat{R}_{set}(\F_{\Omega}, 0)\right) &= {\frac {1}{n}}\mathbb {E} _{\sigma }\left[\sup _{f\in \hat{R}_{set}(\F_{\Omega}, 0)}\sum _{i=1}^{n}\sigma_{i}f(x_{i})\right] \\ 
		        &= {\frac {1}{n}}\mathbb {E} _{\sigma }\left[\sum _{i=1}^{n}\sigma_{i} \hat{f}(x_{i})\right] = 0,
		    \end{split}
		\end{equation*}
		where in the penultimate equality we have used the fact that, in the case of separable data and $\theta=0$, all models in the Rashomon set will perform identically on any permutation of the labels.

		Equality of the empirical Rademacher complexity of the optimal model to zero follows from the symmetric data considered and symmetrical patterns of all possible target assignments. For example, for the toy data set in Figure \ref{fig:rademacher4}:  
		$\hat{R}_n^{S}\Big( \hat{R}_{set}(\F_{\Omega}, 0) \Big) = \frac{1}{2}\frac{1}{4}\bigg[ \Big(\hat{f}(x_1) + \hat{f}(x_2)\Big) + \Big(\hat{f}(x_1) - \hat{f}(x_2)\Big) + \Big(-\hat{f}(x_1) + \hat{f}(x_2)\Big) + \Big(-\hat{f}(x_1) - \hat{f}(x_2)\Big) \Big) \bigg]= 0$.
		
		Since both $S_1$ and $S_2$ are separable and symmetric we get that:
		$$\hat{R}_n^{S_1}\left(\hat{R}_{set}(\F_{\Omega}, 0)\right) = 0= \hat{R}_n^{S_2}\left(\hat{R}_{set}(\F_{\Omega}, 0)\right).$$
		
		
		
	\end{proof}
	

\section{Proofs for Generalization Results}

\subsection{Proof of Theorem \ref{th:approximating_finite_1}}\label{appendix:proof_approximate_finite}

We recall and provide the proof of Theorem \ref{th:approximating_finite_1} after Proposition \ref{stmt:anc_rset_true_empirical} that we use to prove Theorem \ref{th:approximating_finite_1}.

    Given a parameter $\eta \geq 0$, we call the Rashomon set with restricted empirical risk an \textit{anchored Rashomon set}:
    	\[\hat{R}_{set}^{anc}(\F, \eta) := \{ f \in \F: \hat{L}(f) \leq \eta  \}.\]
    	We define also the \textit{true anchored Rashomon set} based on the true risk as follows:
    	\[R_{set}^{anc}(\F,\eta) := \{ f \in \F: L(f) \leq \eta  \}.\]

   \newcommand{\StmtAncRsetTrueEmpirical}
    {
    For a  loss $l$ bounded by $b$ and for any $\epsilon >0$, and for a fixed $f$, if $f \in R_{set}^{\anc}(\F, \eta)$ then with probability at least $1 - e^{-2n(\epsilon /b)^2}$ with respect to the random draw of training data, 
    \[f \in \hat{R}_{set}^{\anc}(\F, \eta+\epsilon).\]
    }
    
    \begin{proposition} [True  anchored Rashomon set is close to empirical]\label{stmt:anc_rset_true_empirical} 
    \StmtAncRsetTrueEmpirical
    \end{proposition}

    \begin{proof}
For a fixed $f \in R_{set}^{\anc}(\F, \eta)$ by Hoeffding's inequality:
\begin{equation*}
    \begin{split}
    P\left[   \hat{L}(f) - L(f) >\epsilon \right] &= P\left[ \frac{1}{n}\sum_{i=1}^n l(f, z_i) - \E\left[ l(f, z)\right]  > \epsilon \right] \\
& \leq e^{-2n(\epsilon /b)^2}.
    \end{split}
\end{equation*}
Therefore, with probability at least $1 - e^{-2n(\epsilon /b)^2}$ with respect to the random draw of data, $\hat{L}(f)  - L(f) \leq \epsilon$.

Since $f \in R_{set}^{\anc}(\F, \eta)$, then by definition of the Rashomon set, $L(f) \leq \eta$. Combining this with Hoeffding's inequality, we get that with probability at least $1 - e^{-2n(\epsilon /b)^2}$:
\[\hat{L}(f) \leq L(f) +  \epsilon \leq \eta + \epsilon,\]
therefore $f \in \hat{R}_{set}^{\anc}(\F, \eta+\epsilon).$
\end{proof}

    Proposition \ref{stmt:anc_rset_true_empirical} is based on the same intuition as Lemma 23 in the work of \citet{fisher2018model}, which is used to bound the probability with which a given model is not in the empirical Rashomon set; this is used in a proof of a bound for model class reliance. We use the proposition to indicate the probability with which the empirical anchored Rashomon set is as close as possible to the true anchored Rashomon set for a given model.



\begingroup
\def\thetheorem{\ref{th:approximating_finite_1}}
\begin{theorem}[The advantage of a true  Rashomon set]
 Consider finite hypothesis spaces $\F_1$ and $\F_2$, such that $\F_1\subset\F_2$. Let the loss $l$ be bounded by $b$, $l(f_2,z)\in[0,b] \;\;\forall f_2\in\F_2, \forall z\in\mathcal{Z}$.
        Define an optimal function $\fts\in \textrm{\rm argmin}_{f_2\in \F_2} L(f_2)$. Assume that the true Rashomon set includes a function from $\F_1$, so there exists a model $\tilde{f}_1\in\F_1$ such that $\tilde{f}_1\in R_{set}(\F_2,\gamma)$. (Note that we do not know $\tilde{f}_1$.)
        In that case, for any $\epsilon > 0$ with probability at least $1 - \epsilon$ with respect to the random draw of data:
        \begin{equation*}
        L(f_2^*) - b\sqrt{\frac{\log |\F_1| + \log 2/\epsilon}{2n}} \leq \hat{L}(\hat{f}_1) \leq L(f_2^*) + \gamma+b\sqrt{\frac{ \log 1/\epsilon}{2n}},
        \end{equation*}
        where $\hat{f}_1 \in \argmin_{f_1\in\F_1}\hat{L}(f_1)$. (Unlike $\tilde{f}_1$, we do know $\hat{f}_1$ because we can calculate it.)
\end{theorem}
\addtocounter{theorem}{-1}
\endgroup

\begin{proof}
 \textbf{Lower bound.} We apply the union bound and Hoeffding's inequality. The result is that with probability at least $1-\epsilon$ for every $f_1 \in \F_1$ we have, for finite hypothesis space $\F_1$:
\begin{equation}\label{eq:th41proof}
L(f_1) \leq \hat{L}(f_1) + b\sqrt{\frac{\log |\F_1| + \log 2/\epsilon}{2n}}.    
\end{equation}
Combining this Occam's razor bound with the definition of  $f_2^*\in\arg\min_{f\in\F_2}L(f)$ we get that, under the same conditions:
\[L(f_2^*) \leq L(\hat{f}_1) \leq \hat{L}(\hat{f}_1) + b\sqrt{\frac{\log |\F_1| + \log 2/\epsilon}{2n}}.\]

\textbf{Upper bound.} 
By the assumption of the theorem, we have that $L(\tilde{f}_1) \leq L(f_2^*) + \gamma$. Also, by the definition of an optimal model $f_1^*$, $L(f^*_1) \leq L(\tilde{f}_1)$. Combining these, we get that $L(f_1^*) \leq L(\tilde{f}_1)\leq L(f_2^*) + \gamma$. Thus $f_1^*$ is in the true Rashomon set of $\F_2$ with parameter $\gamma$. Alternatively, $f_1^*$ is in the true anchored Rashomon set of $\F_2$ with parameter $\eta = L(f_2^*) + \gamma$, $f_1^* \in R_{set}^{anc}(\F_2, \eta)$. Following Proposition \ref{stmt:anc_rset_true_empirical}, we have that for any $\epsilon_1 > 0$ with probability at least $1 - e^{-2n(\epsilon_1 /b)^2}$ with respect to the random draw of data, $f_1^*$ is in the slightly larger anchored Rashomon set $\hat{R}_{set}^{anc}(\F_2, \eta + \epsilon_1)$, and therefore, with high probability, $\hat{L}(f_1^*) \leq \eta + \epsilon_1$. Or alternatively, by setting $\epsilon = e^{-2n(\epsilon_1 /b)^2}$ we get that for any $\epsilon > 0$ with probability at least $1 - \epsilon$, we have $\hat{L}(f_1^*) \leq \eta + b\sqrt{\frac{\log 1/\epsilon}{2n}}$. Further, by definition of the empirical risk minimizer and given that $\eta = L(f_2^*) + \gamma$ we get:
\begin{equation*}\label{eq:helper1}
 \hat{L}(\hat{f}_1)\leq   \hat{L}(f_1^*) \leq L(f_2^*) + \gamma + b\sqrt{\frac{\log 1/\epsilon}{2n}}.
\end{equation*}
Combining the previous two equations yields the statement of the theorem.

\end{proof}

\subsection{Proof of Theorem \ref{th:approximating_finite_3} via Lemma \ref{lemma:true_ratio_size_estimation}}\label{appendix:lemma41}

 Theorem \ref{th:approximating_finite_3} follows directly from Lemma \ref{lemma:true_ratio_size_estimation} below and Theorem \ref{th:approximating_finite_1}, which guarantees that with high probability, the sampled space $\F_1$ will contain at least one model from the true anchored Rashomon set. 
 
\newcommand{\LemmaTrueRatioSizeEstimation}
    {
    For a finite hypothesis space $\F_2$ of size $|\F_2|$, we will draw $|\F_1|$ functions uniformly without replacement from $\F_2$ to form $\F_1$. 
    If the true Rashomon ratio of the hypothesis space $\F_2$ is at least  
    $$R_{ratio}(\F_2,\gamma)\geq 
    1 - \epsilon^{\frac{1}{|\F_1|}}$$ 
    then with probability at least $1-\epsilon$ with respect to the random draw of functions from $\F_2$ to form $\F_1$, the Rashomon set contains at least one model $\tilde{f}_1$ from $\F_1$.
    }

    \begin{lemma}\label{lemma:true_ratio_size_estimation}
    \LemmaTrueRatioSizeEstimation
    \end{lemma}

\begin{proof}

  The probability of an individual sample from $\F_2$ missing the true Rashomon set is $1 - R_{ratio}(\F_2,\gamma)$. The probability if this happening $|\F_1|$ times independently is $(1 - R_{ratio}(\F_2,\gamma))^{|\F_1|}$. Thus, for any $\epsilon > 0$, if the Rashomon ratio is at least $R_{ratio}(\F_2,\gamma)\geq 1 - \epsilon^{\frac{1}{|\F_1|}}$, the probability $p_w$ of sampling, with replacement, at least one hypothesis from $R_{ratio}(\F_2,\gamma)$ is:
 \begin{equation*}
p_w  = 1- \left(1 - R_{ratio}(\F_2,\gamma)\right)^{|\F_1|}\geq 1 - \epsilon.
\end{equation*}
Let $p_i$ be the probability, under sampling without replacement, that samples $1 \ldots i$ have missed $R_{ratio}(\F_2,\gamma)$. $p_1 = 1 - R_{ratio}(\F_2,\gamma)$, and $p_i \leq (1 - R_{ratio}(\F_2,\gamma))^i$. The probability, under sampling without replacement, that at least one hypothesis from $R_{ratio}(\F_2,\gamma)$ in $\F_1$ is therefore $1-p_{|\F_1|} \geq p_w$. Thus the statement of the lemma holds with probability at least $1-\epsilon$.
\end{proof}

Let us recall Theorem \ref{th:approximating_finite_3}:
\begingroup
\def\thetheorem{\ref{th:approximating_finite_3}}
\begin{theorem}[Example of the advantage of a large true Rashomon set]
    \ThApproximatingFiniteThree
\end{theorem}
\addtocounter{theorem}{-1}
\endgroup
\begin{proof}   
According to the Lemma \ref{lemma:true_ratio_size_estimation}, for any $\epsilon > 0$ with probability at least $1 - \epsilon$ with respect to the random draw of functions, if the Rashomon set it at least $R_{ratio}(\F_2,\gamma)\geq 
    1 - \epsilon^{\frac{1}{|\F_1|}}$, then the Rashomon set contains at least one model $\tilde{f}$ from $\F_1$. In that case, according to Theorem \ref{th:approximating_finite_1} with probability at least $1 - \epsilon$ with respect to the random draw of data, the bound \eqref{eq:theorem43} holds. Therefore with probability at least $(1 - \epsilon)^2$ we get the statement of the theorem.
\end{proof}

\subsection{Proof of Theorem \ref{th:existence_multiple}}

\begingroup
\def\thetheorem{\ref{th:existence_multiple}}
\begin{theorem}[Existence of multiple simpler models]
    \ThExistenceMultiple
\end{theorem}
\addtocounter{theorem}{-1}
\endgroup

\begin{proof}
Starting from the packing  number of the Rashomon set $\mathcal{B}(\hat{R}_{set}(\F_2, \theta), 2\delta)$, there exists a $2\delta$-packing $\Xi =\{\xi_1, ..., \xi_k | \xi_i \in \hat{R}_{set}(\F_2, \theta)\}$ such that $\|\xi_i - \xi_j\|_p > 2\delta$ for all $i\neq j$. On the other hand, for each $\xi_i \in \hat{R}_{set}(\F_2,\theta)$ there exists $\bar{f}_1^i \in \F_1$ such that $\|\xi_i - \bar{f}_1^i\|_p \leq \delta$ (this is the assumption that $\F_1$ serves as a good cover for $\F_2$). Therefore for each ball center $\xi_i$ in the packing  there is a distinct model $\bar{f}_1^i$ from the simpler hypothesis space $\F_1$. Thus, the Rashomon set contains at least $B = \mathcal{B}(\hat{R}_{set}(\F_2, \theta), 2\delta)$ models from  $\F_1$. 

The generalization bound follows \citet{bartlett2002rademacher}.
\end{proof}

\subsection{Examples of function approximation in different hypothesis spaces}\label{appendix:approximatingtable}

Table \ref{table:approximation} shows examples of function classes where good approximating sets occur. More specifically, Table \ref{table:approximation} describes classes of functions $\F_2$ that can be approximated with functions from classes $\F_1$ within $\delta$ using a specified norm. 

	\begin{table*}[t]	
	\caption {Examples of function approximation in different hypothesis spaces: a function from space $\F_1$ approximates a function in space $\F_2$ with given guarantee $\delta$. }
		\label{table:approximation}
		\centering
		\small
			\begin{tabular}{ |p{0.25\textwidth}|p{0.25\textwidth}|p{0.3\textwidth}|p{0.1\textwidth}| } 
				\hline
				$\F_2$ & $\F_1$ &  $\delta$ (depends on parameters in bounds below)& Source\\\hline 
				$f \in L_{\infty}(\Omega)$,  \newline$\|f\|_{\infty} \in [m, M]$ & $s_N \in S(\Omega)$,  \newline $s_N$---piecewise constant, \newline $N$---number of constants&  $\|f - s_N\|_{\infty} \leq \frac{M - m}{2N}$& \citet{devore1998nonlinear, davydov2011algorithms} \\\hline 
				$f \in W_p^1(\Omega)$, $1 \leq p \leq \infty$, \newline where $W_p^1$ is a Sobolev space & $s_{\Delta}(f) \in S(\Omega) $, \newline  $s_{\Delta}$---piecewise constant,\newline $\Delta$---fixed partition, \newline $\Omega=(0,1)^d$, \newline $N$---number of constants  &$\|f - s_{\Delta}(f)\|_{p} \leq C N^{-1/d}|f|_{W_p^1{\Omega}}$ &\citet{davydov2011algorithms}\\\hline
				$f \in\{ x^k$, $k \in N$\} & $P(n)$---polynomials of degree \newline at most $n \in N$ & $\|f - P(n)\|_{\infty} \leq \frac{1}{2^{k-1}} \sum_{j > (n+k)/2} \binom{k}{j}$&\citet{newman1976approximation}\\\hline
				$f\in C[0,1]$ is a non-constant symmetric
				boolean function  on $x_1$,..,$x_n$ &$P(d)$---algebraic polynomials 
				of \newline degree $d$& $\|f - P(d) \|_{\infty} \leq \mathcal{O}(\sqrt{n(n - \Gamma(f))})$&\citet{paturi1992degree}\\\hline
				$f\in Lip_{M}(\alpha)$,  $f$ is Lipschitz \newline continuous with constant $M$&$N_n:[a,b]\rightarrow \R$ is a feedforward neural network with one layer and bounded, monotone and odd \newline defined activation function, $n\in \mathbb{N}$&$\sup_{x\in[a,b]}|f(x)-N_n(x)|\leq$ $\frac{5M}{2}\left(\frac{b-a}{n}\right)^{\alpha}$&\citet{cao2008estimate}\\\hline
				$f\in L_p(I)$, where $I\subset \mathbb{R}^d$ is a cube in $\mathbb{R}^d$, $\|\cdot\|_{W^{r}(L_p(I))}$---Sobolev semi norm &$P_r$---space of polynomials of  order $r$ in $d$, constant $C$ depends on $r$ &$\inf_{p\in P_r}\|f-p\|_{L_p(I)}\leq C|I|^{r/d}|f|_{W^{r}(L_p(I))}$&\citet{devore1998nonlinear}\\\hline
			\end{tabular}

	\end{table*}

\section{Data Set Descriptions}\label{apendix:more_data}
	
We provide a description of the data sets used in our experiments in Table \ref{table:datasets}. All of them were downloaded from the UCI Machine Learning Repository \citep{Dua:2017}. We show the number of features in each data set, sizes of the data set and any preprocessing steps that we used mainly to convert data to binary classification. For each data set, we performed cross-validation over ten folds for data sets with more than 200 points and over five folds for data sets with less than 200 points. We reserve one fold for testing, one for validation (e.g., hyper-parameter optimization) and the rest for training. All of the real-valued data sets were normalized to fit the unit-cube, and we did not standardize the data. During data processing, we omitted data records with missing values. We also omitted non-numerical features (e.g., date or text) when there was not a natural way to convert them to categorical features.

	\begin{table*}[t]
	\caption{Classification data sets description and processing notes.}
		\label{table:datasets}
		
		\centering
		\small
			
		\resizebox{\textwidth}{!}{%
			\begin{tabular} { p{0.19\textwidth}  p{0.1\textwidth} p{0.12\textwidth} p{0.12\textwidth} p{0.4\textwidth}} 
				\hline
Data Set Name           &Type of Features&Number of Features&Number of Data Points &Processing notes                               \\ \hline
\href{https://archive.ics.uci.edu/ml/datasets/MONK\%27s+Problems}{Monks-1}                &Binary       &15&556&                                               \\
\href{https://archive.ics.uci.edu/ml/datasets/MONK\%27s+Problems}{Monks-2}            &Binary       &15&601&                                               \\
\href{https://archive.ics.uci.edu/ml/datasets/MONK\%27s+Problems}{Monks-3}              &Binary       &15&554&                                               \\
\href{https://archive.ics.uci.edu/ml/datasets/congressional+voting+records}{Voting}                 &Binary       &16&232&                                               \\
\href{https://archive.ics.uci.edu/ml/datasets/SPECT+Heart}{SPECT}                  &Binary       &22&267&                                               \\
\href{https://archive.ics.uci.edu/ml/datasets/Tic-Tac-Toe+Endgame}{Tic-tac-toe}            &Binary       &27&958&                                               \\
\href{https://archive.ics.uci.edu/ml/datasets/Hayes-Roth}{Hayes-Roth}             &Binary       &12&160&Considered class 1 versus classes 2 and 3 \\
\href{https://archive.ics.uci.edu/ml/datasets/Nursery}{Nursery-1}              &Binary       &27&8586&Considered classes not\_recom and priority\\
\href{https://archive.ics.uci.edu/ml/datasets/Nursery}{Nursery-2}              &Binary       &27&8310&Considered classes priority and spec\_prior\\
\href{https://archive.ics.uci.edu/ml/datasets/mushroom}{Mushroom}               &Binary       &117&8124&                                               \\
\href{https://archive.ics.uci.edu/ml/datasets/Breast\%2BCancer}{Breast Cancer}          &Binary       &43&286&                                               \\
\href{https://archive.ics.uci.edu/ml/datasets/car+evaluation}{Car Evaluation}         &Binary       &21&1728&     Converted to one vs all problem: class 1  versus all others                                  \\                   
\href{https://archive.ics.uci.edu/ml/datasets/primary+tumor}{Primary Tumor}          &Binary       &31&336&Converted to binary classification by considering classes 1, 2, 3, 4, 22, 10 versus all others\\
\href{http://archive.ics.uci.edu/ml/datasets/mammographic+mass}{Mammographic Masses}    &Binary       &25&830&                                               \\
\href{https://archive.ics.uci.edu/ml/datasets/Website+Phishing}{Phishing}          &Binary       &23&1353&    Considered classes 0 and 1 versus class -1                                           \\
\href{http://archive.ics.uci.edu/ml/datasets/balance+scale}{Balance}          &Binary       &20&576&Considered classes L and R \\
\href{https://archive.ics.uci.edu/ml/datasets/wine}{Wine}                   &Real         &13&130&     Considered classes 0 and 1                                          \\
\href{https://archive.ics.uci.edu/ml/datasets/iris}{Iris}                   &Real         &4&100&   Considered classes versicolour and viginica                                             \\
\href{https://archive.ics.uci.edu/ml/datasets/Breast+Cancer+Wisconsin+(Diagnostic)}{Breast Cancer Wisconsin}&Real         &30&569&                                               \\
\href{https://archive.ics.uci.edu/ml/datasets/Breast+Cancer+Coimbra}{Breast Cancer Coimbra}  &Real         &9&116&                                               \\
\href{http://archive.ics.uci.edu/ml/datasets/Optical+Recognition+of+Handwritten+Digits}{Digits 0-4}             &Real         &64&363&        Classes 0 and 4 considered only                                                    \\
\href{http://archive.ics.uci.edu/ml/datasets/Optical+Recognition+of+Handwritten+Digits}{Digits 6-8}             &Real         &64&355&          Classes 6 and 8 considered only                                                  \\
\href{https://stats.idre.ucla.edu/stat/data/binary.csv}{Student}                &Real         &3 & 400&                                               \\
\href{https://archive.ics.uci.edu/ml/datasets/banknote+authentication}{Banknote}               &Real         &4&1372&                                               \\
\href{https://archive.ics.uci.edu/ml/datasets/Crowdsourced+Mapping}{Mapping}                &Real         &28&10545& Converted to one vs all problem: class forest versus all others                                  \\
\href{https://archive.ics.uci.edu/ml/datasets/Wireless+Indoor+Localization}{Wifi Localization}      &Real         &7&1000&   Considered classes that represent rooms 2 and 3                                            \\
\href{http://archive.ics.uci.edu/ml/datasets/vertebral+column}{Column 2C}                &Real         &6&310&                                               \\
\href{https://archive.ics.uci.edu/ml/datasets/default+of+credit+card+clients}{Credit Card}            &Real         &23&30000&                                               \\
\href{https://archive.ics.uci.edu/ml/datasets/Planning+Relax}{Planing Relax}     &Real         &12&182&                                               \\
\href{https://archive.ics.uci.edu/ml/datasets/Diabetic+Retinopathy+Debrecen+Data+Set}{Diabetic Retinopathy}   &Real         &19&1151&                                               \\
\href{https://archive.ics.uci.edu/ml/datasets/Haberman\%27s+Survival}{Survival}               &Real         &3&306&                                               \\
\href{https://archive.ics.uci.edu/ml/datasets/Skin+Segmentation}{Skin Segmentation}      &Real         &3&245057&                                               \\
\href{https://archive.ics.uci.edu/ml/datasets/HTRU2}{HTRU\_2}                &Real         &8&17898&                                               \\
\href{https://archive.ics.uci.edu/ml/datasets/MAGIC+Gamma+Telescope}{Magic}             &Real         &10&19020&                                               \\
\href{https://archive.ics.uci.edu/ml/datasets/seeds}{Seeds}                  &Real         &7&140&   Considered classes 1 and 2 \\
\href{https://archive.ics.uci.edu/ml/datasets/EEG+Eye+State}{Eye State}          &Real         &14&14980\\
\href{http://yann.lecun.com/exdb/mnist/}{MNIST 0-1}              &Real         &784&13738& Considered classes 0 and 1                                              \\
\href{http://yann.lecun.com/exdb/mnist/}{MNIST 4-9}              &Real         &784&12752&   Considered classes 4 and 9   		\\\hline		                        	
			\end{tabular}
		}
		
	\end{table*}		

Additionally, we performed experiments on twelve synthetic binary classification data sets. These data sets have two real features and represent different geometrical concepts for two-dimensional classification (e.g., large and small margins, concentric circles, half moons, etc.) as in Figure \ref{fig:app_syn_datasets}. 
Results and implications for synthetic data sets are consistent with those on the UCI data sets. 


	\begin{figure*}[t]
		\centering
		\begin{tabular}{cc}
			\includegraphics[width=0.22\textwidth]{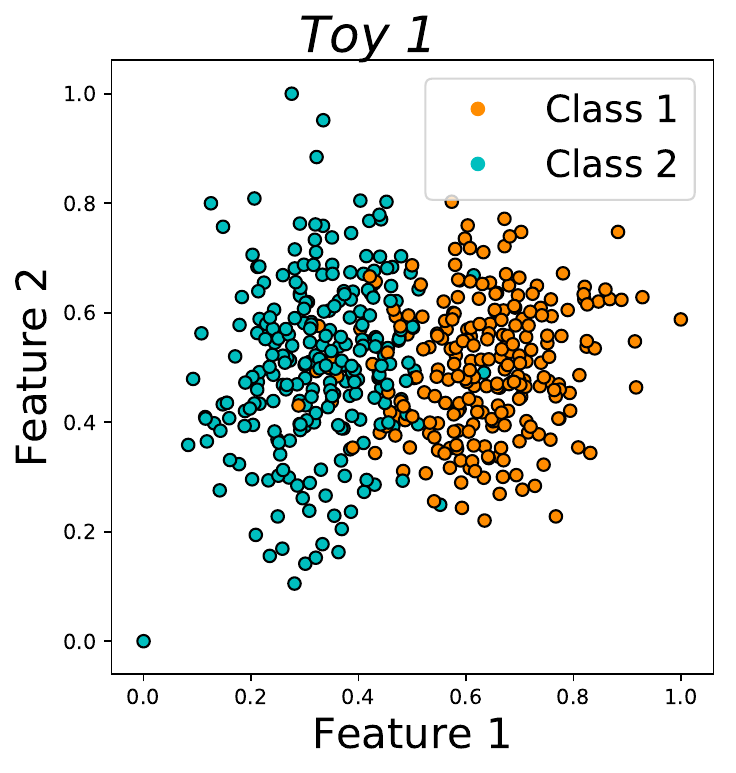}\quad	
			\includegraphics[width=0.22\textwidth]{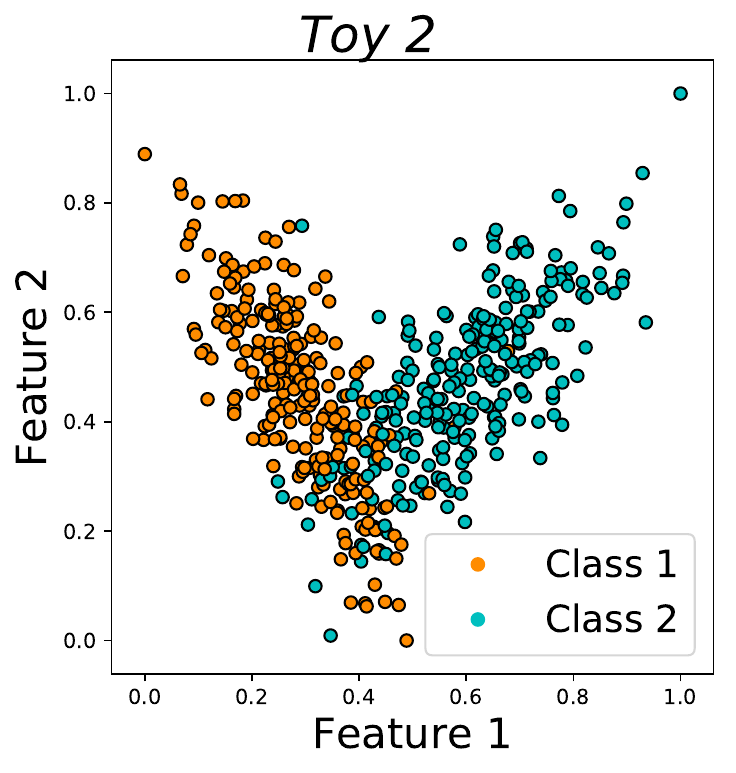}\quad
			\includegraphics[width=0.22\textwidth]{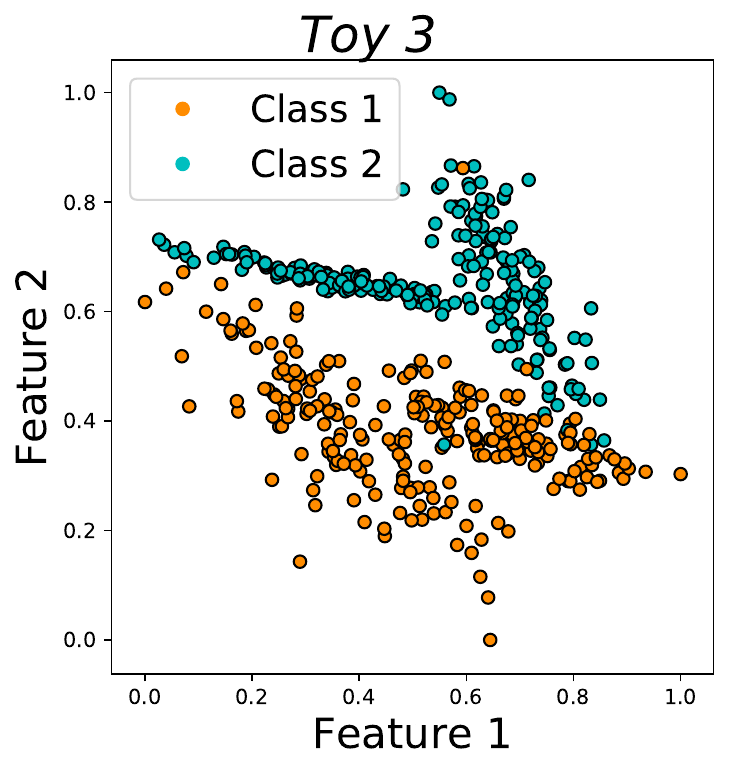}\quad
			\includegraphics[width=0.22\textwidth]{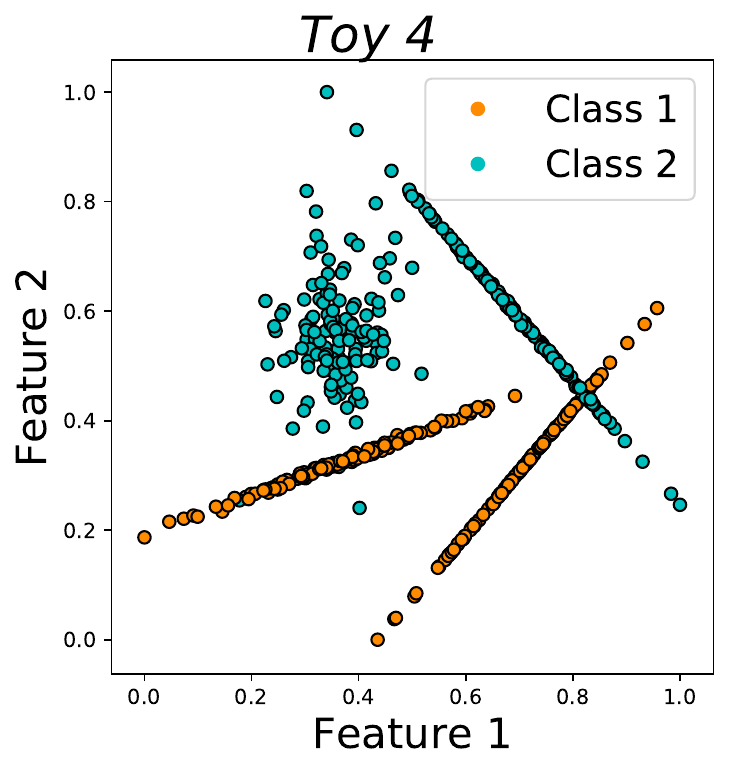}\\
			\includegraphics[width=0.22\textwidth]{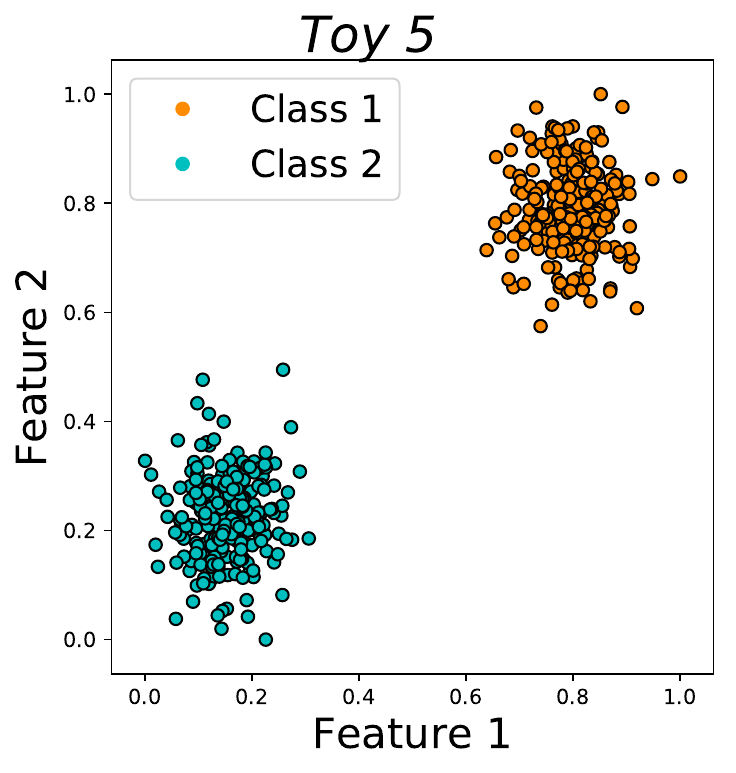}\quad
			\includegraphics[width=0.22\textwidth]{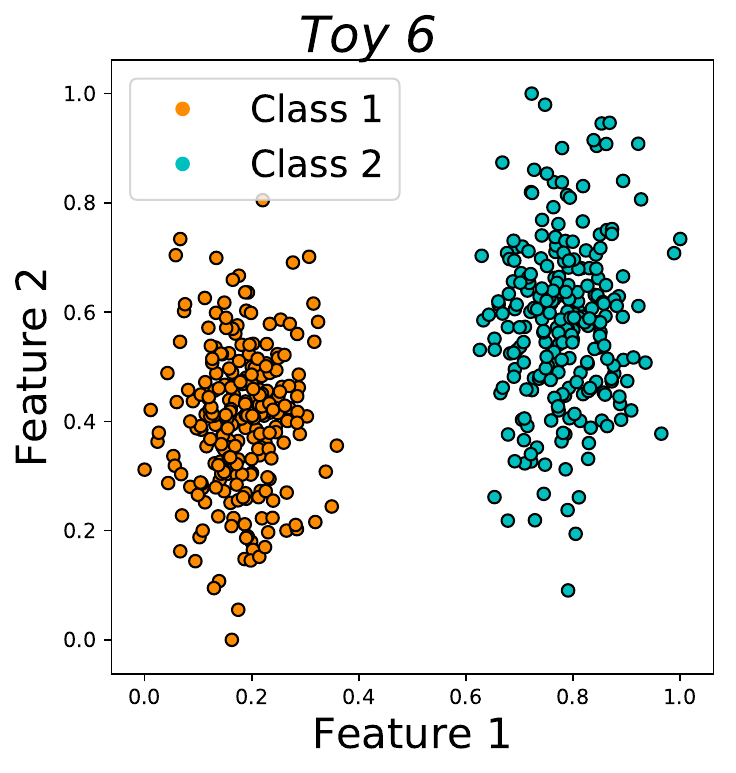}\quad
			\includegraphics[width=0.22\textwidth]{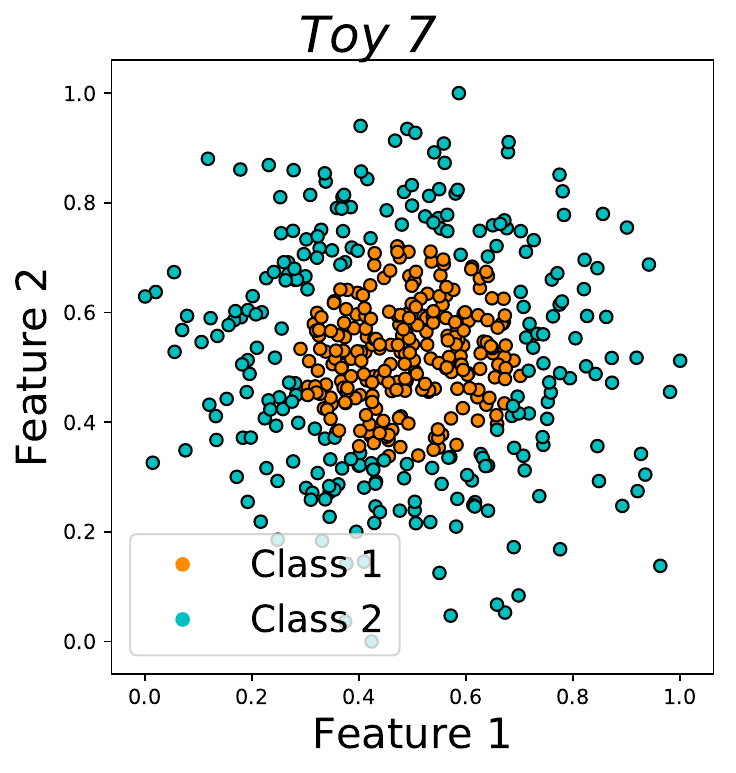}\quad
			\includegraphics[width=0.22\textwidth]{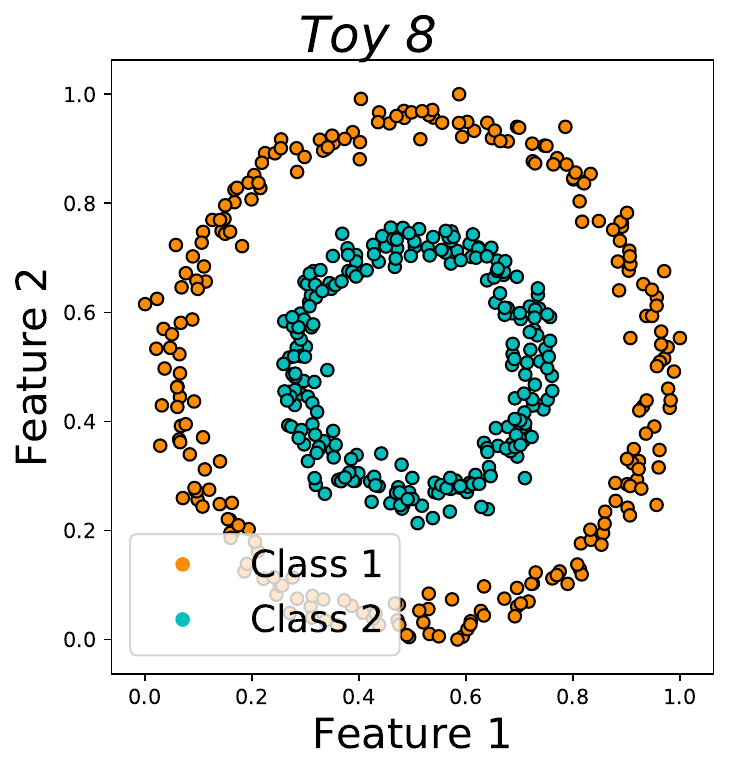}\\
			\includegraphics[width=0.22\textwidth]{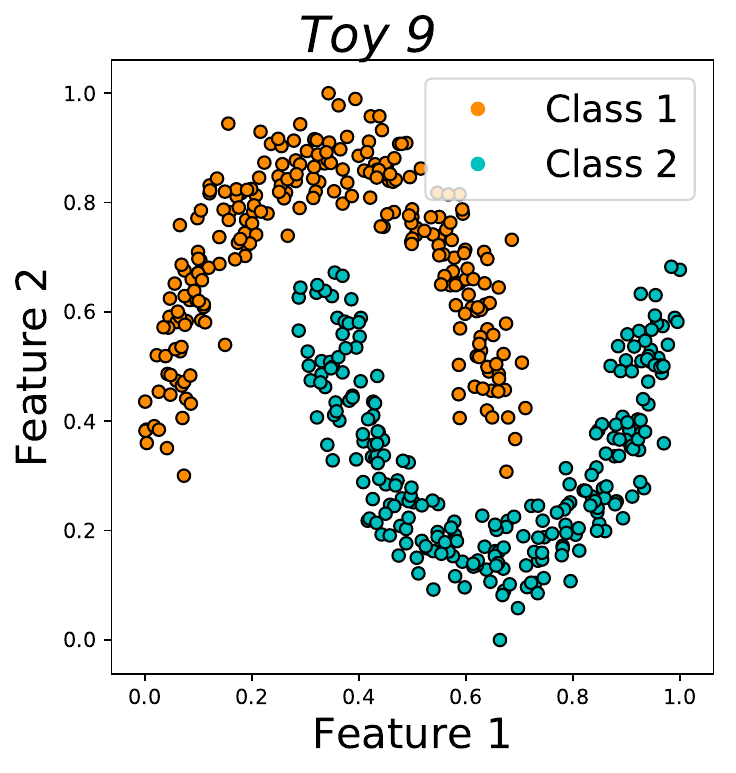}\quad	
			\includegraphics[width=0.22\textwidth]{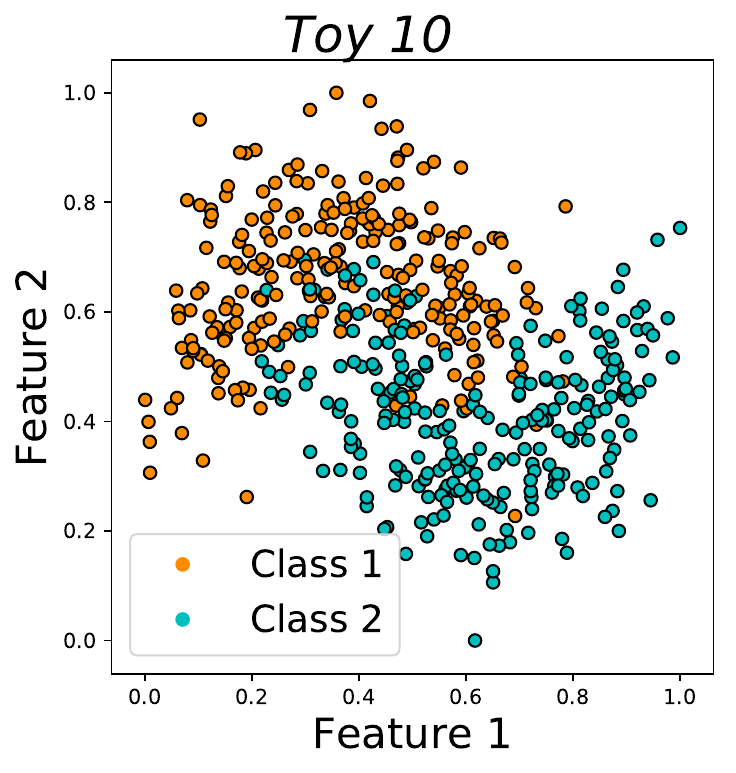}\quad
			\includegraphics[width=0.22\textwidth]{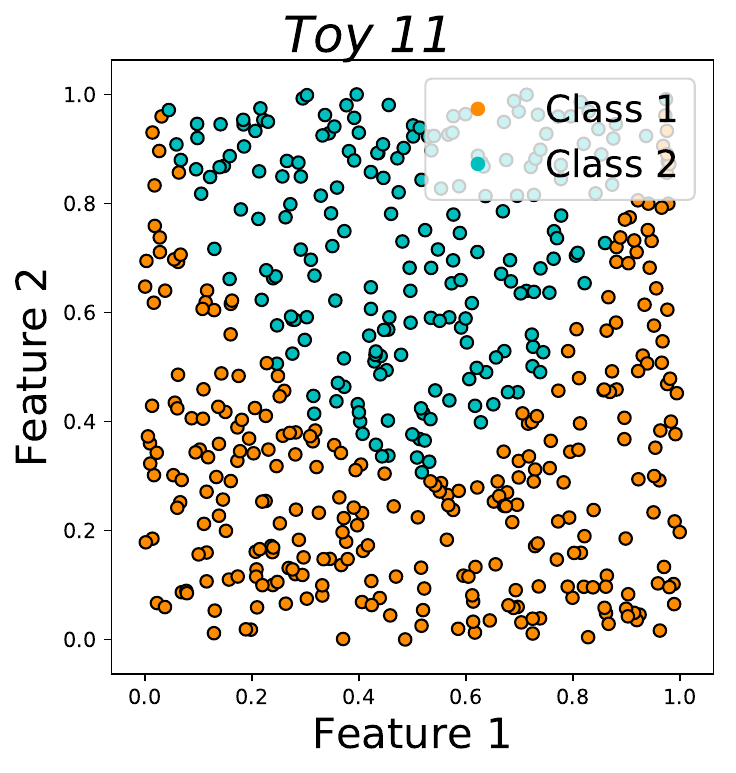}\quad
			\includegraphics[width=0.22\textwidth]{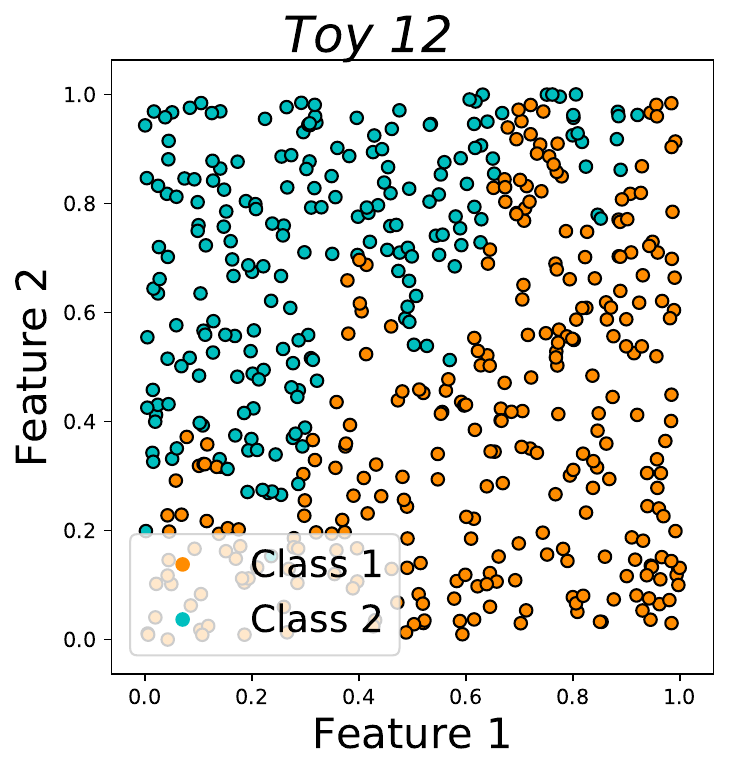}\\
			\\	
		\end{tabular}
		\caption{Synthetic two-dimensional data sets that we used in the experiments.}
		\label{fig:app_syn_datasets}
	\end{figure*}
	

				

\section{Performance of Different Machine Learning Algorithms and Rashomon Ratio}\label{appendix:perfomance_figures}

 Figure \ref{fig:bar_plots_1} and  Figure \ref{fig:bar_plots_2} show a performance comparison of different machine learning algorithms with regularization for the categorical and real-valued data sets. Data sets shown in  Figures \ref{fig:bar_plots_1} and \ref{fig:bar_plots_2}  are shown in decreasing order of the Rashomon ratio, from the highest in  Figure \ref{fig:bar_plots_1} to the Rashomon ratios that were so small that we were not able to measure them.

Regularization limits the hypothesis space and thus changes the nature of the Rashomon set's measurements. Each value of the regularization parameter corresponds to a soft constraint on the hypothesis space, which in turn can be realized as a hard constraint on this space. 
The Rashomon ratio in the regularized case will typically be larger or equal to the Rashomon ratio in the unregularized case. There are two reasons for this, explained below. 

First, regularization reduces the hypothesis space. Hypotheses that were available when learning without regularization may be excluded when learning with regularization. 
As a result, the size of the hypothesis space decreases, which increases the Rashomon ratio. 

Second, the empirical risk minimizer changes between the regularized and unregularized hypothesis sets, which means the criterion for falling into the Rashomon set changes as well. Recall that the Rashomon set is defined based on the best performing model on the training set.
The regularized hypothesis space is less likely to contain overfitted models than the unregularized space. This means the regularized hypothesis space's empirical risk minimizer typically has higher empirical risk than that of the unregularized hypothesis space. Then, if the Rashomon parameter $\theta$ is fixed when comparing the two hypothesis spaces, there may be more models in the Rashomon set for the regularized case. Thus, in the regularized case, the size of the Rashomon set would be larger, and, therefore, the Rashomon ratio would be larger too.


Figure \ref{fig:bar_plots_nr_1} and  Figure \ref{fig:bar_plots_nr_2} show a comparison of the performance of different machine learning algorithms without regularization for the categorical and real-valued data sets. 
 Figure~\ref{fig:bar_plots_3} and  Figure \ref{fig:bar_plots_nr_3} shows a performance comparison of different machine learning algorithms for the synthetic data sets with and without regularization.

As we mentioned before, we estimate the Rashomon ratio with importance sampling. For the proposal distribution, we generate a tree of depth $D$ by randomly splitting on features. We assign labels to all $2^{D}$ leaves using the training data. If a leaf contains no training points, it acquires its label from the nearest ancestor that any training data pass through. 
The probability of sampling any tree from the proposal distribution is $p_p = p_f \times \prod_{i=1}^{2^D} 1$, where $p_f$ is the probability of randomly sampling all of the features that comprise the splits of the tree. Our target distribution is a randomly sampled decision tree (both features and leaves) of depth $D$. Therefore, the probability of sampling a given tree from the target distribution is $p_t = p_f \times \prod_{i=1}^{2^D} \frac{1}{2}$, since we have two classification classes, where, as before, $p_f$ is the probability of randomly sampling all features used within splits of the tree. Thus, for one tree of depth seven, its importance weight will be $\frac{p_t}{p_p} = \left(\frac{1}{2}\right)^{2^{7}} \approx 3 \times 10^{-39}$. The importance weight clearly dictates the order of magnitude of the Rashomon ratio in our experiments. The smallest possible non-zero Rashomon ratio ($\approx 1.175 \times 10^{-42}\%$) arises when we sample one model that is in the Rashomon set among 250,000 total models that were sampled. Therefore, we consider the Rashomon ratios of order $10^{-37}\%$ and $10^{-38}\%$ to be large, and Rashomon ratios of order $10^{-40}\%$, $10^{-41}\%$, etc$.$, to be small. Note that if there are more than two classes in the data set, the importance weight will be even smaller, as the probability of sampling a random tree from a target distribution decreases. Thus, trees built on a data set with three classes will have lower probability than trees built on a data set with two classes. That is why we considered binary classification only and modified data as described in Table \ref{table:datasets}, as it is essential for us to compare ratios over different data sets.

If we choose another importance sampling method (for example with data assignment for only half of the leaves or with the guidance of both features and leaves) the Rashomon ratio may have different importance weights and therefore might have a different estimated size as well. This issue would be resolved if we sample a huge number of trees, which is hard to do in practice. Therefore, since our goal is to compare the Rashomon ratios across data sets and feature spaces, we use a consistent method of leaf-based importance sampling across all data sets and sample a manageable number of trees (250,000 in our case).

\section{Large Rashomon Ratios May Appear Artificially Small} \label{sec:sm_rsetmeasure}
Even when the Rashomon ratio is a good driver of generalization performance, it may appear artificially small because of a poor representation of data or poor choice of hypothesis space. For instance, if the features are highly correlated, this artificially deflates the size of the Rashomon ratio as discussed in Appendix \ref{appendix:ratios_and_features}. 
	Moreover, if the hypothesis space is poorly designed to include an overly large number of models, then the Rashomon ratio may appear artificially small.
	The issues with measuring the Rashomon ratio may be a possible explanation for some of the results in Figure \ref{fig:exp_bar}(b), which includes some data sets with high-performing algorithms, yet (by the way we measured it) a small Rashomon ratio. In any case, small Rashomon ratios are not our main interest; here we are interested in what we would observe under large Rashomon ratios. 

\section{Quality of the Features}\label{appendix:ratios_and_features}

		\begin{figure*}[t]
		\centering
		\begin{tabular}{ccc}
		\begin{subfigure}[b]{0.3\textwidth}
		\includegraphics[width=\textwidth]{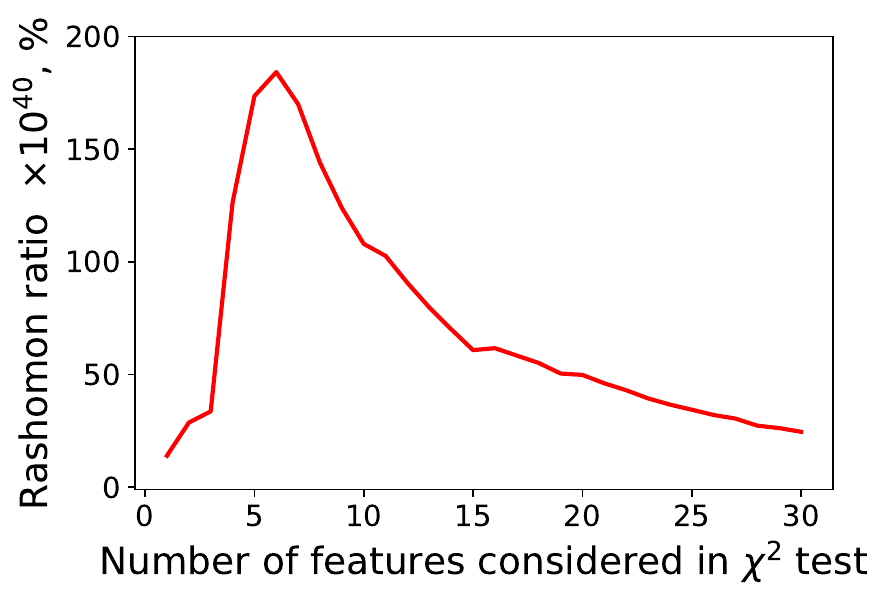}\quad	
		\caption{Reducing features}
		\end{subfigure}
		\begin{subfigure}[b]{0.3\textwidth}
		\includegraphics[width=\textwidth]{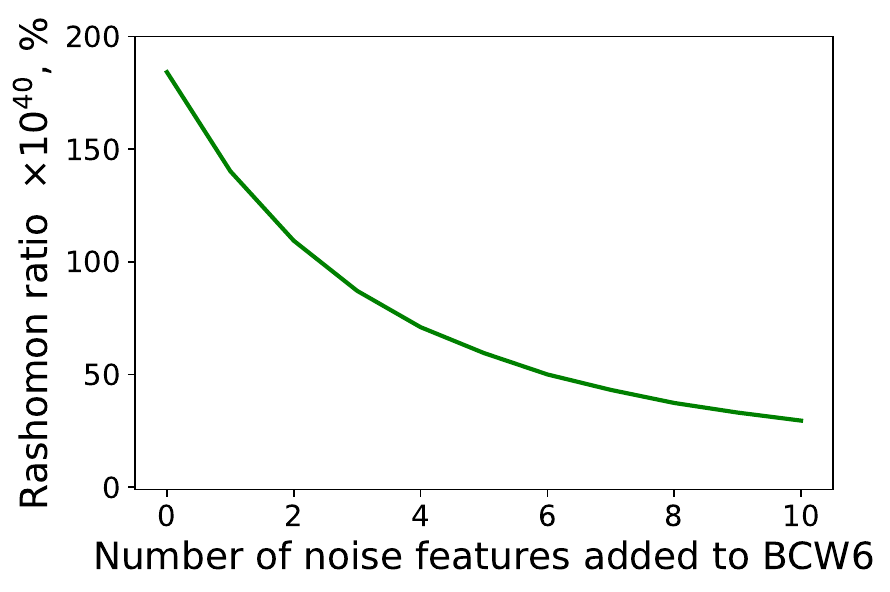}\quad	
		\caption{Adding noise}
		\end{subfigure}
		\begin{subfigure}[b]{0.3\textwidth}
		\includegraphics[width=\textwidth]{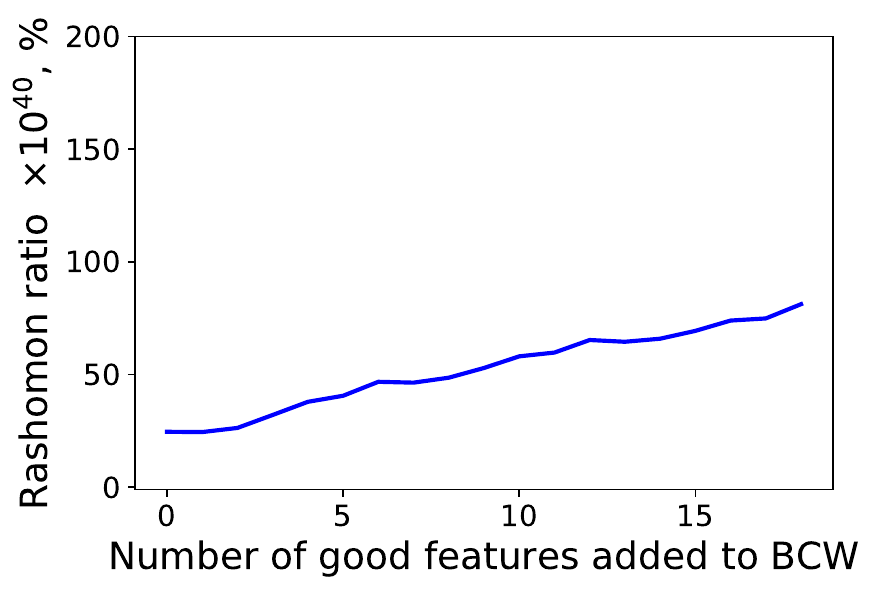}\quad	
		\caption{Adding good features}
		\end{subfigure}
		\end{tabular}
		\caption{An illustration of the influence of feature quality on the Rashomon ratio for the Breast Cancer Wisconsin data set (BCW). (a) shows the Rashomon ratio for the data set with different numbers of significant features according to a $\chi^2$ test. Denote the BCW with the six most significant features as BCW6. (b) depicts the correspondence between the Rashomon ratio and different numbers of noisy features added to the BCW6 data set. The  noise features are sampled from normal distribution $\mathcal{N}(0,1)$ and then standardized to be in a hypercube of volume one. (c) shows the change in the Rashomon ratio as we add more redundant features to the BCW data set. We iteratively add one out of six features from the BCW6 data set at a time. Rashomon ratios in (a)--(c) are averaged over ten folds. The Rashomon parameter $\theta$ is set to 0.05. Rashomon ratios are computed with respect to the best sampled model across all variations of the data set.\label{fig:experiments_features_bcw}}
	\end{figure*}

In our experiments, we observed a connection between the quality of the features and Rashomon ratios. The Rashomon ratio
in its simplest form, under uniform prior on the hypothesis space, is the 
fraction of models that are inside the Rashomon set compared to the models in the hypothesis space. When a data set is augmented with additional features, the size of the hypothesis space grows. If the added features are completely irrelevant (consisting, for instance, of noise) then adding these features increases the size of the hypothesis space but does not increase the size of the Rashomon set. Thus, we might predict that the Rashomon ratio could decrease as irrelevant features are added to a data set. 

Additionally, if we augment a data set with features that are highly correlated or identical to features that improve performance, then not only is the size of the hypothesis space increased, but also the size of the Rashomon set is likely to increase, as there exist more relevant models (even if the set becomes redundant with models that predict equivalently). Thus, we might predict that the Rashomon ratio increases as we add copies of relevant features.

In general, these two examples of irrelevant and redundant features are corner cases, however, they do occur to a lesser degree in real world data sets, and we are interested in whether these cases have potentially influenced our experimental results in Section \ref{section:experiments} in our observed Rashomon ratios. To investigate this, we augmented a data set with noise features, and separately, augmented the same data set with copies of useful features to see whether irrelevant or correlated features may have influenced our findings on the measurement of the Rashomon ratio. We used the Breast Cancer Wisconsin (Diagnostic) data set (shortly, BCW), which has approximately six important features. The results are shown in Figure \ref{fig:experiments_features_bcw}. As before, our hypothesis space is decision trees of depth seven.

\paragraph{Irrelevant features} If the data set contains a lot of irrelevant or noisy features, we expect the Rashomon set to be relatively small compared to the hypothesis space. Figure  \ref{fig:experiments_features_bcw}(a) shows how the Rashomon ratio changes as we iteratively decrease the number of features in the Wine data set, eliminating the least relevant features first, leaving the most significant ones (where relevance is determined according to a $\chi^2$ test with the label). The Rashomon ratio grows as we first remove non-significant features, and after reaching a peak at around six features, it starts to decrease as we remove relevant features, and as models lose accuracy. Similarly, Figure \ref{fig:experiments_features_bcw}(b) shows the influence of noisy features on the Rashomon ratio. Particularly, as we add more noisy irrelevant features, the Rashomon ratio starts to decrease. This is due to the same fact, that we artificially enlarge the hypothesis space while keeping the Rashomon set approximately the same. The noise features do not help improve the empirical risk, they only increase the size of the reasonable set.

\paragraph{Redundant features} As a contrast to how we increased the hypothesis space in the previous experiment, we can increase the Rashomon set by adding more redundant, good features. Figure \ref{fig:experiments_features_bcw}(c) shows how the Rashomon ratio changes for the BCW data set as we add more copies of the six most significant features. We observe that the Rashomon ratio increases. By adding copies of relevant features, we increased the number of trees at a given depth that could be good enough to be in the Rashomon set.

Our findings show a possible connection between the Rashomon ratio and feature analysis. In particular, in the case where different algorithms perform similarly, but the Rashomon ratio is observed to be small, it could be due to the reason that the data set contains noisy or irrelevant features. In that case, it may be possible to iteratively remove features to find those that produce the largest Rashomon ratio without changes to the empirical risk. The other extreme is less likely to be observed in practice, which is when the Rashomon ratio is extremely large due to redundant features. In that case, one could remove redundant (highly correlated) features before measuring the Rashomon ratio. The data sets with smaller numbers of features induce easier learning/optimization problems in general. As we discussed earlier, the Rashomon ratio would generally not be measured in practice, and would be inferred in other ways. Thus, these results mainly pertain to an understanding of the experiments we did in Section \ref{section:experiments} to provide a possible explanation for cases of small observed Rashomon ratios but where all methods perform the same and all functions generalize.

\section{Proof of Theorem \ref{th:noise_labels}}

\begingroup
\def\thetheorem{\ref{th:noise_labels}}
\begin{theorem}[Expected size of the true Rashomon set cannot decrease under random classification noise]
    \ThNoiseLabels
\end{theorem}
\addtocounter{theorem}{-1}
\endgroup

    \begin{proof}
   We are given that the probability with which each label is flipped is $\rho$, $P(\tilde{y}_i\neq y_i) = \rho$, where $\tilde{y}_i$ is a flipped label of $y_i$. Recall that the true risk $L_{\mathcal{D}}(f) = \E_{(x,y)\sim \mathcal{D}}[\phi(f(x), y)] = \E_{(x,y)\sim \mathcal{D}}[\mathbb{1}_{[f(x)\neq y]}]$, 
   and $f^*$ is an optimal function, meaning that $f^* \in \textrm{\rm argmin}_{f\in \F} L_{\mathcal{D}}(f)$. 
   Given the noisy distribution $\mathcal{D}_{\rho}$, denote $f_{\rho}^*\in \textrm{\rm argmin}_{f\in \F} L_{\mathcal{D}_{\rho}}(f)$.
    For any model from the true Rashomon set $f \in R_{{set}_{\mathcal{D}}}(\F,\gamma) = \{f: \mbox{ }L_{\mathcal{D}}(f) \leq L_{\mathcal{D}}(f^*) + \gamma\}$, we have:
\begin{equation*}
    \begin{split}
        L_{\mathcal{D}_{\rho}}(f) &= \E_{(x,\tilde{y})\sim \mathcal{D}}\left[\mathbb{1}_{[f(x)\neq \tilde{y}]}\right] \\
        &= \E_{(x,\tilde{y})\sim \mathcal{D}}\left[\mathbb{1}_{[f(x)\neq \tilde{y}]} | \tilde{y} = y\right]P(\tilde{y}=y) \\
        &+ \E_{(x,\tilde{y})\sim \mathcal{D}}\left[\mathbb{1}_{[f(x)\neq \tilde{y}]} | \tilde{y} \neq y\right]P(\tilde{y}\neq y)\\
        & = \E_{(x,y)\sim \mathcal{D}}\left[\mathbb{1}_{[f(x)\neq y]}\right]P(\tilde{y}= y) \\
        &+ \E_{(x,y)\sim \mathcal{D}}\left[\mathbb{1}_{[f(x)= y]}\right]P(\tilde{y}\neq y)\\
        & = \E_{(x,y)\sim \mathcal{D}}\left[\mathbb{1}_{[f(x)\neq y]}\right] \left( 1- P(\tilde{y}\neq y)\right) \\
        &+ \left( 1- \E_{(x,y)\sim \mathcal{D}}\left[\mathbb{1}_{[f(x)\neq y]}\right]\right)P(\tilde{y}\neq y)\\
        & = \E_{(x,y)\sim \mathcal{D}}\left[\mathbb{1}_{[f(x)\neq y]}\right] \left( 1- 2 P(\tilde{y}\neq y)\right) + P(\tilde{y}\neq y)\\
        &= \E_{(x,y)\sim \mathcal{D}}\left[\mathbb{1}_{[f(x)\neq y]}\right] \left( 1- 2 \rho\right) + \rho \\ 
        & =L_{\mathcal{D}}(f)(1 - 2\rho) + \rho.
    \end{split}
\end{equation*}

Since $L_{\mathcal{D}}(f^*) \leq L_{\mathcal{D}}(f_{\rho}^*)$ as $f^*$ is an optimal model over $\mathcal{D}$, $0 < \rho < 0.5$ by assumption, and, again, $f$ is in the true Rashomon set $R_{{set}_{\mathcal{D}}}(\F,\gamma)$, we have:
\begin{equation}
    \begin{split}
        L_{\mathcal{D}_{\rho}}(f)  - L_{\mathcal{D}_{\rho}}(f_{\rho}^*) & = L_{\mathcal{D}}(f)(1 - 2\rho) + \rho - L_{\mathcal{D}}(f_{\rho}^*)(1 - 2\rho) - \rho \\
        & =(L_{\mathcal{D}}(f)  - L_{\mathcal{D}}(f_{\rho}^*) ) (1 - 2\rho)  \\
        &\leq (L_{\mathcal{D}}(f)  - L_{\mathcal{D}}(f^*) ) (1 - 2\rho) \\
        & \leq \gamma (1 - 2\rho) \leq \gamma.
    \end{split}
\end{equation}
Therefore, $f$ is in the true Rashomon set $R_{{set}_{\mathcal{D}_{\rho}}}(\F,\gamma)$. As this calculation holds for every model $f$ from the true Rashomon set $R_{{set}_{\mathcal{D}}}(\F,\gamma)$, then in expectation  $R_{{set}_{\mathcal{D}}}(\F,\gamma)\subseteq R_{{set}_{\mathcal{D}_{\rho}}}(\F,\gamma)$.


This proof uses similar technique as Theorem 1 in the work of \citep{manwani2013noise}, where the purpose is to show noise tolerance of the optimal model.
\end{proof}

\section{Evidence for Conjecture \ref{th:noise_gaussian}}\label{appendix:gaus_proof}

While there is not yet a proof for the conjecture, there is substantial evidence, which we present in this section. Please see Figure \ref{fig:gaussetup} for illustration and details that will help with the intuition for these conjectures.

\begin{figure*}[t]
	\centering
		\includegraphics[width=0.5\textwidth]{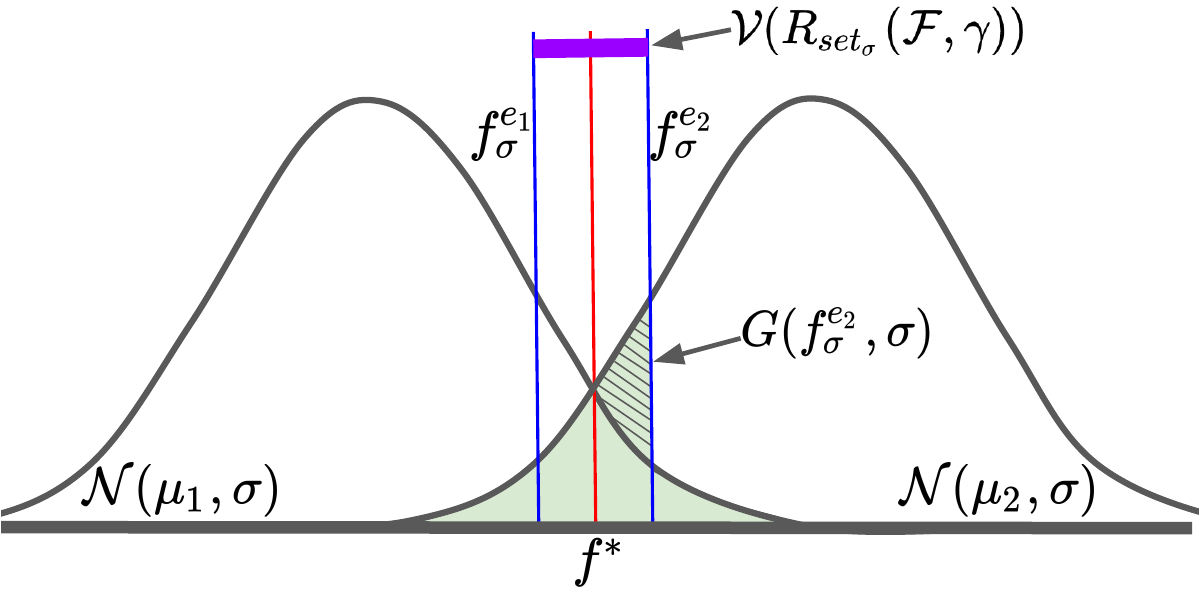}
		\caption{The setup for Conjecture  \ref{th:noise_gaussian}. We show two Gaussians $\mathcal{N}(\mu_1, \sigma)$ and $\mathcal{N}(\mu2, \sigma)$, with optimal model $f^* = \frac{\mu_1 + \mu_2}{2}$ (shown in \textcolor{red}{red}). Models $f_{\sigma}^{e_1}$ and $f_{\sigma}^{e_2}$ (shown in \textcolor{blue}{blue}) correspond to the left and right edges of the Rashomon set (shown in \textcolor{purple}{purple}). The loss of model $f_{\sigma}^{e_2}$ is computed as the sum of two Gaussian tails and is equal to the area of the \textcolor{MyLightGreen}{green} region. The objective $G(f_{\sigma}^{e_2},\sigma)=L(f_{\sigma}^{e_2})-L(f^*)$ corresponds to the area of the shaded region to the left of function $f_{\sigma}^{e_2}$.}
		\label{fig:gaussetup}
	\end{figure*}

\begingroup
\def\theconjecture{\ref{th:noise_gaussian}}
\begin{conjecture}[The Rashomon set can increase with feature noise]

Consider data distribution $\mathcal{D} = \mathcal{X}\times \mathcal{Y}$, where, $\mathcal{X}\in\R$, $\mathcal{Y}\in\{-1,1\}$, and classes are balanced $P(Y=-1)=P(Y=1)$ and generated by Gaussian distributions $P(X|Y=-1)=\mathcal{N}(\mu_1,\sigma^2)$,  $P(X|Y=1)=\mathcal{N}(\mu_2,\sigma^2)$, where $0 \leq \mu_1 <\mu_2$. For the hypothesis space $\F = \{f: f\in (\beta_1, \beta_2) \}$, where $ (\mu_1, \mu_2) \subset (\beta_1, \beta_2)$,  $\beta_1 \ll \mu_1$, and $\mu_2 \ll \beta_2$, 
    and the Rashomon parameter $\gamma > 0$:
    \begin{enumerate}
        \item[(I)] The volume of the Rashomon set is $\V(R_{{set}_{\sigma}}(\F, \gamma)) = |f^{e_1}_\sigma- f^{e_2}_\sigma|$, where $f^{e_1}_\sigma$ and $f^{e_2}_\sigma$ are the two solutions to Eqn$.$ \eqref{eq:gaus_obg}, where $\Phi$ is the CDF of the standard normal:
    \begin{equation}
        2\Phi \left(\frac{\mu_2-\mu_1}{2\sigma}\right) - \Phi \left(\frac{\mu_2-f}{\sigma}\right) - \Phi \left(\frac{f-\mu_1}{\sigma}\right) = \gamma.\tag{\ref{eq:gaus_obg}}
    \end{equation}
    \item[(II)] We conjecture that for $\F = \{f: f\in (\mu_1, \mu_2) \}$, as we add feature noise to the data set by increasing the standard deviation $\sigma$, for all $\sigma$ such that $\sigma > \tilde{\sigma} = \frac{\mu_2-\mu_1}{2\sqrt{2}}$, the volume of the Rashomon set 
    increases as a function of $\sigma$.
    \item[(III)] Consider the setting where $\sigma = 1$ for both Gaussians, and we add or remove noise by moving the means $\mu_1$ and $\mu_2$ of the Gaussians towards or away from each other. For any $\gamma>0$, the volume of the Rashomon set is minimized when $\mu_2\approx \mu_1+2$. Moving the Gaussians either away from or towards each other increases the volume of the Rashomon set.
    \end{enumerate}
\end{conjecture}
\addtocounter{theorem}{-1}
\endgroup

    \begin{proof}
     Without loss of generality, we will assume that $\mu_1 = 0$ and $\mu_2 = \mu>0$. To get results for the original values $\mu_1$ and $\mu_2$, we can simply add $\mu_1$ to $\mu$ and $f$. 
     
     Let us show the first point of the conjecture and show how to compute the volume of the Rashomon set.
     
     \paragraph{Evidence for Part (I)}  
     Denote $\phi_1$ and $\phi_2$ as probability density functions (PDF) and $\Phi_1$ and $\Phi_2$ as cumulative distribution functions (CDF) of classes $Y=-1$ and $Y=1$ correspondingly. For a given model $f \in \mathcal{F}$, the loss can be computed as the sum of areas under the PDF corresponding to misclassification errors:
    \begin{equation}\label{eq:noise_loss}
        \begin{split}
            L(f)& = P(X > f | Y = -1) + P(X \leq f | Y = 1)\\
            &= \int_{f}^{\infty}\phi_1(t) d t + \int_{-\infty}^{f}\phi_2(t) d t = 1 - \Phi_1(f) + \Phi_2(f)\\
            &= 1  - \Phi\left(\frac{f-0}{\sigma}\right) + \Phi\left(\frac{f-\mu}{\sigma}\right)\\ 
            &=2  - \Phi\left(\frac{f}{\sigma}\right) - \Phi\left(\frac{\mu-f}{\sigma}\right),
        \end{split}
    \end{equation}
where $\Phi$ is the CDF of the normal distribution $\mathcal{N}(0,1)$.

The optimal model, $f^*$, can be obtained when $P(X|Y=-1) = P(X|Y=1)$, meaning that $\frac{1}{\sigma \sqrt{2\pi}} \exp\left(\frac{-(x-0)^2}{2\sigma^2}\right)=\frac{1}{\sigma \sqrt{2\pi}} \exp\left(\frac{-(x-\mu)^2}{2\sigma^2}\right)$, which leads to $f^*= \frac{\mu}{2}$. The loss of the optimal model is:
\[L(f^*) = 2 - 2\Phi\left(\frac{\mu}{2\sigma}\right).\]

Denote $G(f, \sigma)$ as the difference in loss between a model $f$ from the Rashomon set and optimal model $f^*$, $$G(f, \sigma) := L(f)-L(f^*) = 2\Phi \left(\frac{\mu}{2\sigma}\right) - \Phi \left(\frac{\mu-f}{\sigma}\right) - \Phi \left(\frac{f}{\sigma}\right).$$ 

To find a model on the edge of the true Rashomon set, we set $G(f, \sigma) = \gamma$ and obtain Equation \eqref{eq:gaus_obg}.
As the problem is symmetric with respect to $f^*$ and $\gamma > 0$, solutions $f^{e_1}_\sigma$ and $f^{e_2}_\sigma$ to Equation \eqref{eq:gaus_obg} correspond to the left and right edges of the true Rashomon set. Thus, for a given $\gamma$ and $\sigma$, we can estimate the volume of the Rashomon set as  $\mathcal{V}(R_{{set}_{\sigma}}(\F, \gamma)) = |f^{e_1}_\sigma- f^{e_2}_\sigma|$. 

Now we will show evidence for the second point of the conjecture that the volume of the Rashomon set increases when $\sigma$ increases.

\paragraph{Evidence for Part (II)}
Let $\tilde{\sigma} = \frac{\mu}{2\sqrt{2}}$. Let's focus on the right edge of the true Rashomon set where $f\in (f^*, \mu)$. We will show that:
\begin{enumerate}
    \item[(i)] For any fixed $\sigma$, and for $f \in (f^*,\mu)$, $G(f, \sigma)$ is monotonically increasing in $f$.
    \item[(ii)] For any fixed $f \in (f^*,\mu)$ and $\sigma > \tilde{\sigma}$, $G(f, \sigma)$ is monotonically decreasing in $\sigma$.
    \item[(iii)] If (i) and (ii) hold, then for any $\sigma_1$ and $\sigma_2$ such that $\tilde{\sigma} < \sigma_1 < \sigma_2$, we have that $\mathcal{V}(R_{{set}_{\sigma_1}}(\F, \gamma)) < \mathcal{V}(R_{{set}_{\sigma_2}}(\F, \gamma))$.
\end{enumerate}
 We will prove (i) and (iii), and conjecture that (ii) is true, based on numerical experiments.
 
 Let's focus on (iii) first. Consider $\sigma_1$ and $\sigma_2$ such that $\tilde{\sigma} <\sigma_1 < \sigma_2$. Let $f^{e}_{\sigma_1}, f^{e}_{\sigma_2} \in (f^*,\mu)$ be the right edge models of the corresponding true Rashomon sets $G(f^{e}_{\sigma_1}, \sigma_1) = \gamma$ and $G(f^{e}_{\sigma_2}, \sigma_2) = \gamma$. 
 The volume of the Rashomon set is $\mathcal{V}(R_{{set}_{\sigma}}(\F, \gamma)) =|f^{e_1}_\sigma- f^{e_2}_\sigma| = 2|f^{e_2}_\sigma - f^*| = 2|f^{e}_\sigma- f^*|$.
 Using monotonicity of $G(f, \sigma)$ with respect to $\sigma$ given fixed $f$ (the result of part (ii)), we get that:
\[G(f^{e}_{\sigma_1}, \sigma_2) < G(f^{e}_{\sigma_1}, \sigma_1) = \gamma = G(f^{e}_{\sigma_2}, \sigma_2),\]
which means that $f^{e}_{\sigma_1} < f^{e}_{\sigma_2}$ due to monotonicity of $G(f, \sigma)$ with respect to $f$ given fixed $\sigma$ (the result of part (i)). Therefore, as we increase the standard deviation from $\sigma_1$ to $\sigma_2$, the Rashomon set increases as well:
\[\mathcal{V}(R_{{set}_{\sigma_1}}(\F, \gamma)) = 2|f^{e}_{\sigma_1}- f^*| <  2|f^{e}_{\sigma_2}- f^*| = \mathcal{V}(R_{{set}_{\sigma_2}}(\F, \gamma)).\]
Thus, we have proved part (iii).

Now, let us prove (i). Given fixed $\sigma$, the derivative of the objective $G(f, \sigma)$ with respect to $f$ is:
\begin{equation*}
\begin{split}
    G_f'(f, \sigma) &= \frac{1}{\sigma}\Phi_f'\left(\frac{\mu-f}{\sigma}\right) - \frac{1}{\sigma}\Phi_f'\left(\frac{f}{\sigma}\right) \\
    &= \frac{1}{ \sqrt{2\pi}\sigma}\left( e^{-\frac{(\mu-f)^2}{2\sigma^2}} - e^{-\frac{(f)^2}{2\sigma^2}}\right) > 0,
\end{split}
\end{equation*}
since $\mu-f < f$ since $f > f^* = \frac{\mu}{2}$. Since the derivative $G_f'(f, \sigma)>0$, the objective $G(f, \sigma)$ monotonically increases in $f$.

Finally, let's show (ii). Given fixed $f \in (f^*,\mu)$ and $\sigma > \tilde{\sigma}$, the  derivative of the objective with respect to $\sigma$ is:
\begin{equation*}
\begin{split}
    G_{\sigma}'(f, \sigma) &= \frac{-2(\mu)}{2\sqrt{2\pi}\sigma^2}\Phi_{\sigma}' \left(\frac{\mu}{2\sigma}\right) +\frac{\mu-f}{\sqrt{2\pi}\sigma^2}\Phi_{\sigma}' \left(\frac{\mu-f}{\sigma}\right) +\frac{f}{\sqrt{2\pi}\sigma^2}\Phi_{\sigma}' \left(\frac{f}{\sigma}\right)\\
    &= \frac{1}{ \sqrt{2\pi}\sigma^2} \left[ - \mu e^{-\frac{\mu^2}{8\sigma^2}} + (\mu-f)e^{-\frac{(\mu-f)^2}{2\sigma^2}}+f e^{-\frac{f^2}{2\sigma^2}}\right]\\
    &= \frac{f}{ \sqrt{2\pi}\sigma^2} \left[ - \frac{\mu}{f} \left(e^{\frac{f^2}{2\sigma^2}}\right)^{\frac{-(\mu/f)^2}{4}} + \left(\frac{\mu}{f}-1\right)\left(e^{\frac{f^2}{2\sigma^2}}\right)^{-((\mu/f)-1)^2}\right.\\
    &\left.+ \left(e^{\frac{f^2}{2\sigma^2}}\right)^{-1}\right].
\end{split}
\end{equation*}

Denote $u = \frac{\mu}{f}$. Since, $\frac{\mu}{2} < f < \mu$, then $u\in(1,2)$. Denote $a = e^{\frac{f^2}{2\sigma^2}}$, then $\sigma^2 = \frac{f^2}{2 \log(a)}$. Note that $a > 1$, and since $\sigma > \tilde{\sigma} = \frac{\mu}{2\sqrt{2}}$, $a < e^{\frac{4f^2}{\mu^2}} = e^{\frac{4}{u^2}} = s(u)$. Then $G_{\sigma}'(f, \sigma)$ can be expressed in terms of parameter $u$ and variable $a$:

\begin{equation*}
\begin{split}
    G_{\sigma}'(u, a) &= \frac{2\log(a)}{ \sqrt{2\pi}f} \left[ - u a^{\frac{-u^2}{4}} + \left(u-1\right)a^{-(u-1)^2}+ a^{-1}\right] \\
    &= \frac{2\log(a)}{ \sqrt{2\pi}f} D(u,a).
\end{split}
\end{equation*}

As $a > 1$ and $f >\frac{\mu}{2} > 0$, $\frac{2\log(a)}{ \sqrt{2\pi}f}$ > 0. To show that $D(u,a) < 0$ for any $u\in(1,2)$ and any $a\in(1, s(u))$, we perform exhaustive numerical calculations spanning the possible values of $u\in(1,2)$ and $a\in(1, s(u))$ in Figure \ref{fig:derivative}. Indeed, $D(u,a) = 0$ when $u = 2$ or $a = 1$, and for all other values of $u$ and $a$, $D(u,a) < 0$. Therefore, $G_{\sigma}'(f, \sigma) = G_{\sigma}'(u, a) < 0$. Since the derivative is negative, when $\sigma > \tilde{\sigma}$ the objective $G(f, \sigma)$ monotonically decreases in $\sigma$, which concludes our evidence for (ii), and thus part (II), of the conjecture.

	\begin{figure*}[t]
	\centering
		\includegraphics[width=0.4\textwidth]{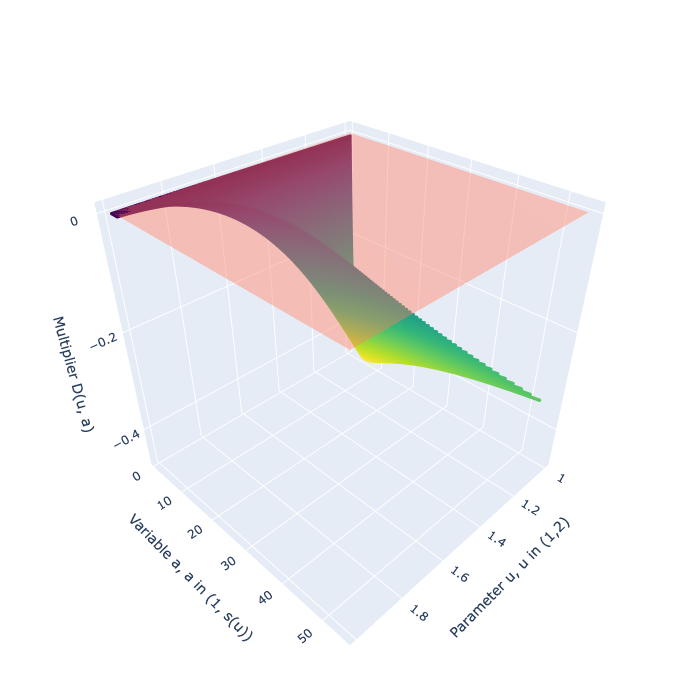}
		\includegraphics[width=0.4\textwidth]{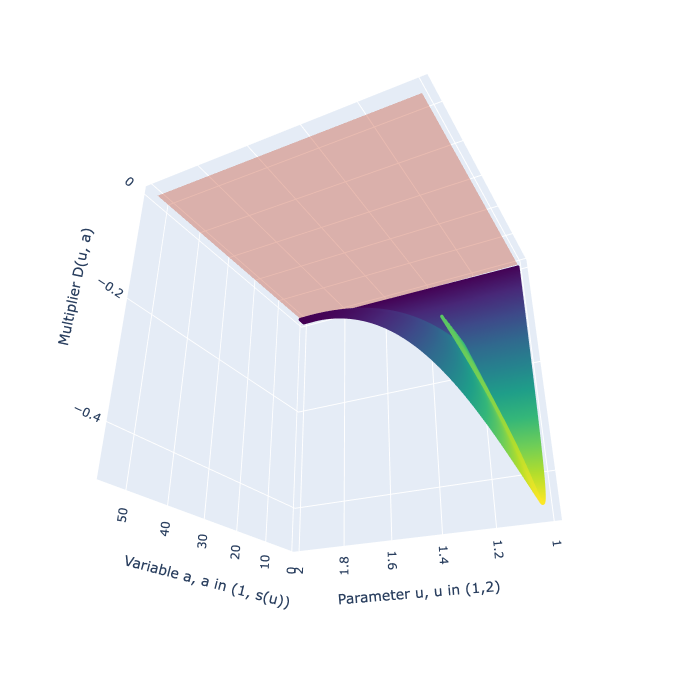}\\
		\includegraphics[width=0.4\textwidth]{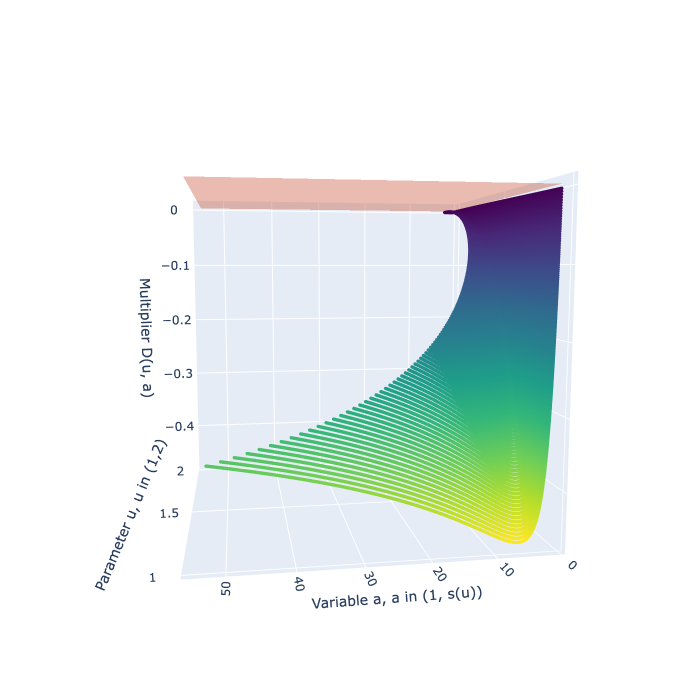}
		\includegraphics[width=0.4\textwidth]{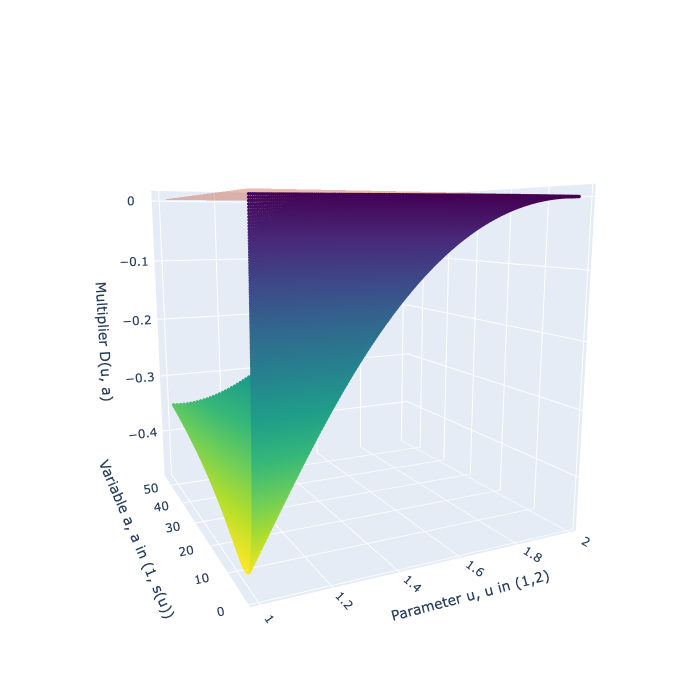}\\
		\includegraphics[width=0.4\textwidth]{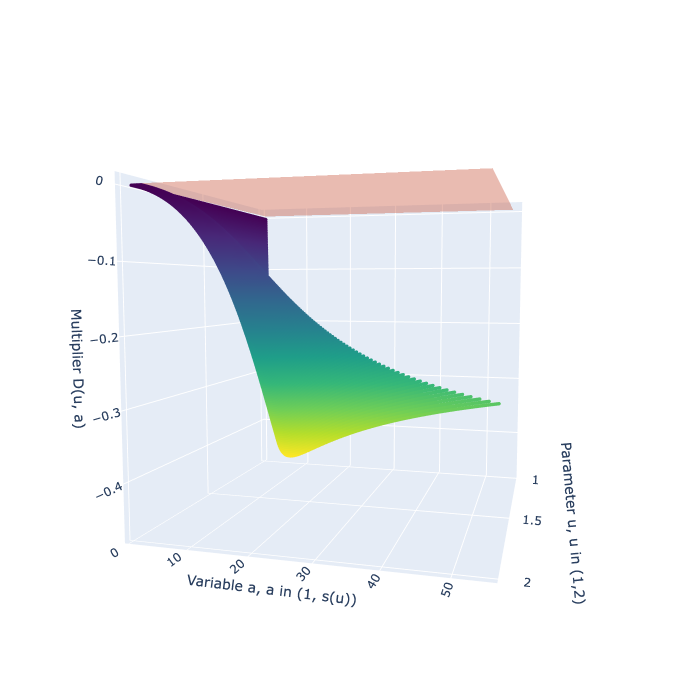}
		\caption{Numerically showing that $D(u,a) < 0$ for any $u\in(1,2)$, and $a\in(1,e^{\frac{4}{u^2}})$. Red plane corresponds to the value 0.}
		\label{fig:derivative}
	\end{figure*}
	
	\begin{figure*}[t]
	\centering
		\includegraphics[width=0.7\textwidth]{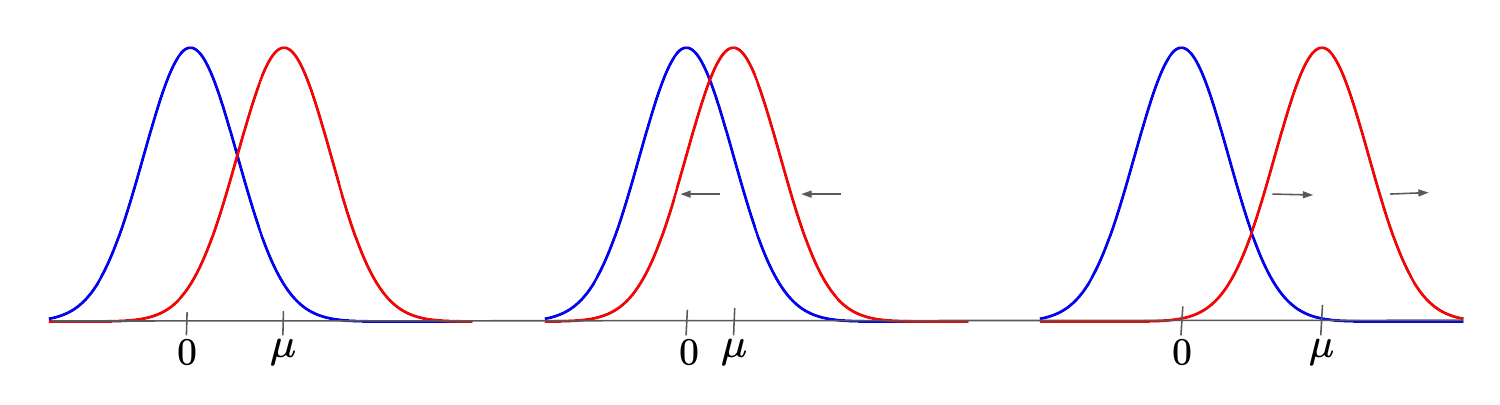}
		\caption{In part (III), we add or remove noise from the data by moving the mean of the right Gaussian. Both Gaussians have standard deviation 1. 
		}
		\label{fig:moving_gaussian}
	\end{figure*}
	
\paragraph{Evidence for Part (III)}
Now we add or remove noise from the data set by moving the two means closer together or further apart (see Figure \ref{fig:moving_gaussian}), for a fixed $\sigma = 1$.
Recall that without a loss of generality, we take  $\mu_1 = 0$ and denote $\mu = \mu_2$. The optimal model $f^*$ is $f^* = \frac{\mu}{2}$. Given $\sigma = 1$ and $\gamma > 0$, we are interested in finding $\mu > 0$ and edge model $f \in (0, \mu/2)$ such that the volume of the Rashomon set is minimal. Therefore we have the following optimization problem:
\begin{equation}\label{eq:minrsetoptimization}
\begin{gathered}
    \min_{\mu} \frac{\mu}{2} - f_{\mu} \text{ s.t. } f_{\mu} \textrm{ is defined by}\\
     2\Phi\left(\frac{\mu}{2}\right) - \Phi(\mu-f_{\mu}) - \Phi(f_{\mu}) = \gamma.
\end{gathered}
\end{equation}

We cannot solve optimization problem \eqref{eq:minrsetoptimization} directly for the best $\mu$ for each $\gamma$, but we provide numerical solutions in Figure \ref{fig:minrsetfigure} for a range of values of $\gamma$. We observe that, regardless of the value of $\gamma > 0$, as we minimize the volume of the Rashomon set to find $\mu$, the $\mu$ corresponding to the optimal solution is always approximately equal to $2$. The edge model $f^e\in(0,1)$ is then the one that satisfies the constraint: $2\Phi(1) - \Phi(2 - f^e) - \Phi(f^e) = \gamma$. (The edge model must vary with $\gamma$ since the optimal $\mu$ does not.)

The value $\mu = 2$ might seem surprising as a solution for {\em any} $\gamma > 0$. However, when $\mu_1 = 0$ and $\mu_2 = 2$, $f^* = 1$, which is one standard deviation away from each mean and therefore corresponds to a value where the inflection points of the two Gaussians coincide (normal distribution $\mathcal{N}(\mu,\sigma)$ has two inflection points $\mu \pm \sigma$). Given fixed $\gamma$, as we move either of the means of the Gaussians so that the inflection points no longer coincide, $f^e$ moves outward to compensate so that the constraint in optimization problem \eqref{eq:minrsetoptimization} is satisfied. Therefore, the volume of the Rashomon set grows as we increase or decrease feature noise by moving $\mu$ away from 1 (in either direction) as shown in Figure \ref{fig:gaus_rset_growth}.


\begin{figure*}[t]
	\centering
		\includegraphics[width=0.4\textwidth]{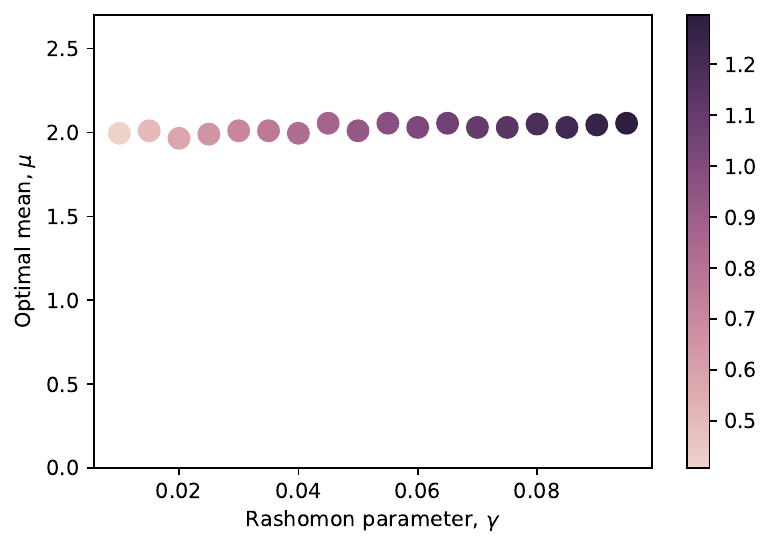}
		\includegraphics[width=0.4\textwidth]{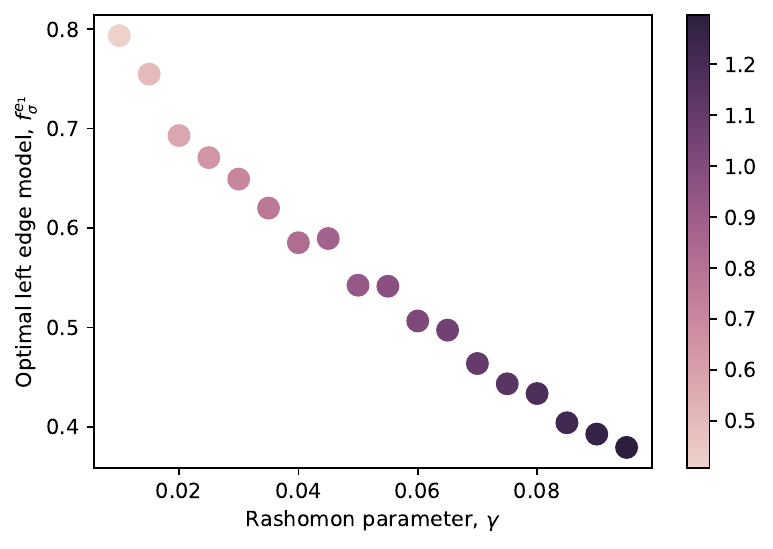}
		\caption{Numerical solution to the optimization problem \eqref{eq:minrsetoptimization}. For each fixed $\gamma$, we plot $f_{\sigma}^{e_1} = f^e \in(0, \mu/2)$, such that the volume of the Rashomon set is minimized. The color of the scatter plot points correspond to the value of the volume of the Rashomon set, where more intense color means higher value.}
		\label{fig:minrsetfigure}
	\end{figure*}

	\begin{figure*}[t]
	\centering
		\includegraphics[width=0.9\textwidth]{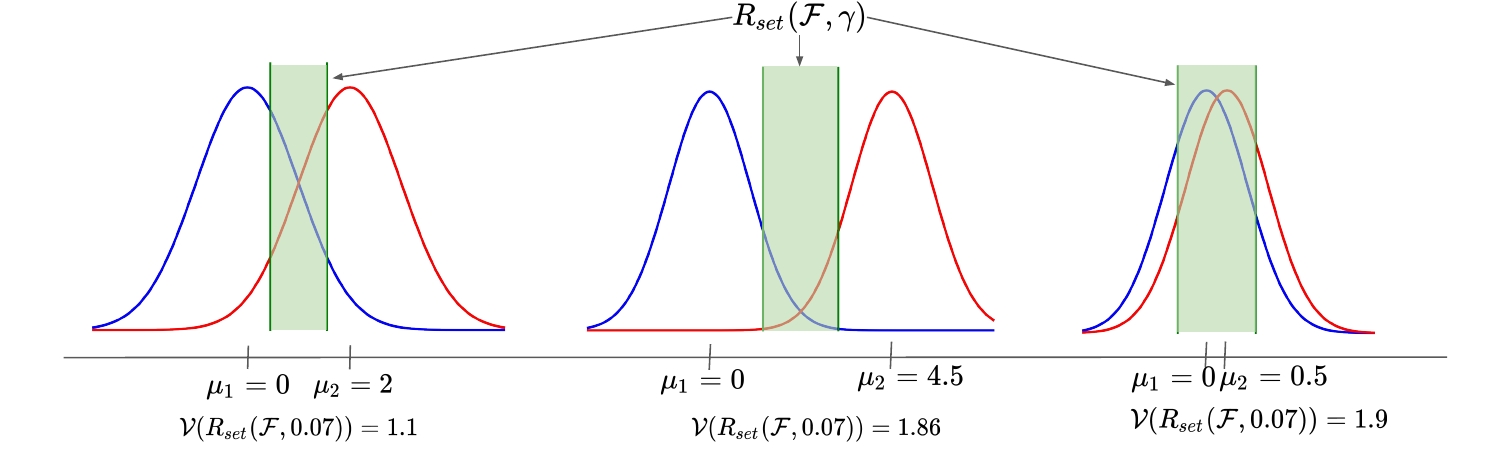}
		\caption{An example that shows that in part (III) as we move the right Gaussian away from a mean of 2 in either direction, the volume of the Rashomon set increases.
		}
		\label{fig:gaus_rset_growth}
	\end{figure*}

\end{proof}

	\begin{figure*}[t]
		\centering
		\begin{tabular}{cc}
			\includegraphics[width=0.22\textwidth]{Figures/bar/35.pdf}\quad	
			\includegraphics[width=0.22\textwidth]{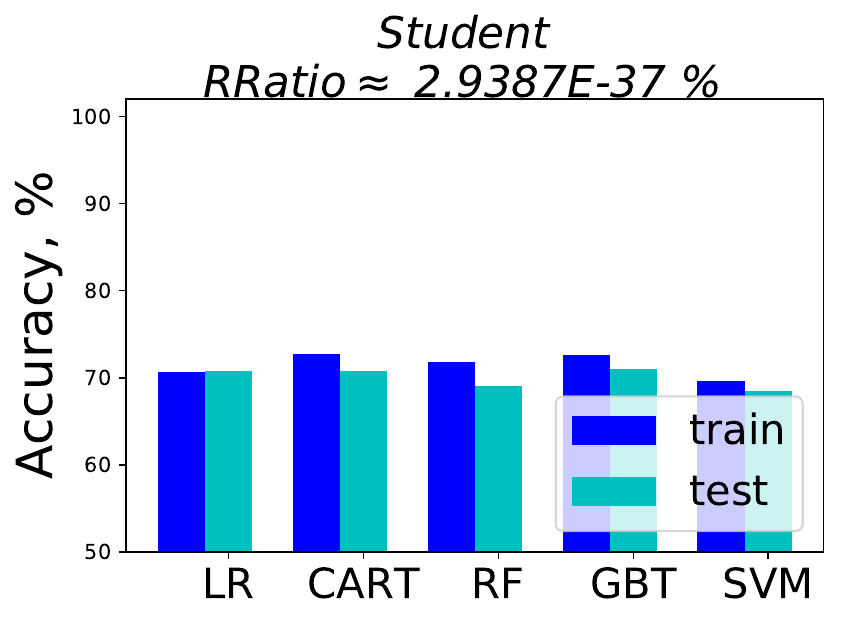}\quad
			\includegraphics[width=0.22\textwidth]{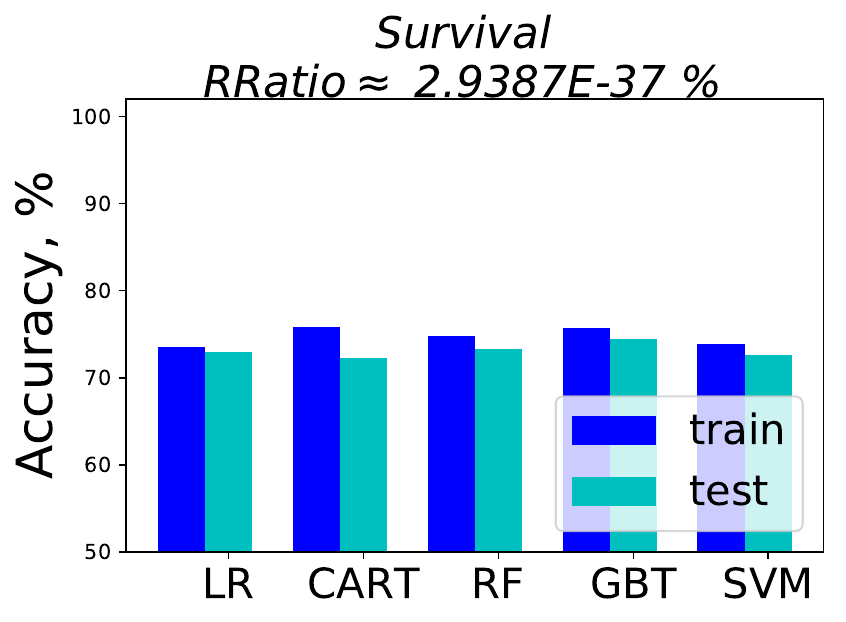}\quad
			\includegraphics[width=0.22\textwidth]{Figures/bar/3.pdf}\\
			\includegraphics[width=0.22\textwidth]{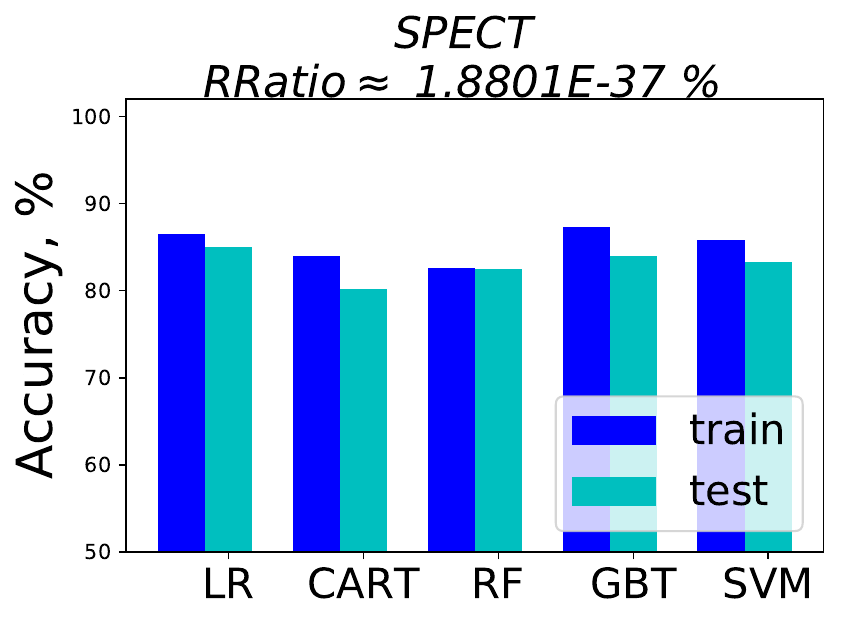}\quad
			\includegraphics[width=0.22\textwidth]{Figures/bar/33.pdf}\quad
			\includegraphics[width=0.22\textwidth]{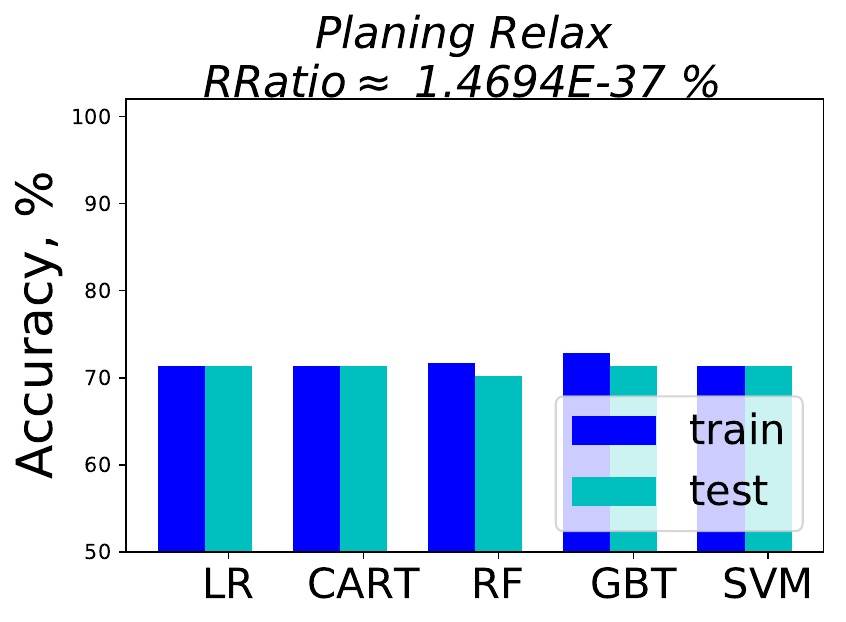}\quad
			\includegraphics[width=0.22\textwidth]{Figures/bar/9.pdf}\\
			\includegraphics[width=0.22\textwidth]{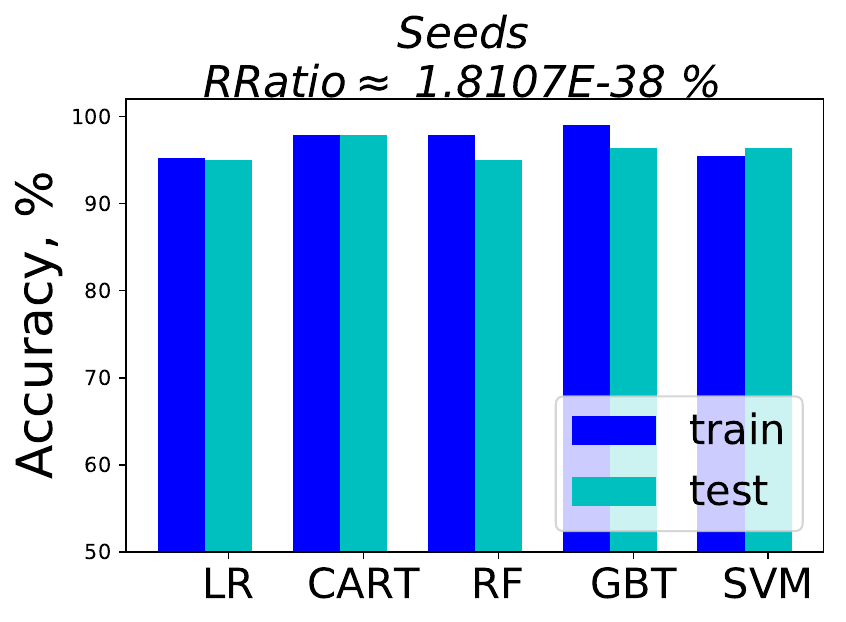}\quad
			\includegraphics[width=0.22\textwidth]{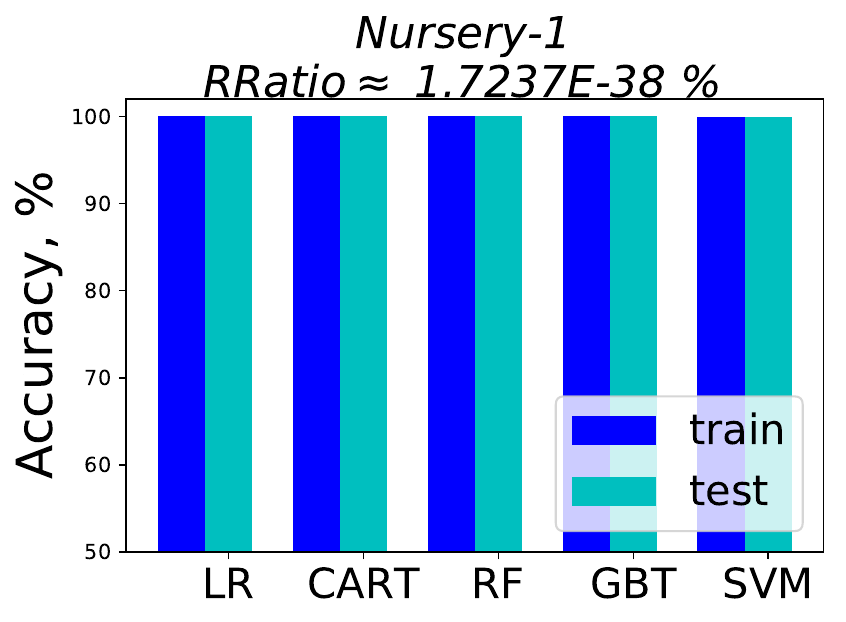}\quad
			\includegraphics[width=0.22\textwidth]{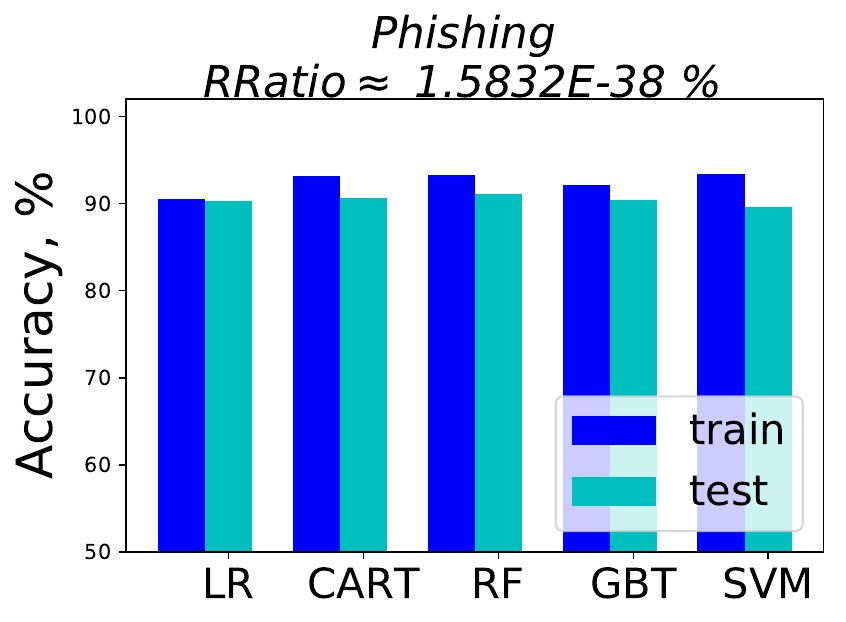}\quad	
			\includegraphics[width=0.22\textwidth]{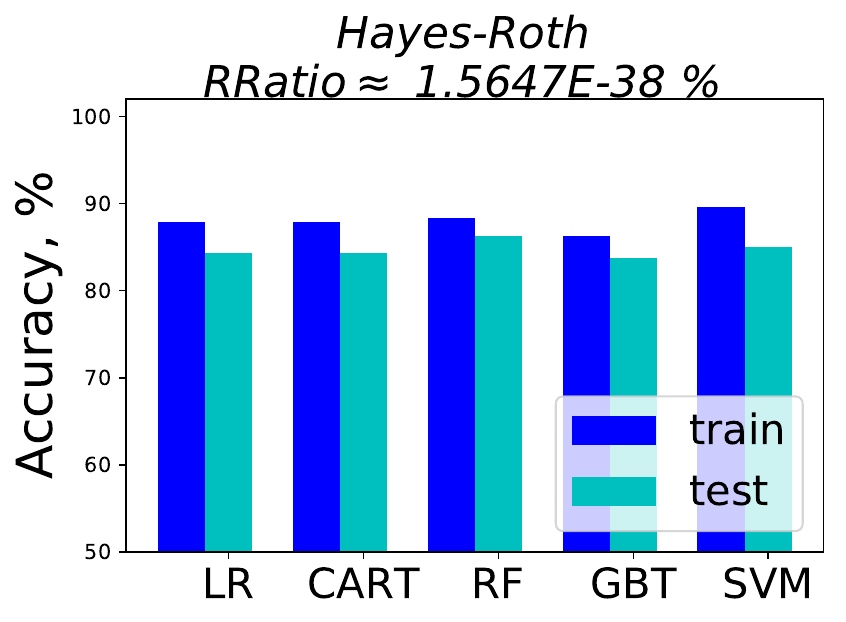}\\
			\includegraphics[width=0.22\textwidth]{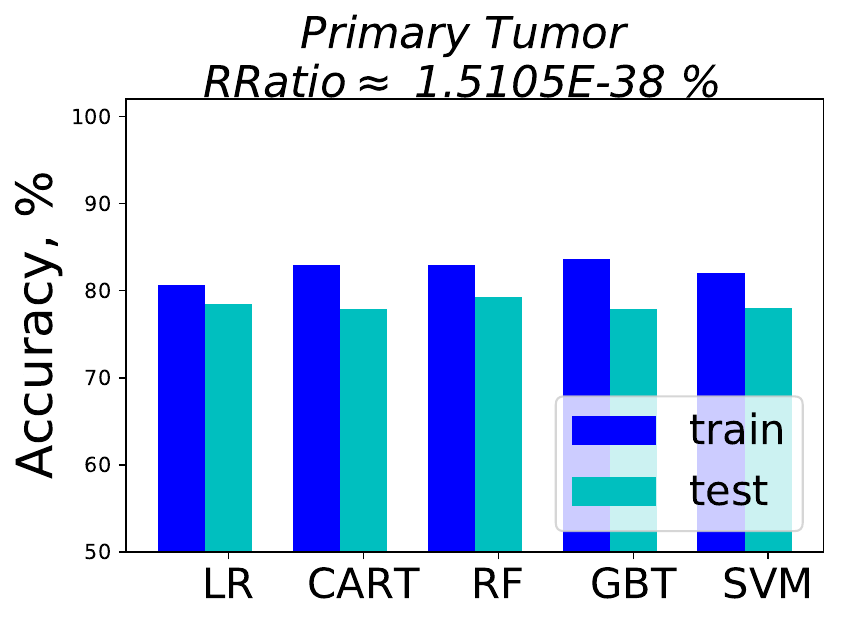}\quad
			\includegraphics[width=0.22\textwidth]{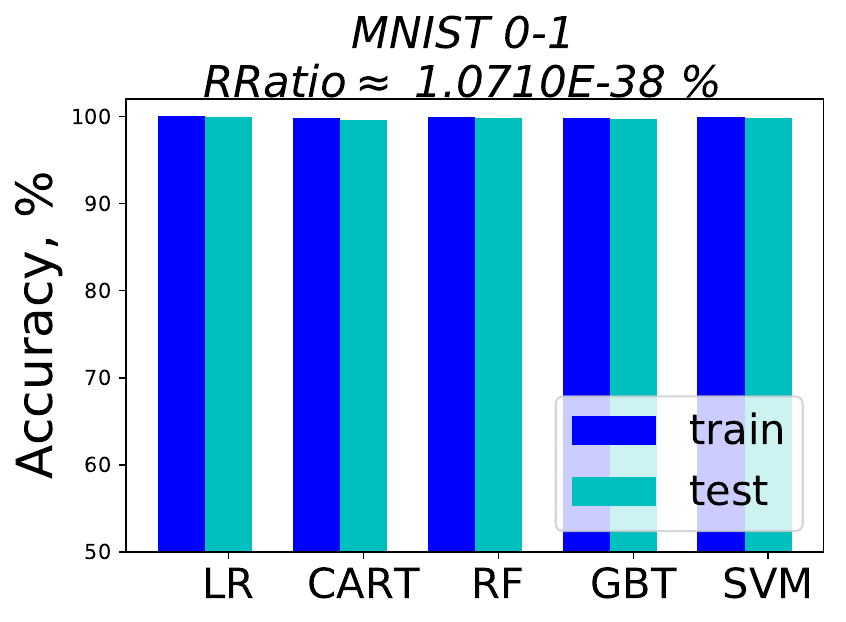}\quad
			\includegraphics[width=0.22\textwidth]{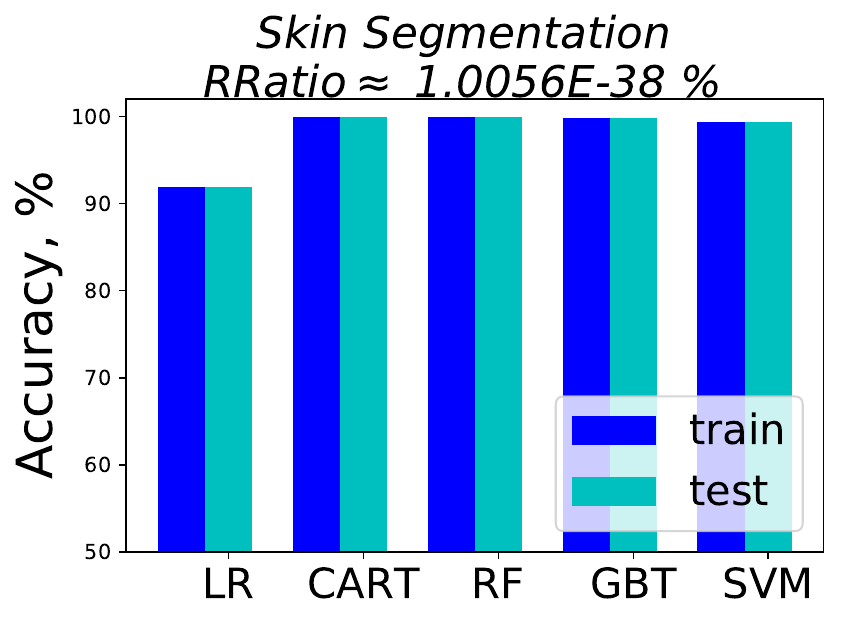}\quad
			\includegraphics[width=0.22\textwidth]{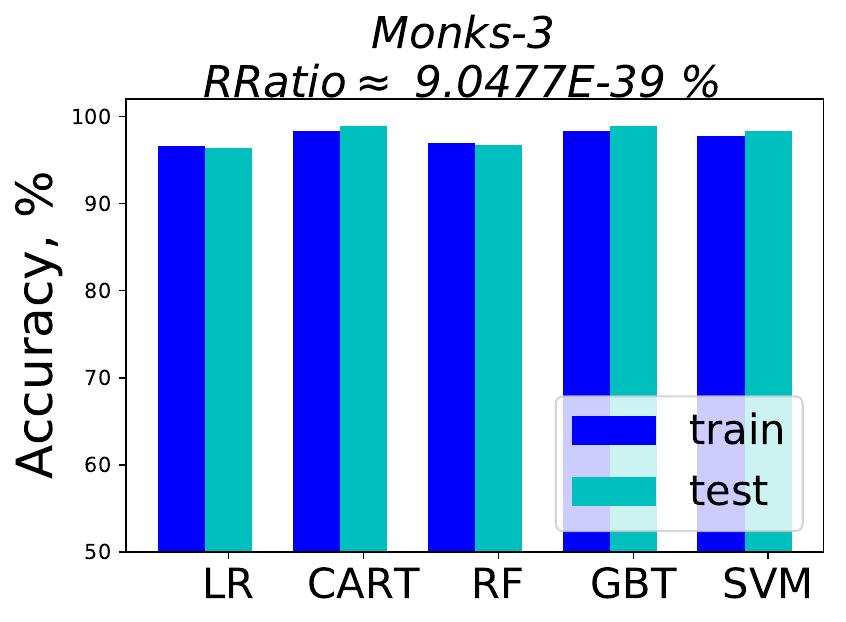}\\
			\includegraphics[width=0.22\textwidth]{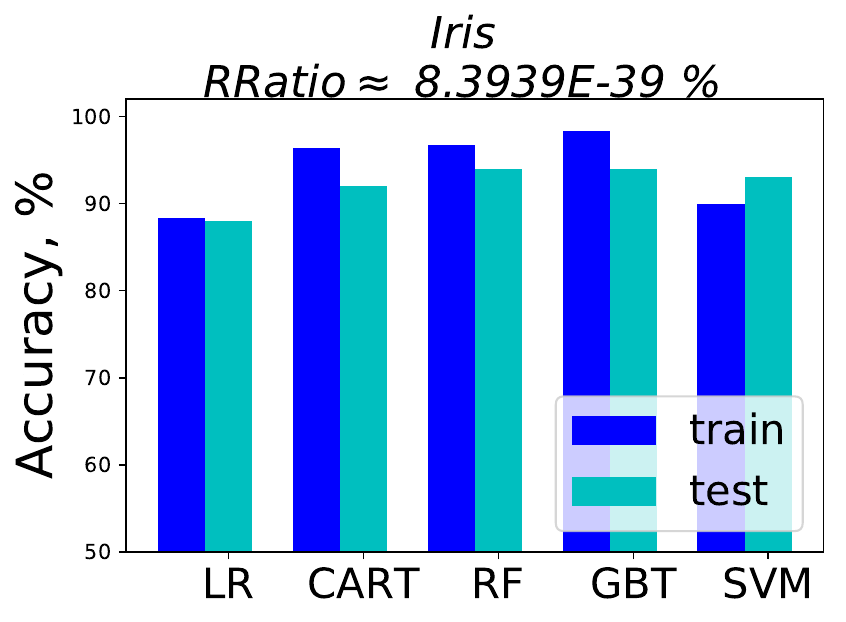}\quad
			\includegraphics[width=0.22\textwidth]{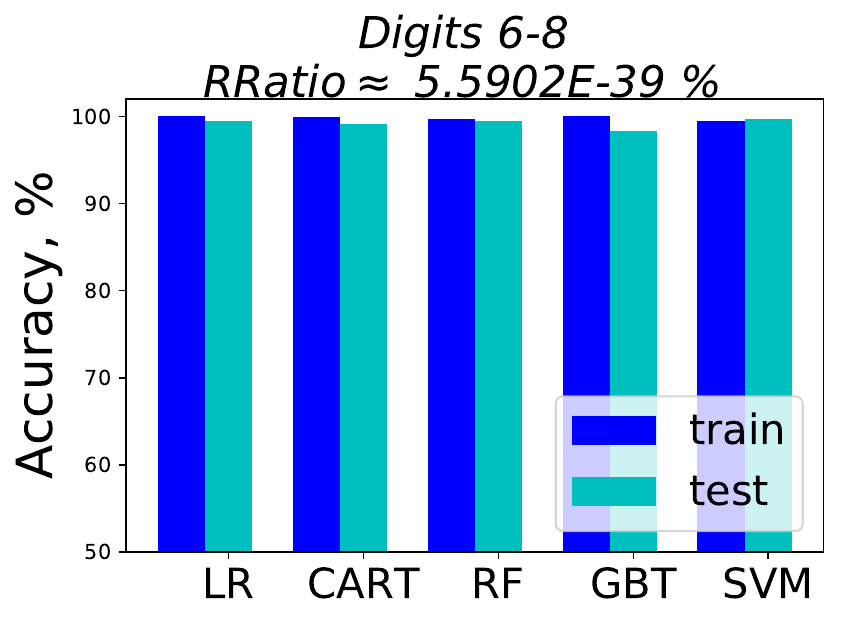}\quad
			\includegraphics[width=0.22\textwidth]{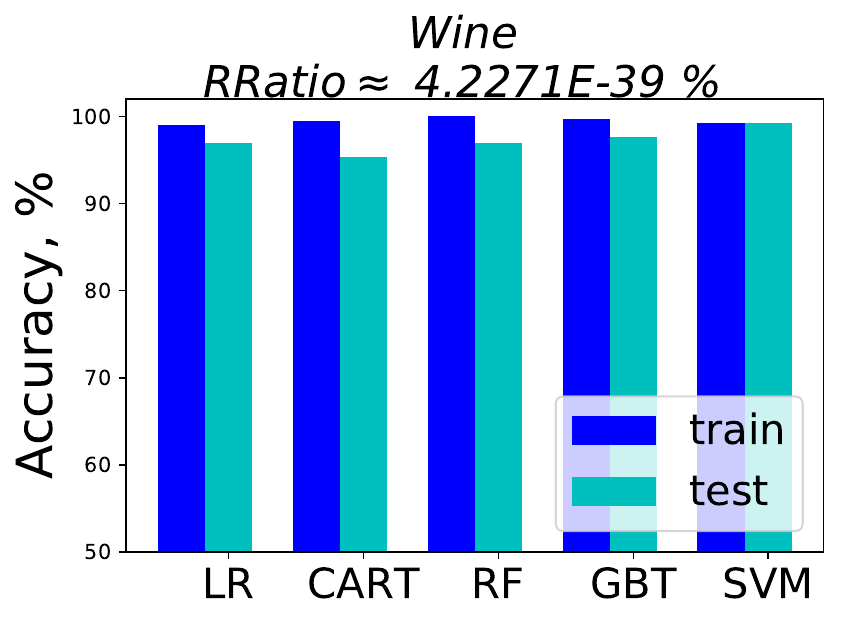}\quad
			\includegraphics[width=0.22\textwidth]{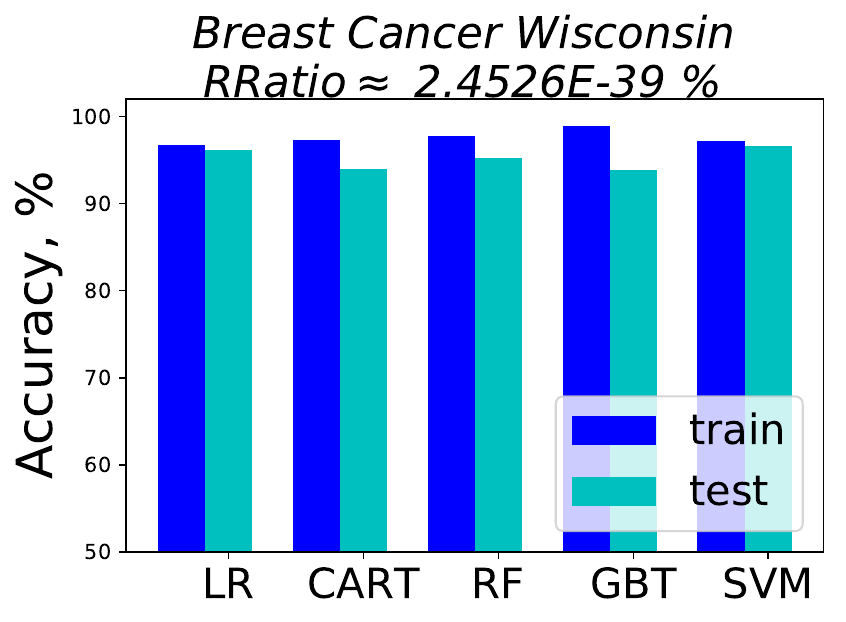}\\

		\end{tabular}
		\caption{Performance of five machine learning algorithms with regularization for the UCI classification data sets. Data sets are listed in decreasing order of Rashomon ratio. Rashomon ratios, train and test accuracies are averaged over ten folds for data sets with more than 200 points and over five folds for data sets with less than 200 points. These plots continue in  Figure \ref{fig:bar_plots_2}. In these cases, test performance seems to be similar across algorithms. This will not be true in all cases as the Rashomon set becomes smaller, in  Figure \ref{fig:bar_plots_2}.
		\label{fig:bar_plots_1}}
	\end{figure*}

	\begin{figure*}[t]
		\centering
		\begin{tabular}{cc}
		    \quad
			\includegraphics[width=0.22\textwidth]{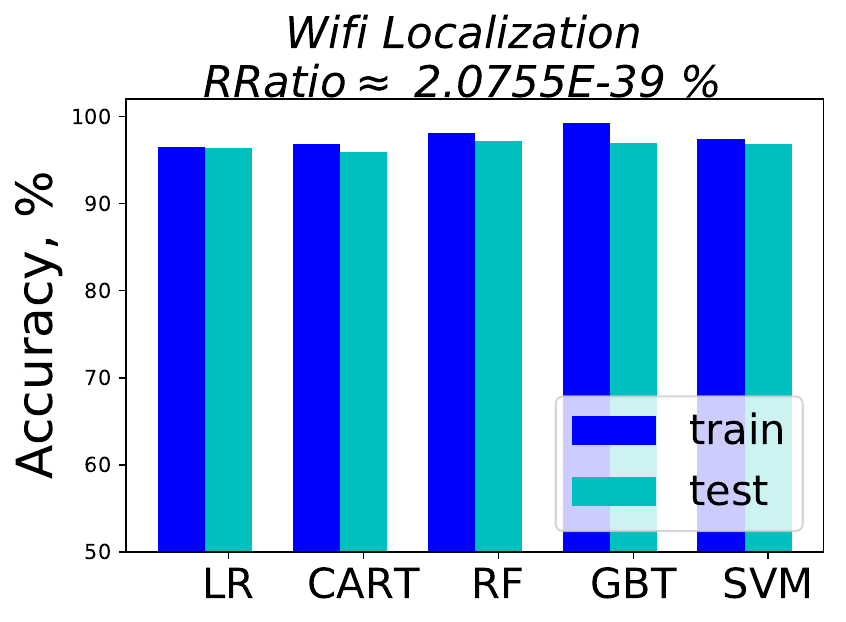}\quad
			\includegraphics[width=0.22\textwidth]{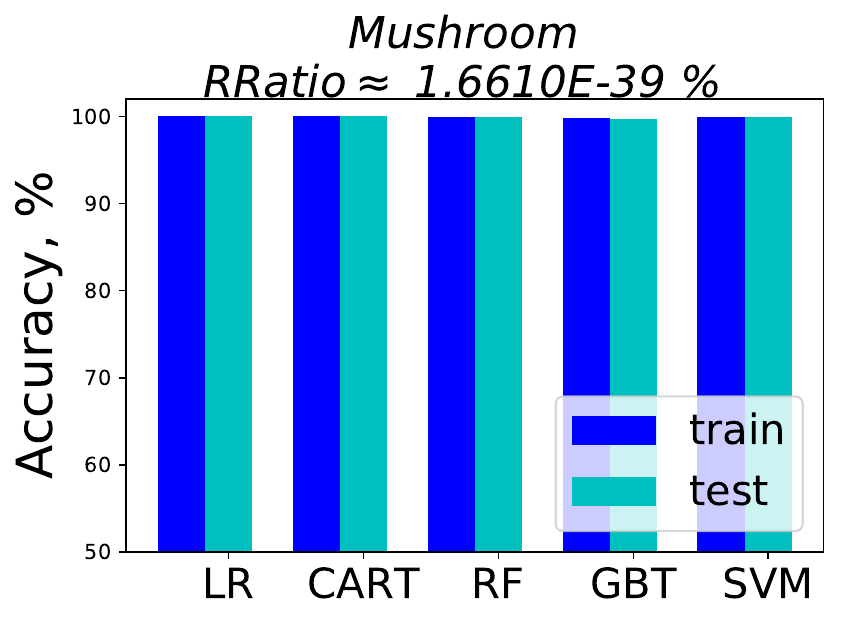}\quad	
			\includegraphics[width=0.22\textwidth]{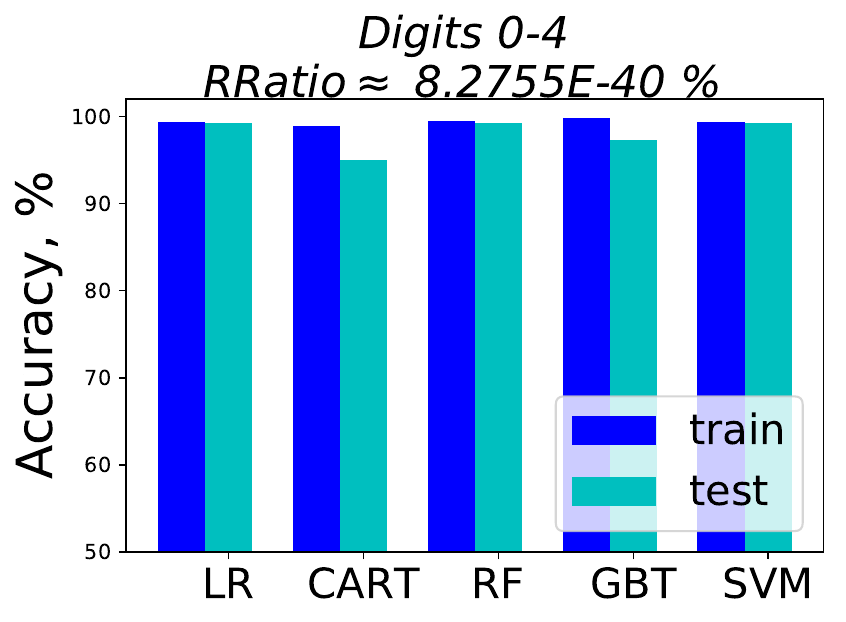}\quad
			\includegraphics[width=0.22\textwidth]{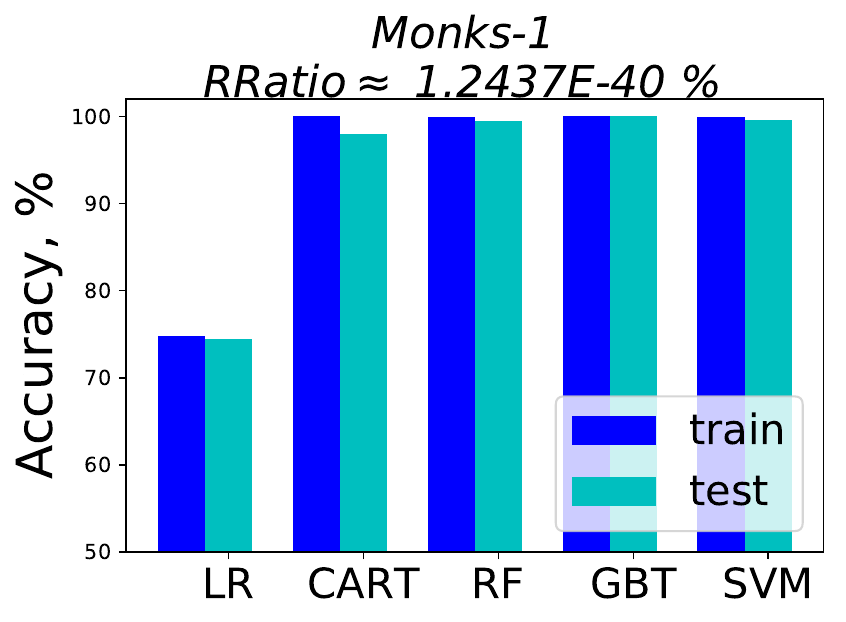}\\
			\includegraphics[width=0.22\textwidth]{Figures/bar/21.pdf}\quad
			\includegraphics[width=0.22\textwidth]{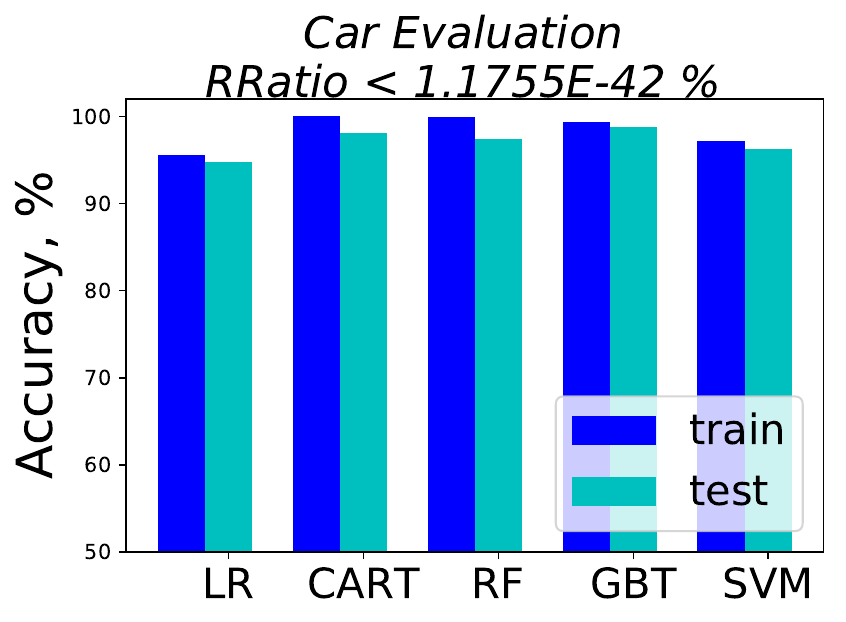}\quad
			\includegraphics[width=0.22\textwidth]{Figures/bar/1.pdf}\quad
			\includegraphics[width=0.22\textwidth]{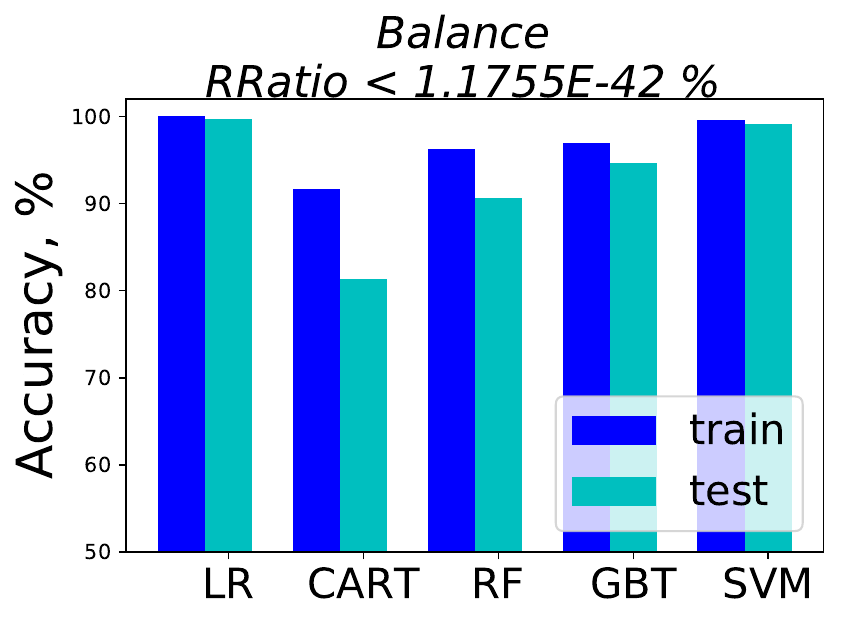}\\
		    \includegraphics[width=0.22\textwidth]{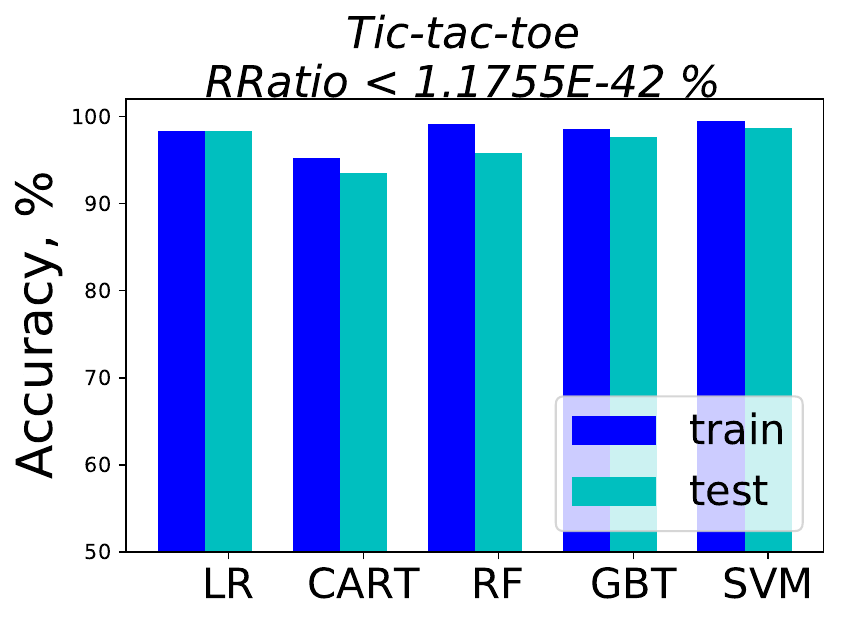}\quad	
			\includegraphics[width=0.22\textwidth]{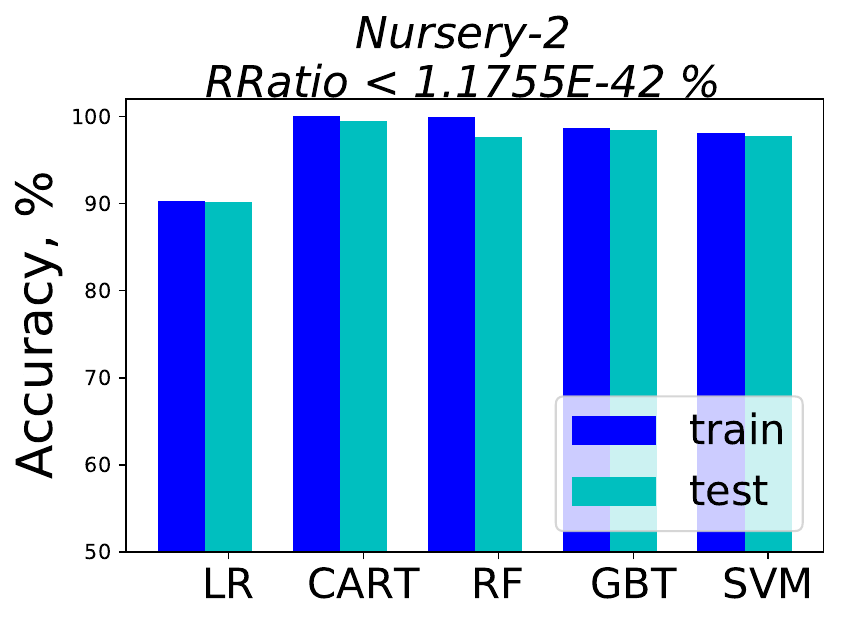}\quad
			\includegraphics[width=0.22\textwidth]{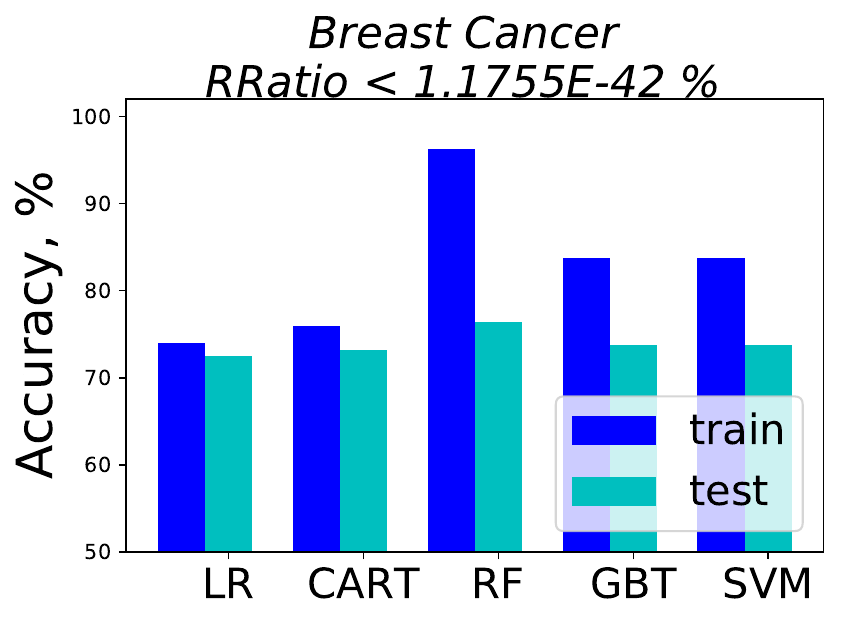}\quad
			\includegraphics[width=0.22\textwidth]{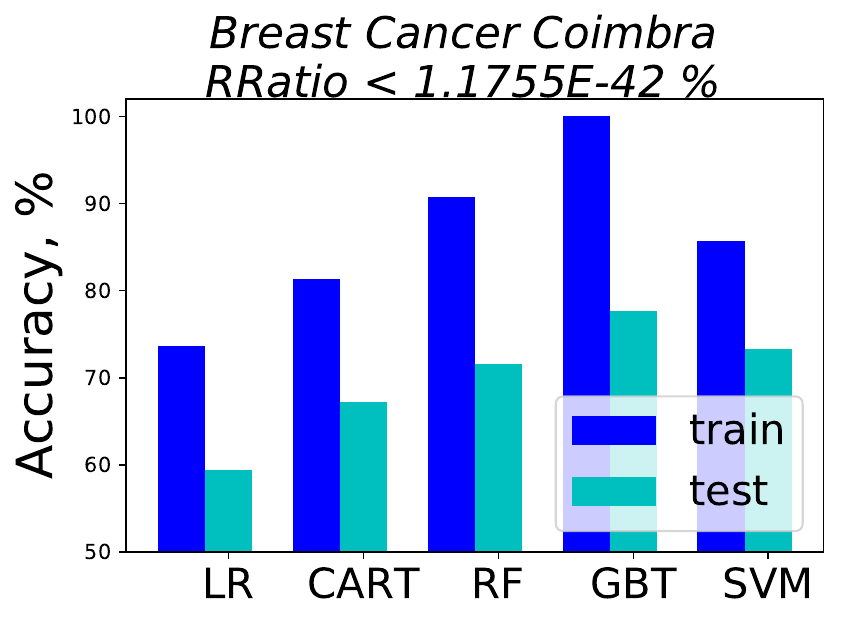}\\
			\includegraphics[width=0.22\textwidth]{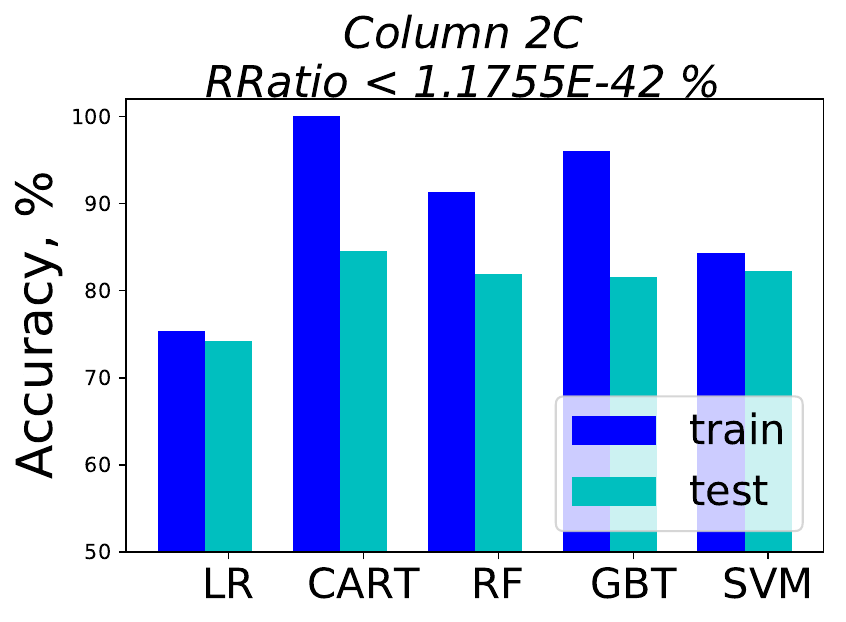}\quad
			\includegraphics[width=0.22\textwidth]{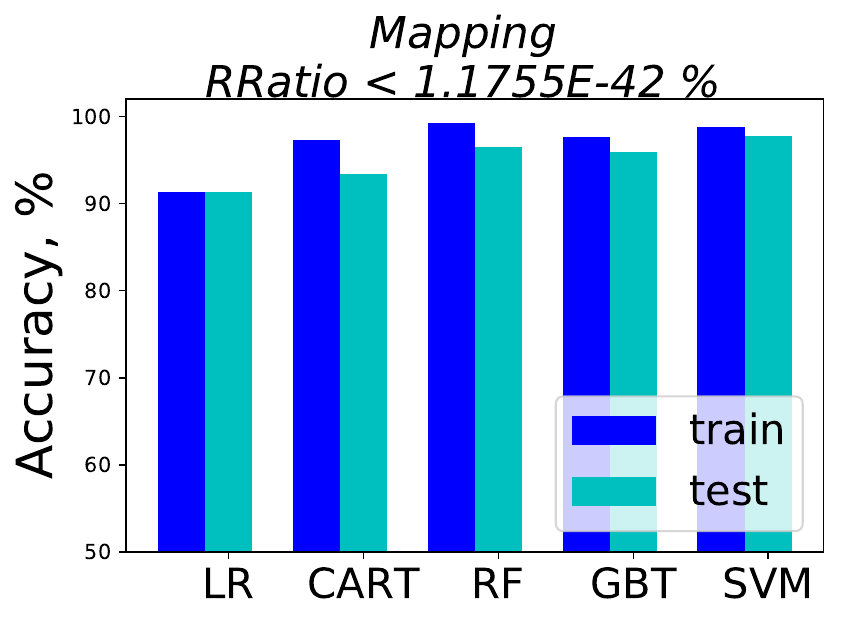}\quad	
			\includegraphics[width=0.22\textwidth]{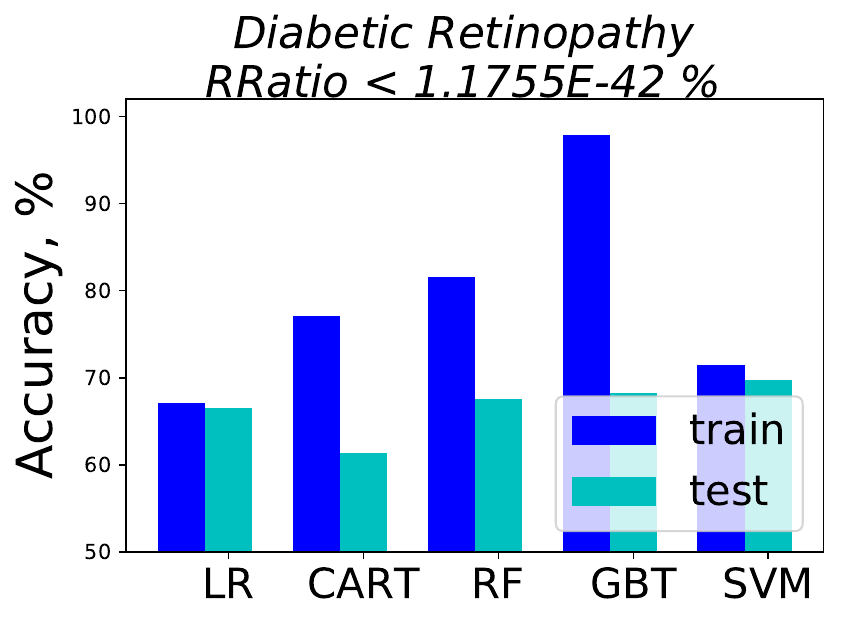}\quad
			\includegraphics[width=0.22\textwidth]{Figures/bar/32.pdf}\\
			\includegraphics[width=0.22\textwidth]{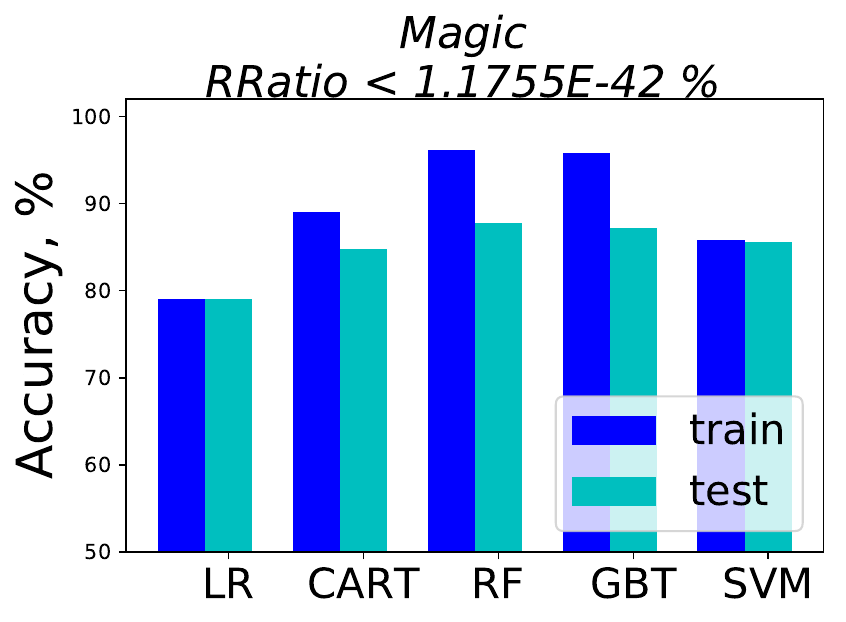}\quad	
			\includegraphics[width=0.22\textwidth]{Figures/bar/37.pdf}\\	
		\end{tabular}
		\caption{Performance of five machine learning algorithms with regularization for the UCI classification data sets. Data sets are listed in decreasing order of the Rashomon ratio, continued from  Figure \ref{fig:bar_plots_1}. Rashomon ratios, training accuracies, and test accuracies are averaged over ten folds for data sets with more than 200 points and over five folds for data sets with less than 200 points. Test performance sometimes varies across algorithms.}
		\label{fig:bar_plots_2}
	\end{figure*}

		\begin{figure*}[t]
		\centering
		\begin{tabular}{cc}
			\includegraphics[width=0.22\textwidth]{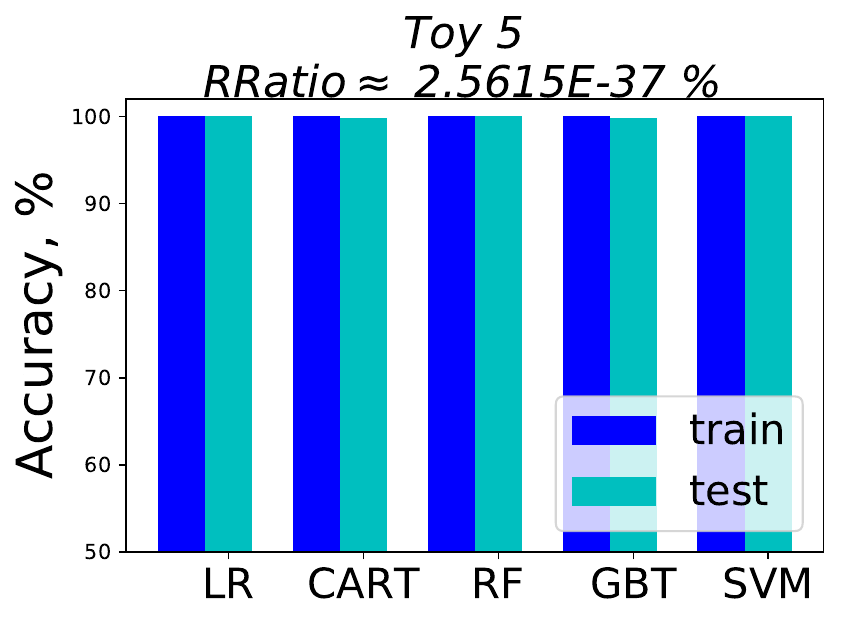}\quad	
			\includegraphics[width=0.22\textwidth]{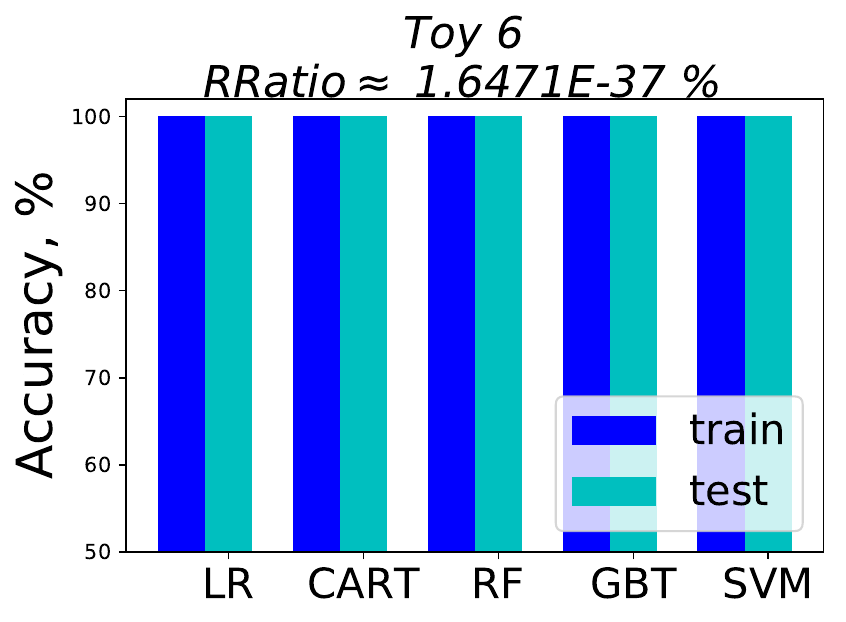}\quad
			\includegraphics[width=0.22\textwidth]{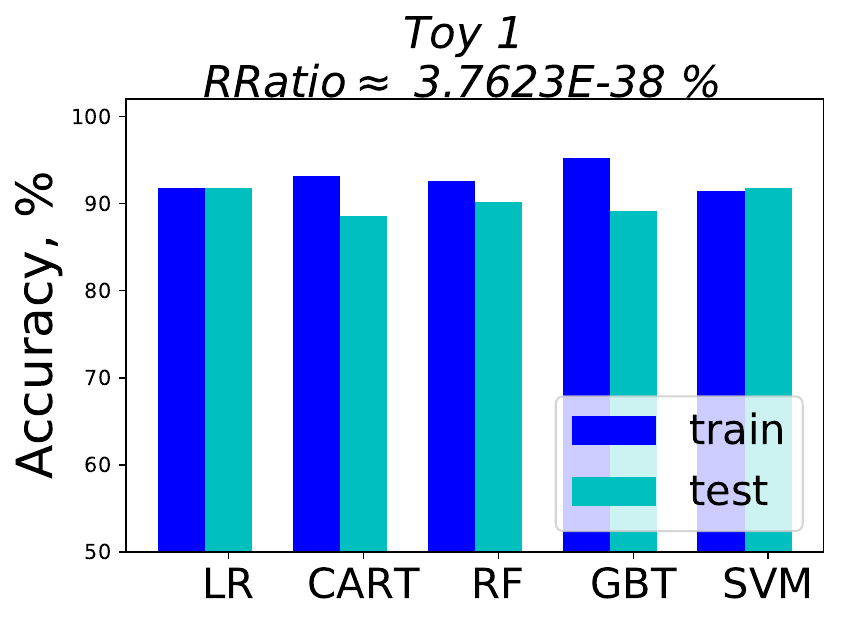}\quad
			\includegraphics[width=0.22\textwidth]{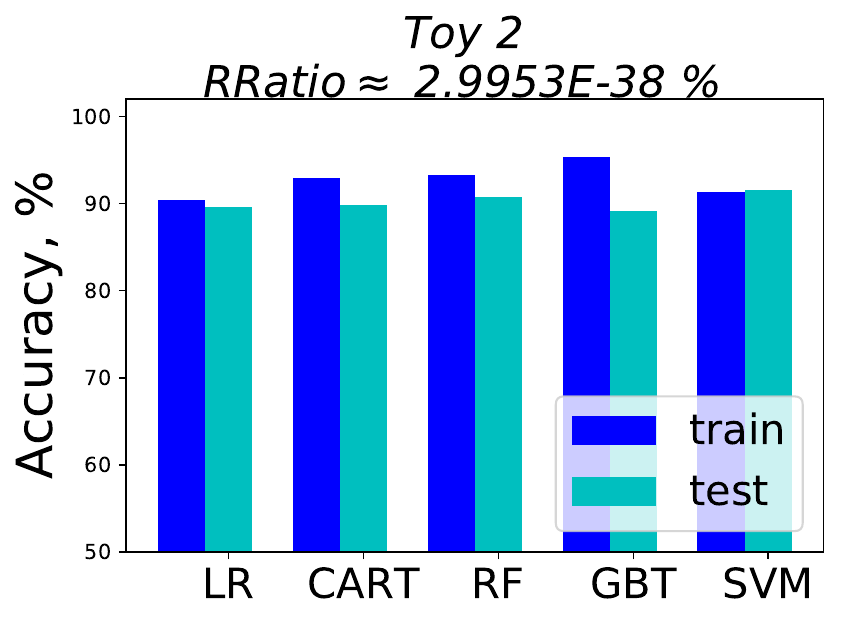}\\
			\includegraphics[width=0.22\textwidth]{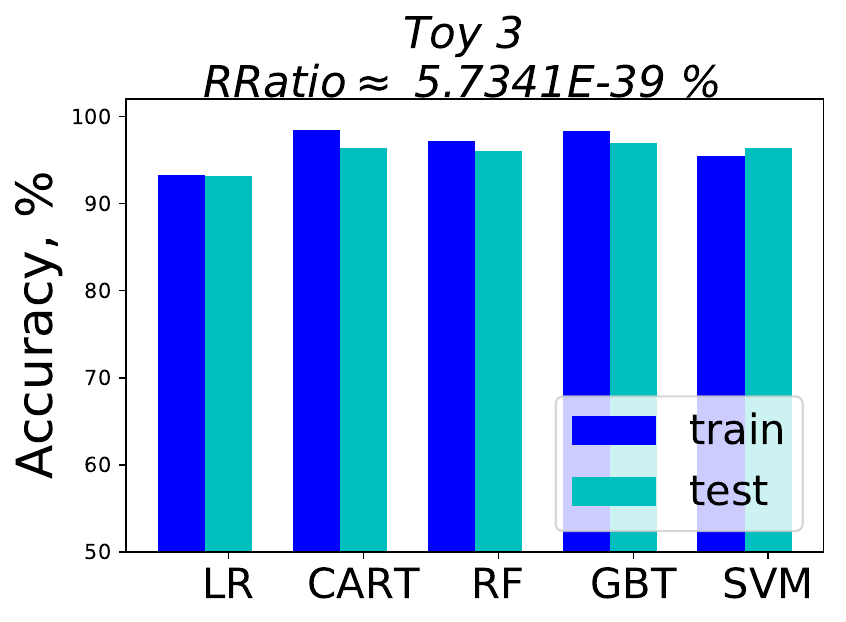}\quad	
			\includegraphics[width=0.22\textwidth]{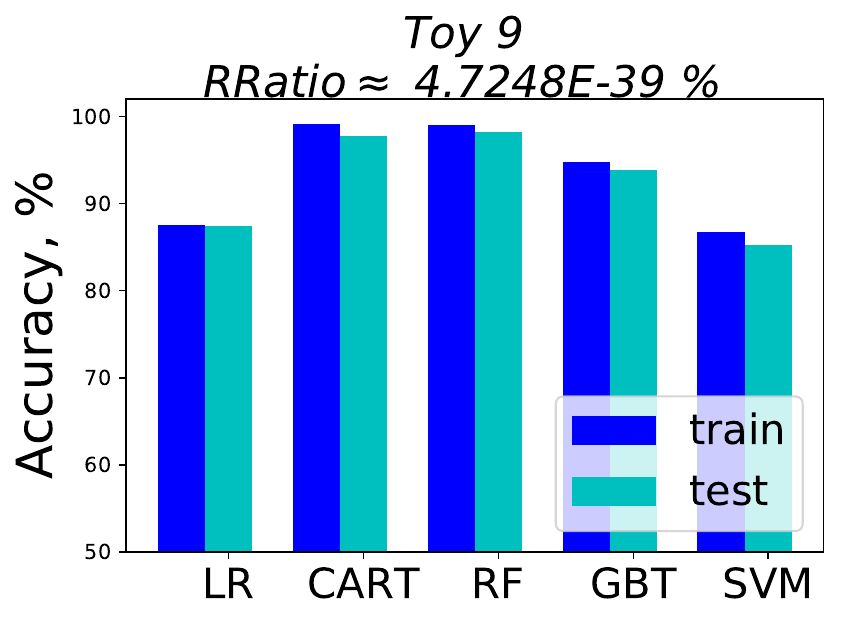}\quad
			\includegraphics[width=0.22\textwidth]{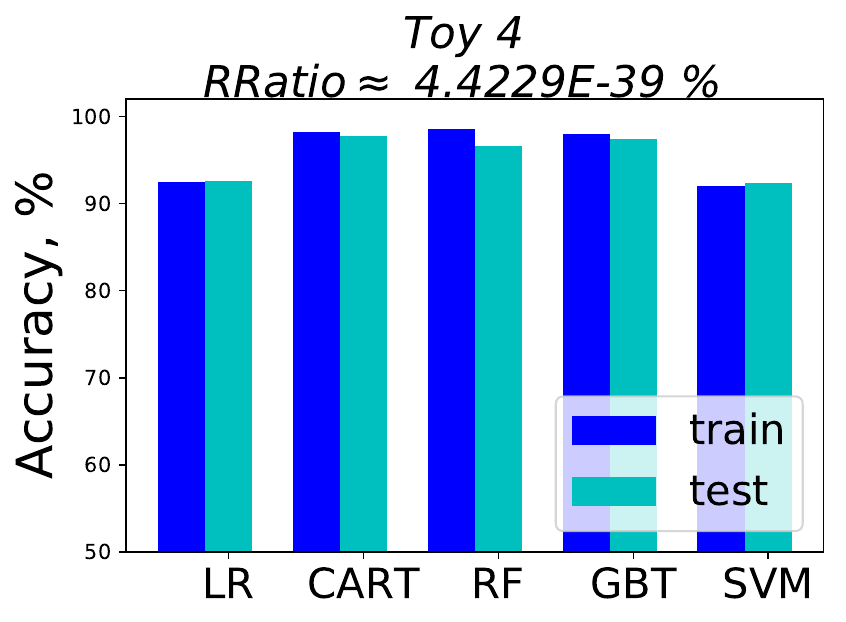}\quad
			\includegraphics[width=0.22\textwidth]{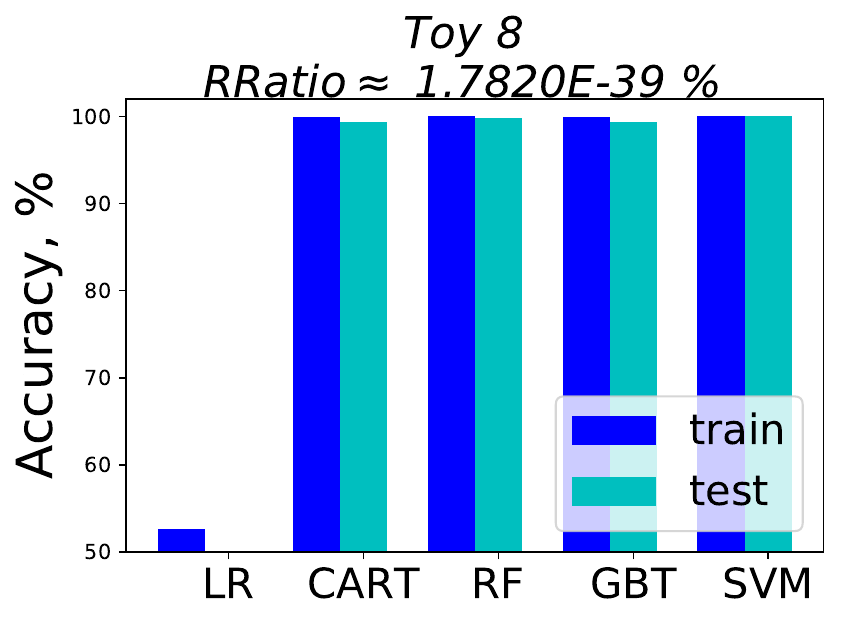}\\
		    \includegraphics[width=0.22\textwidth]{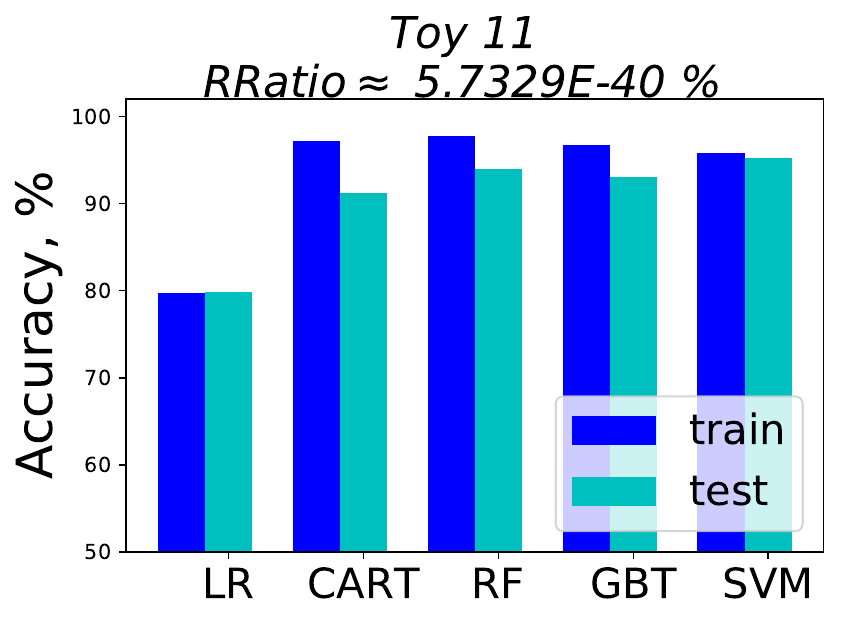}\quad	
			\includegraphics[width=0.22\textwidth]{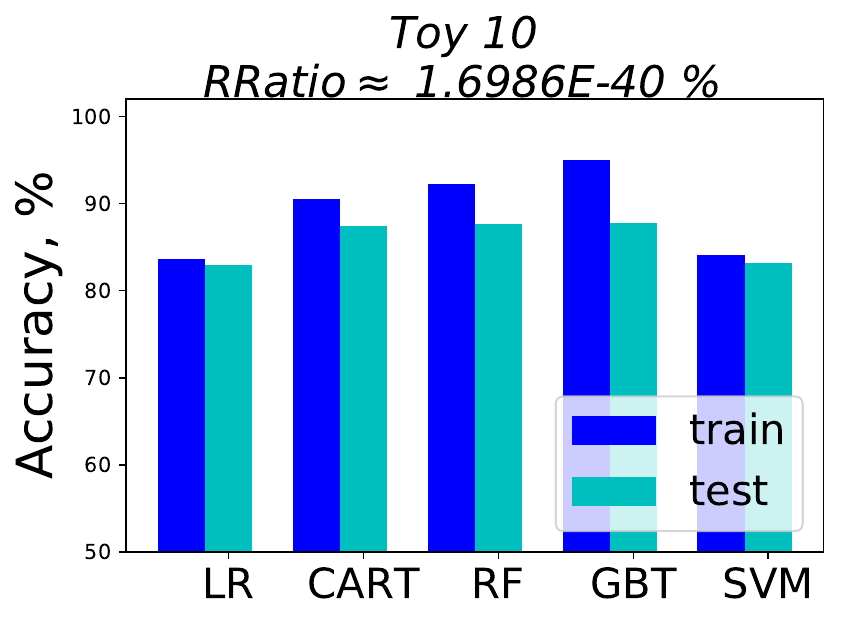}\quad
			\includegraphics[width=0.22\textwidth]{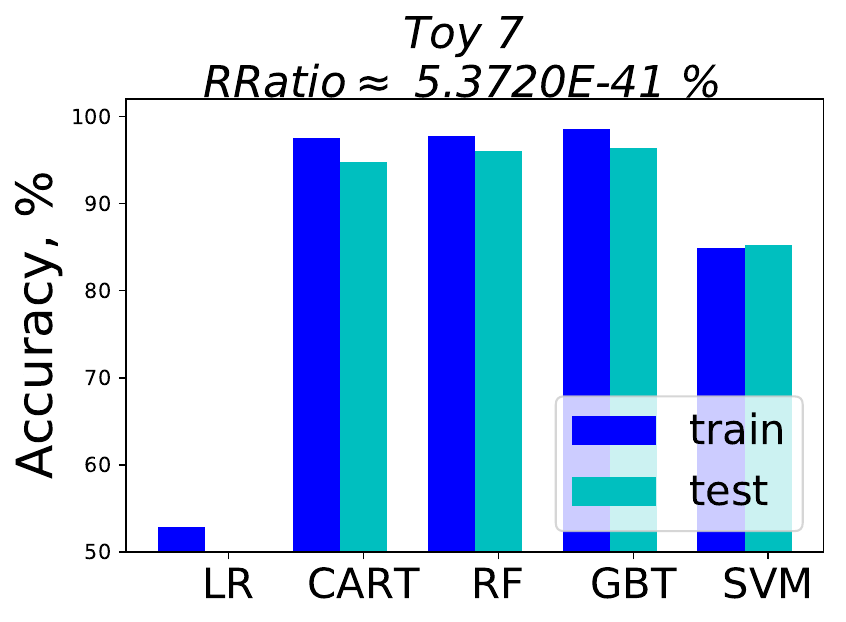}\quad
			\includegraphics[width=0.22\textwidth]{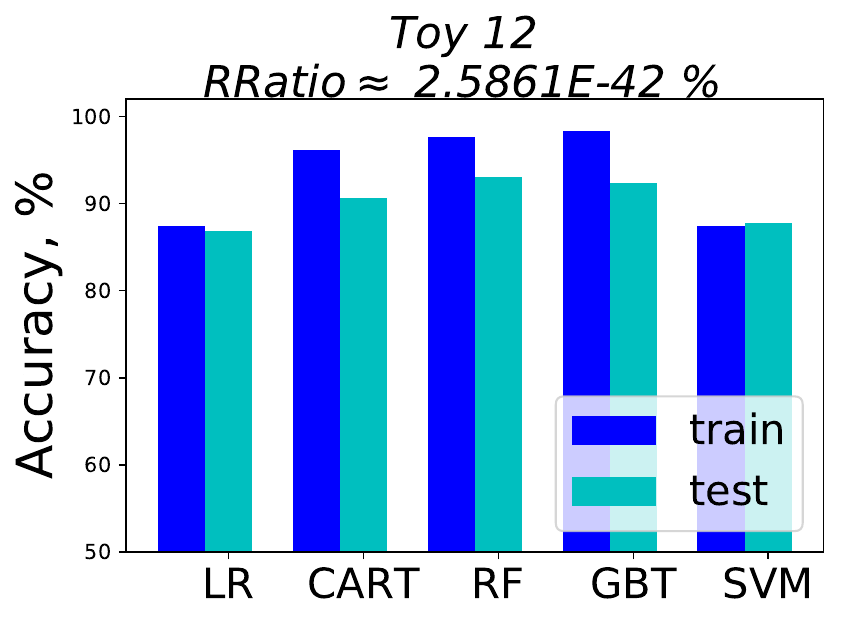}\\
		\end{tabular}
		\caption{Performance of five machine learning algorithms with regularization for the synthetic data sets with real-valued features. Data sets are listed in decreasing order of the Rashomon ratio. Rashomon ratios, train and test accuracies are averaged over ten folds. }
		\label{fig:bar_plots_3}
	\end{figure*}

	\begin{figure*}[t]
		\centering
		\begin{tabular}{cc}
			\includegraphics[width=0.22\textwidth]{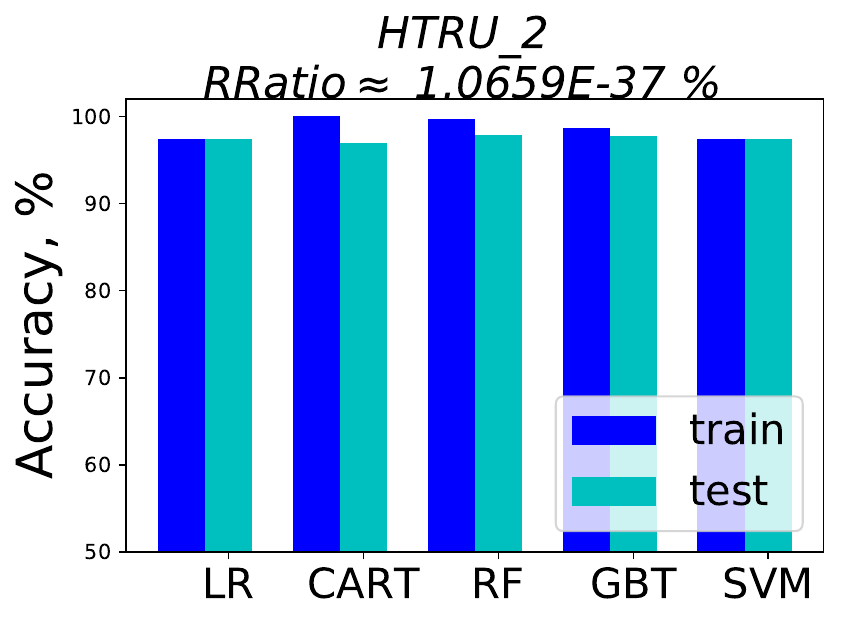}\quad	
			\includegraphics[width=0.22\textwidth]{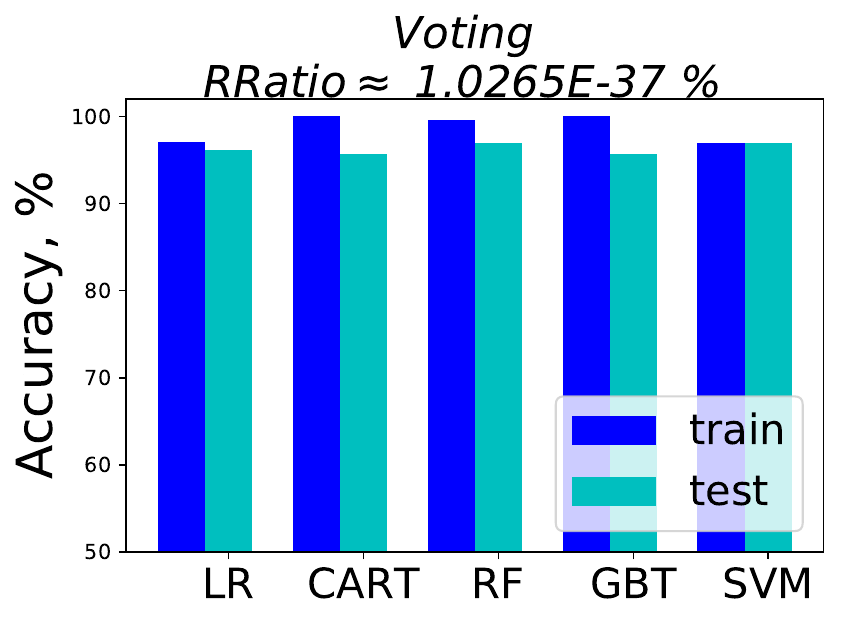}\quad
			\includegraphics[width=0.22\textwidth]{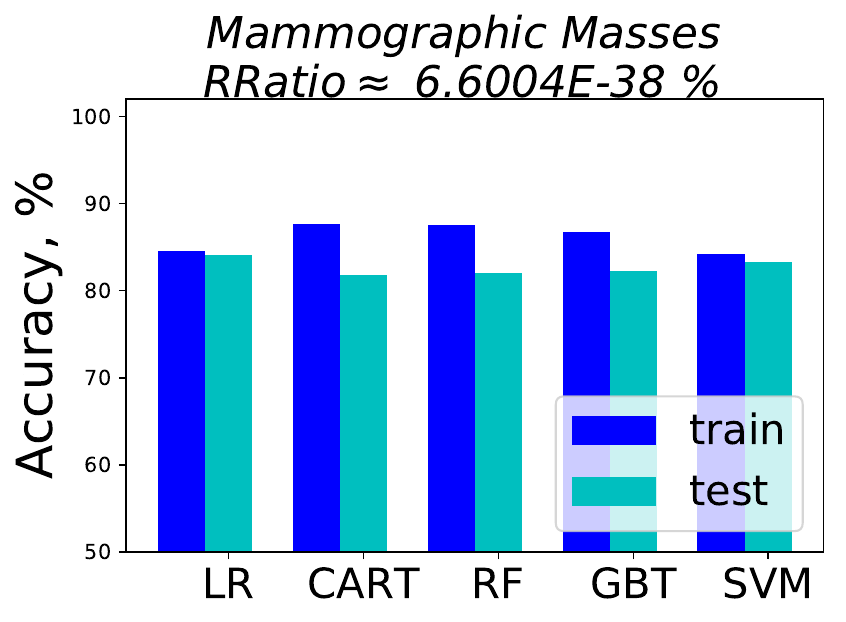}\quad
			\includegraphics[width=0.22\textwidth]{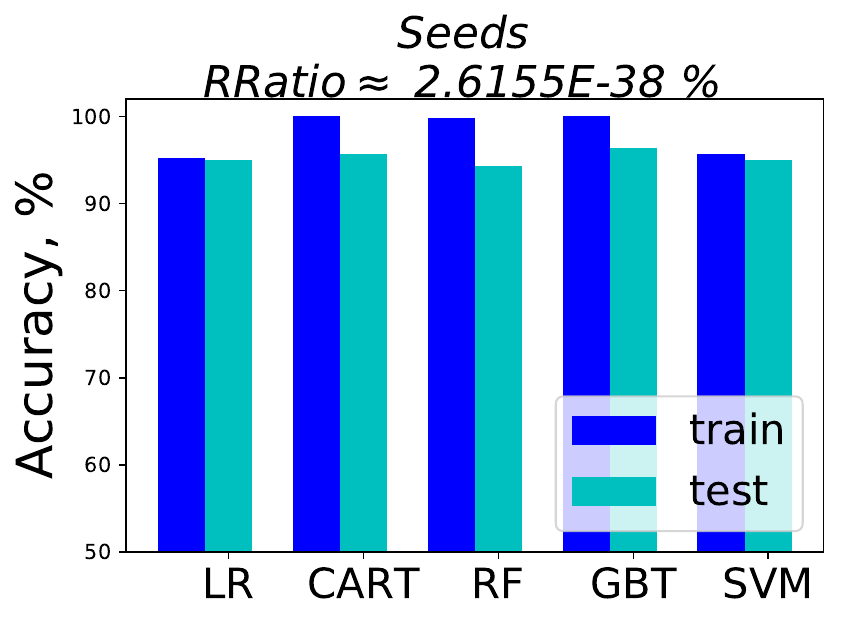}\\
			\includegraphics[width=0.22\textwidth]{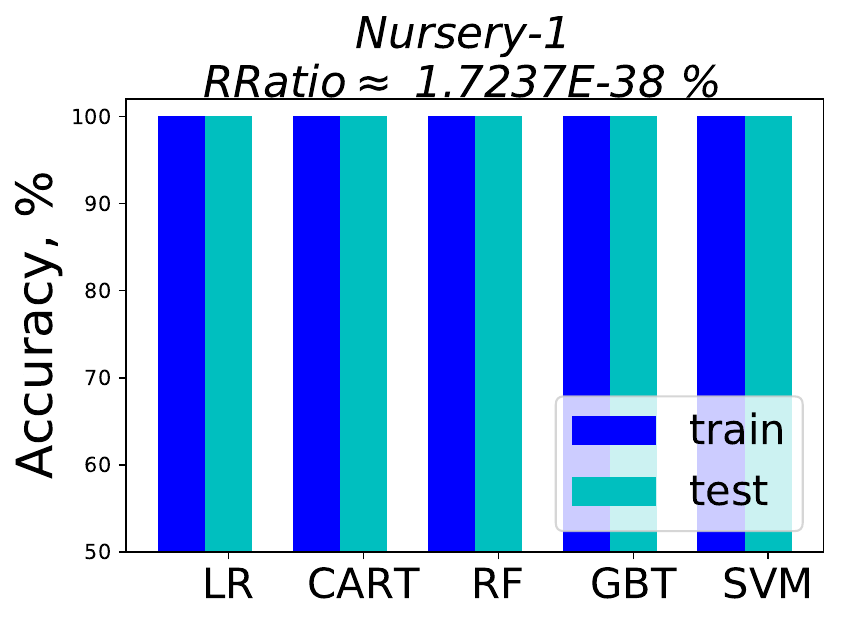}\quad
			\includegraphics[width=0.22\textwidth]{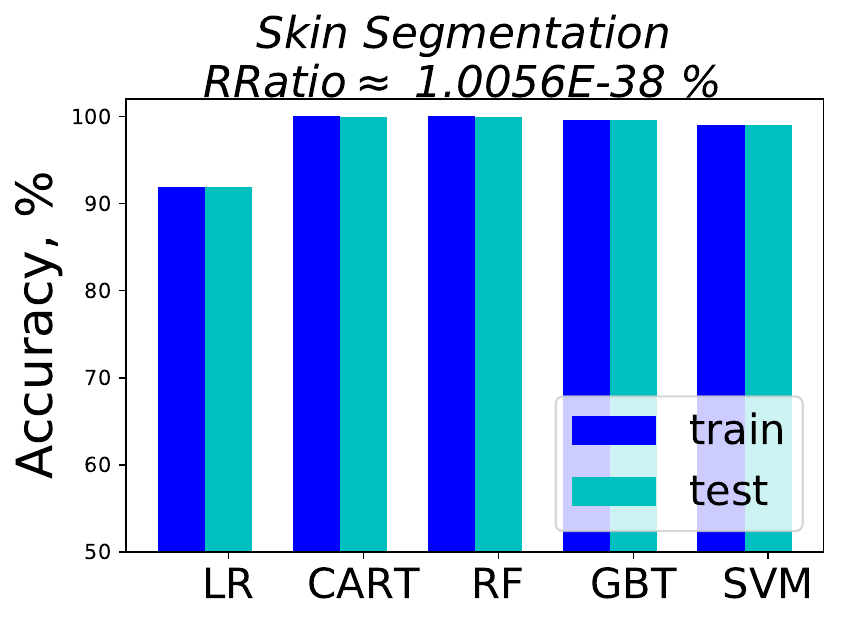}\quad
			\includegraphics[width=0.22\textwidth]{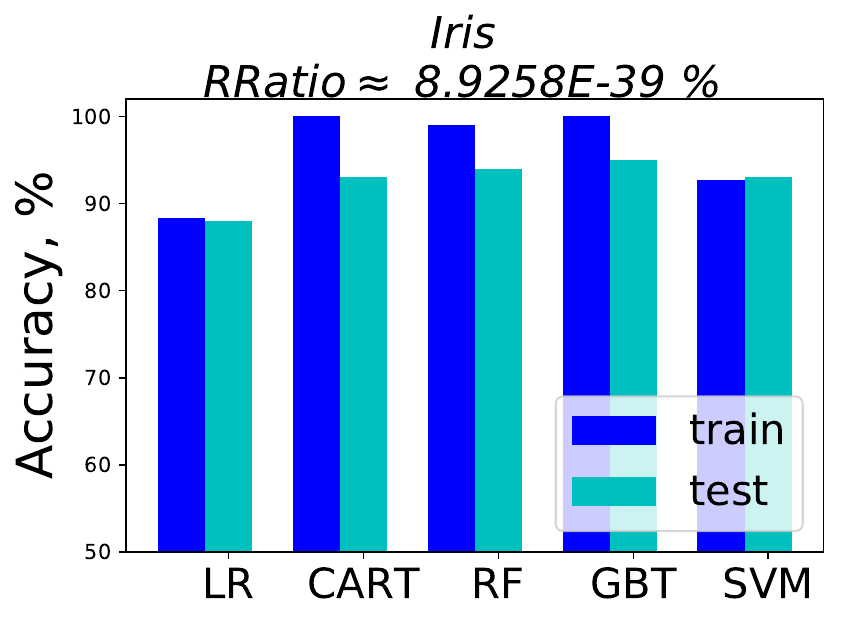}\quad
			\includegraphics[width=0.22\textwidth]{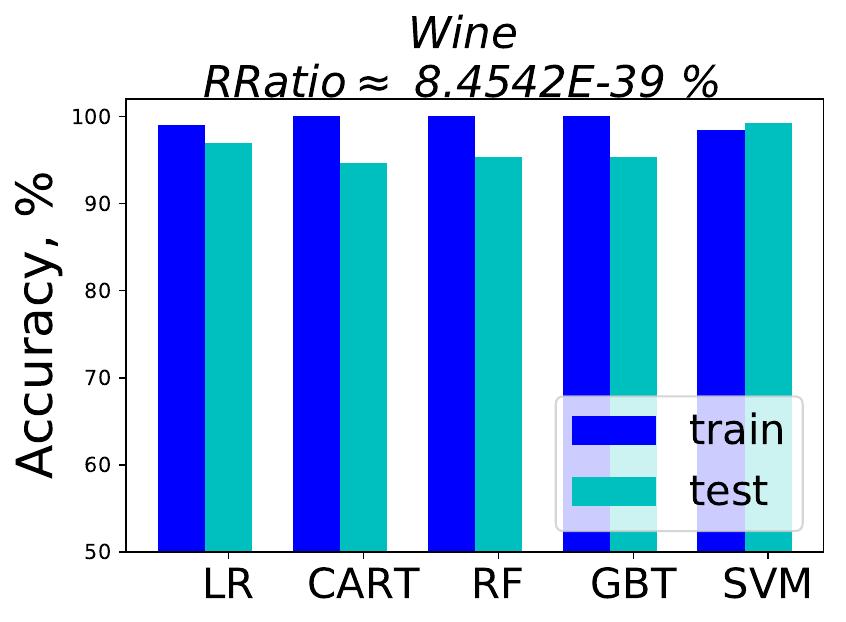}\\
			\includegraphics[width=0.22\textwidth]{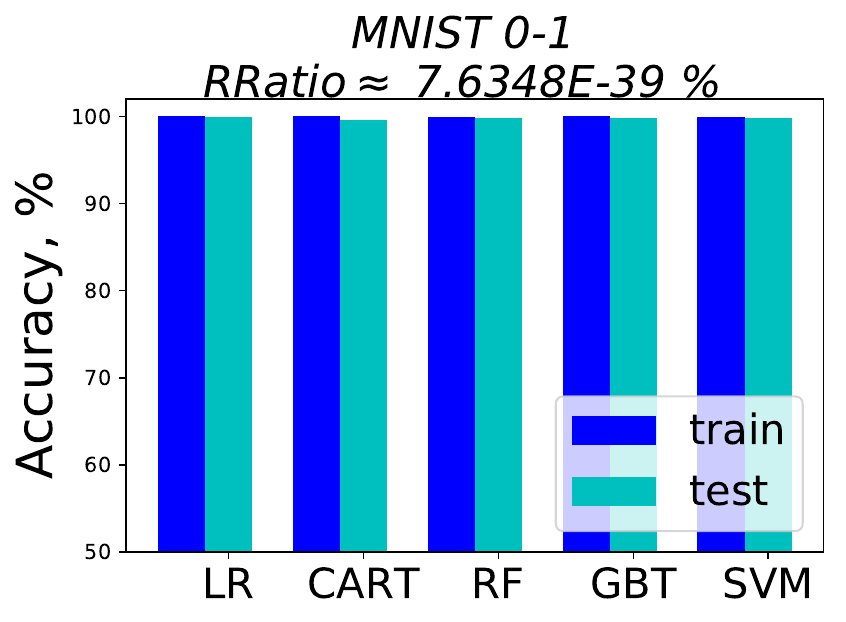}\quad	
			\includegraphics[width=0.22\textwidth]{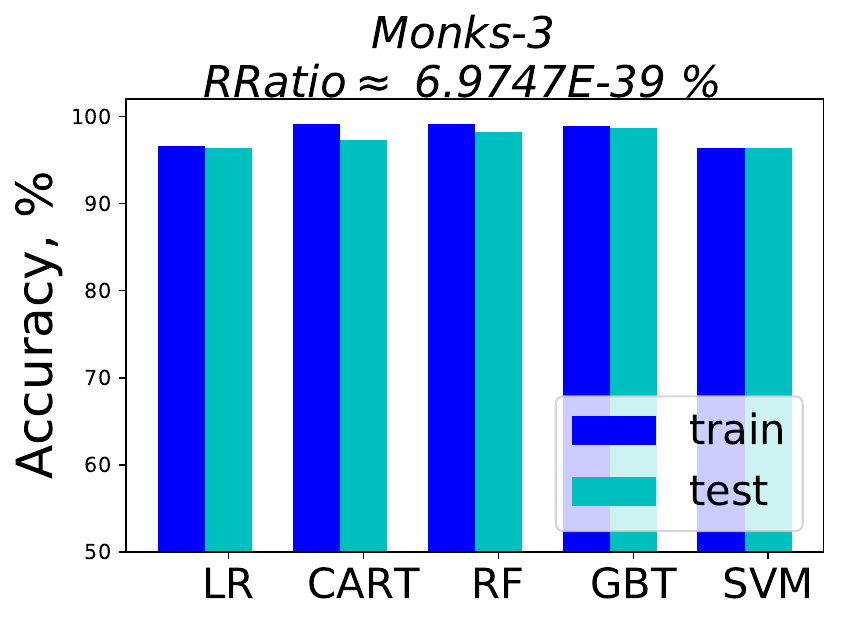}\quad
			\includegraphics[width=0.22\textwidth]{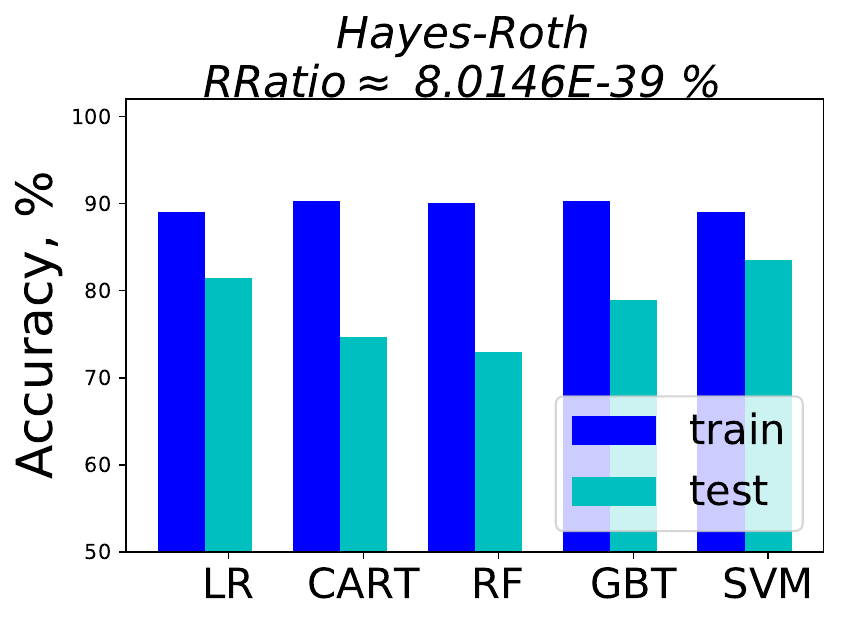}\quad
			\includegraphics[width=0.22\textwidth]{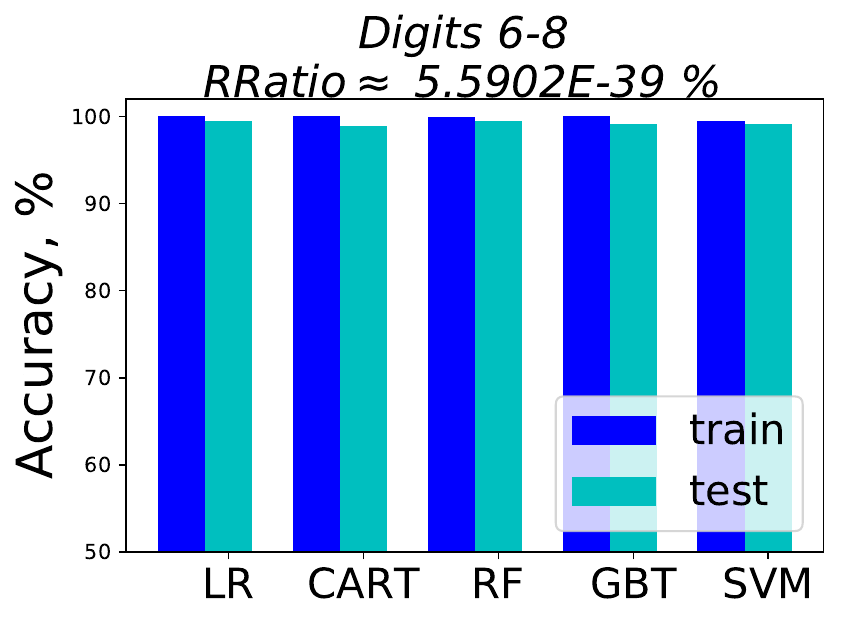}\\	
			\includegraphics[width=0.22\textwidth]{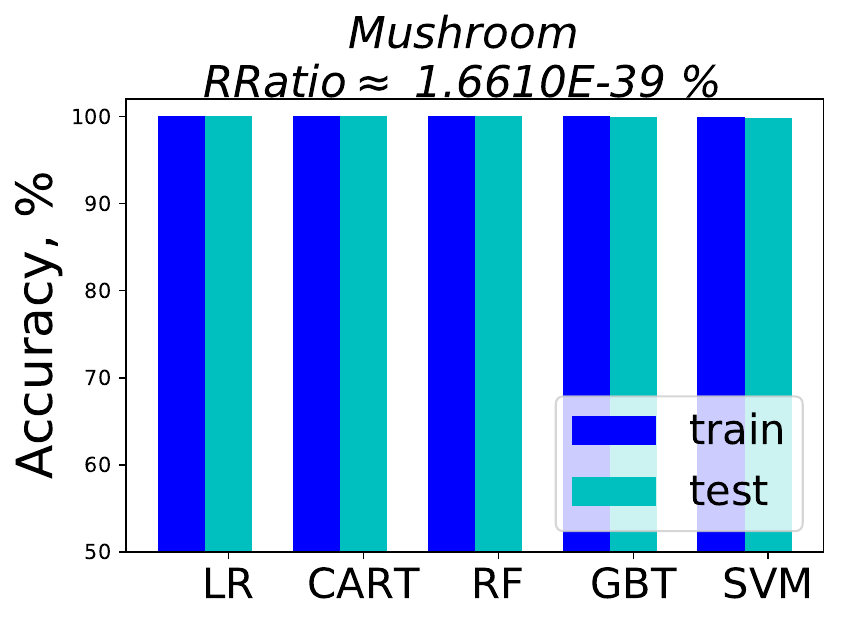}\quad
			\includegraphics[width=0.22\textwidth]{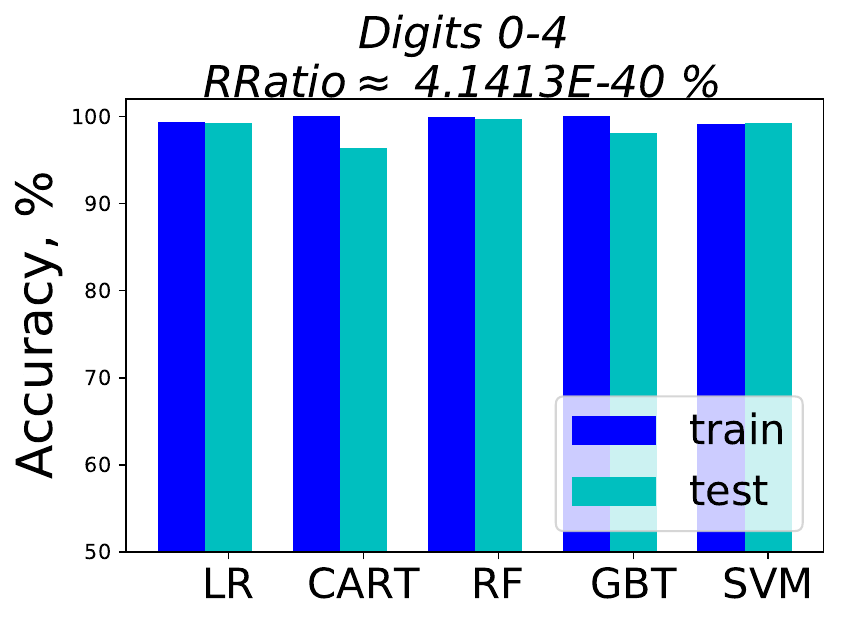}\quad
			\includegraphics[width=0.22\textwidth]{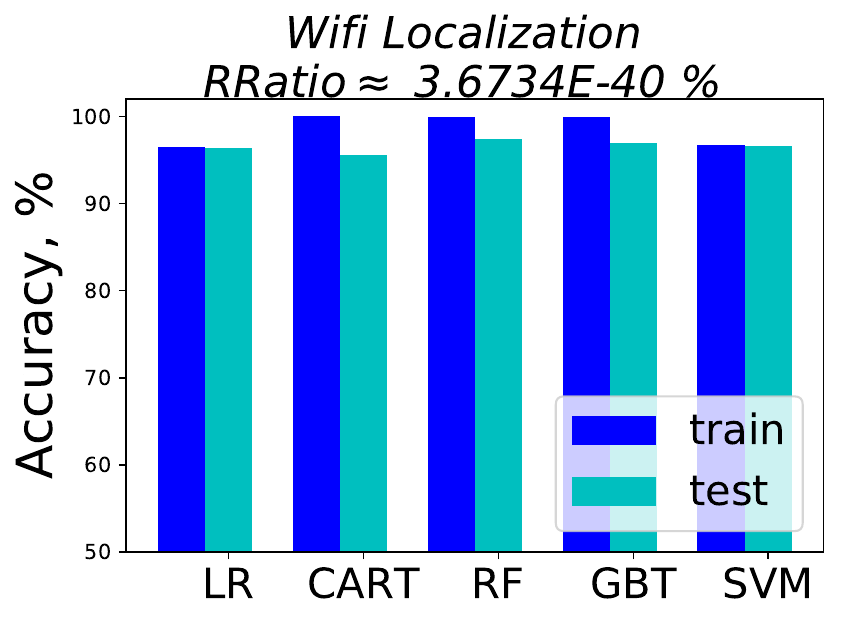}\quad
			\includegraphics[width=0.22\textwidth]{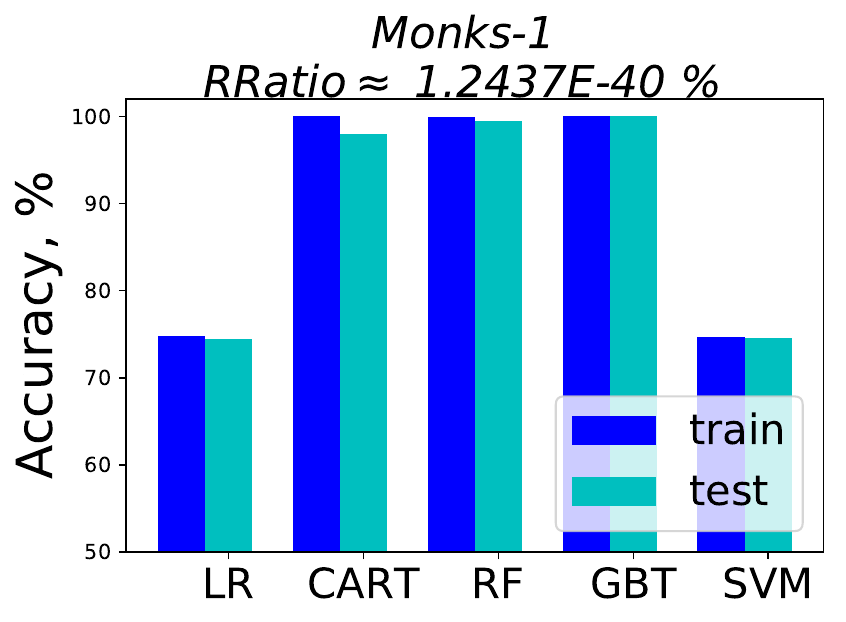}\\
			\includegraphics[width=0.22\textwidth]{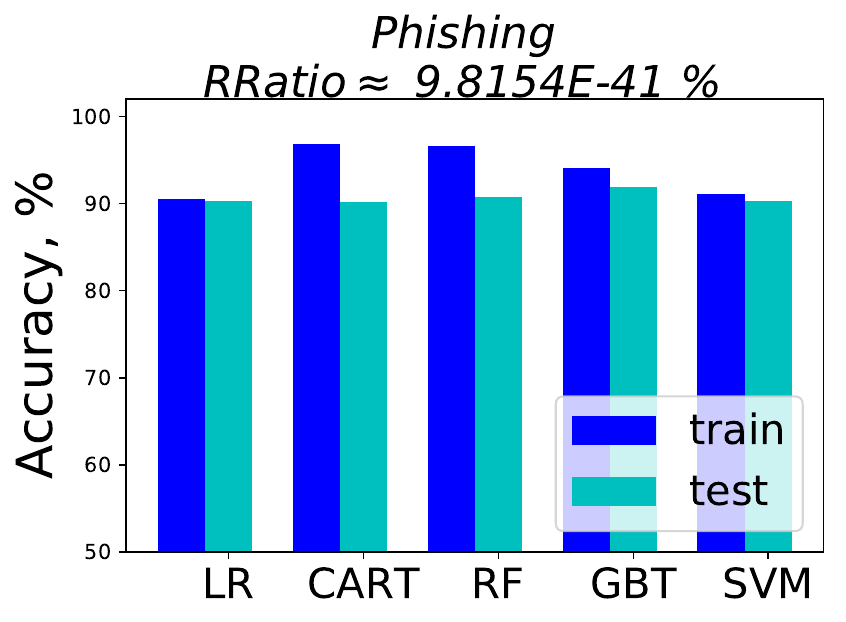}\quad
			\includegraphics[width=0.22\textwidth]{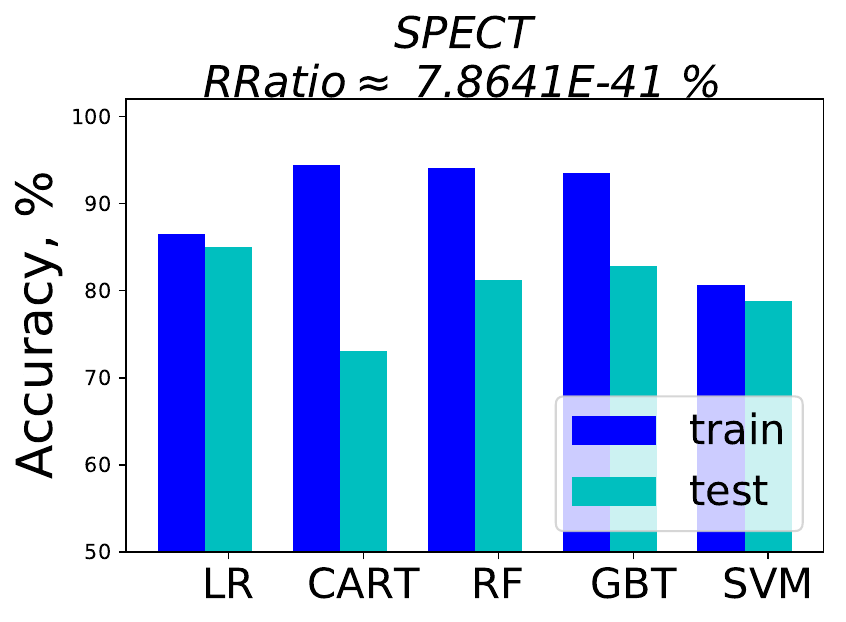}\quad
			\includegraphics[width=0.22\textwidth]{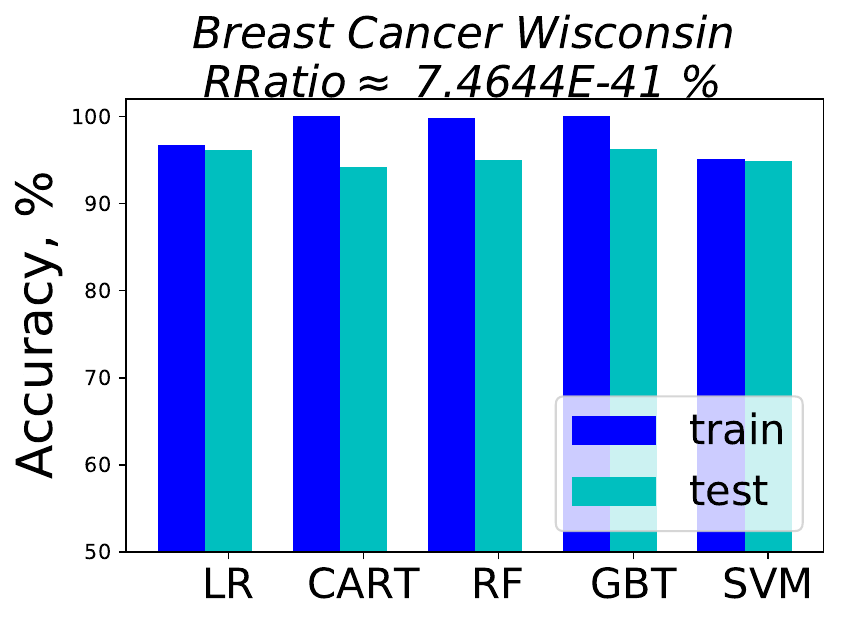}\quad
			\includegraphics[width=0.22\textwidth]{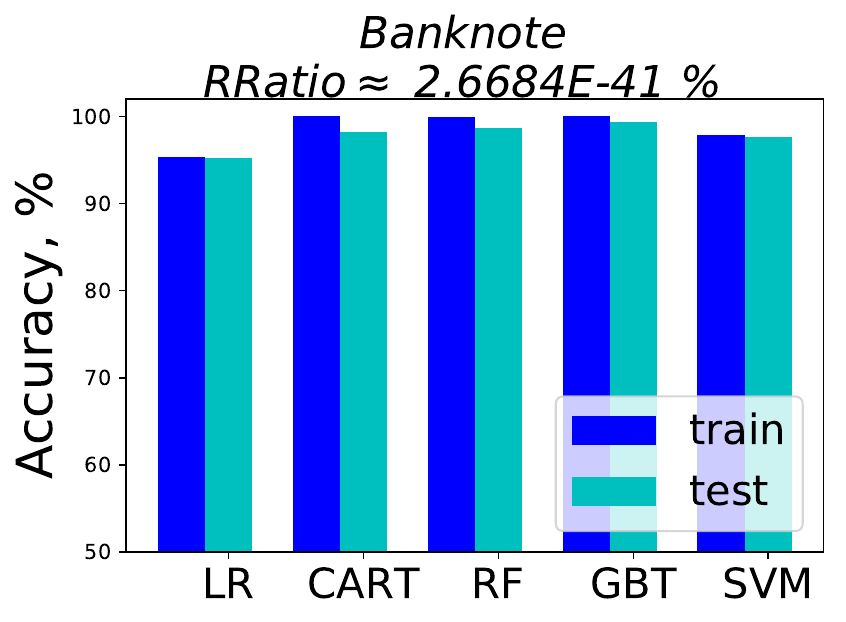}
			\\	
		\end{tabular}
		\caption{Performance of five machine learning algorithms without regularization for the UCI classification data sets. Data sets are listed in decreasing order of Rashomon ratio. Rashomon ratios, train and test accuracies are averaged over ten folds for data sets with more than 200 points and over five folds for data sets with less than 200 points. These plots continue in  Figure \ref{fig:bar_plots_nr_2}. The data sets with larger Rashomon ratios correlate with similar performance of machine learning algorithms and good generalization.} 
		\label{fig:bar_plots_nr_1}. 
	\end{figure*}
	
	\begin{figure*}[t]
		\centering
		\begin{tabular}{cc}
		    
			\includegraphics[width=0.22\textwidth]{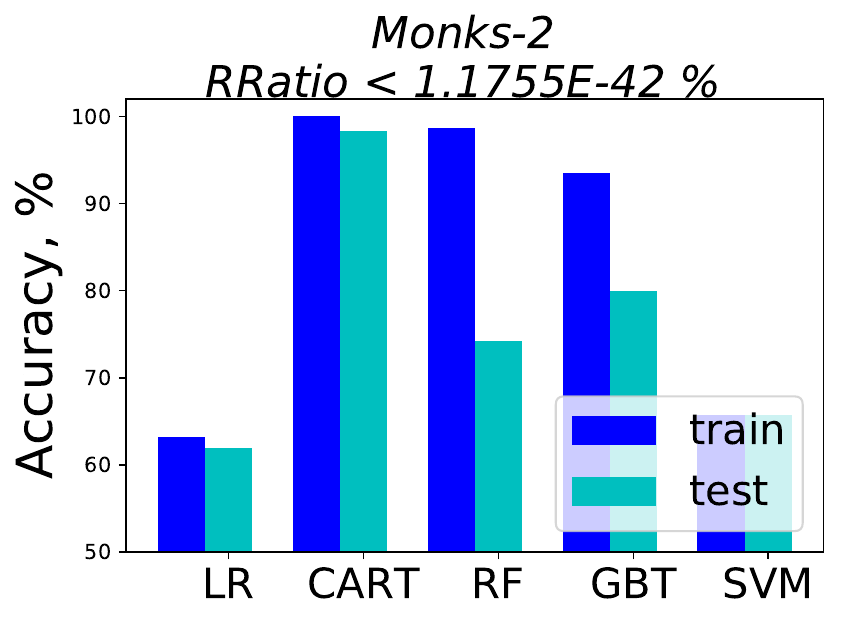}\quad	
			\includegraphics[width=0.22\textwidth]{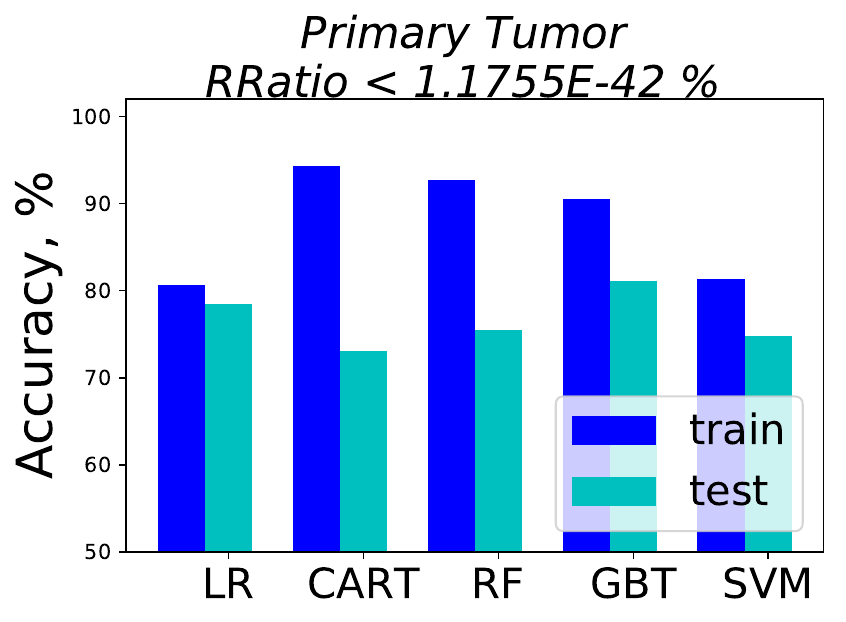}\quad
			\includegraphics[width=0.22\textwidth]{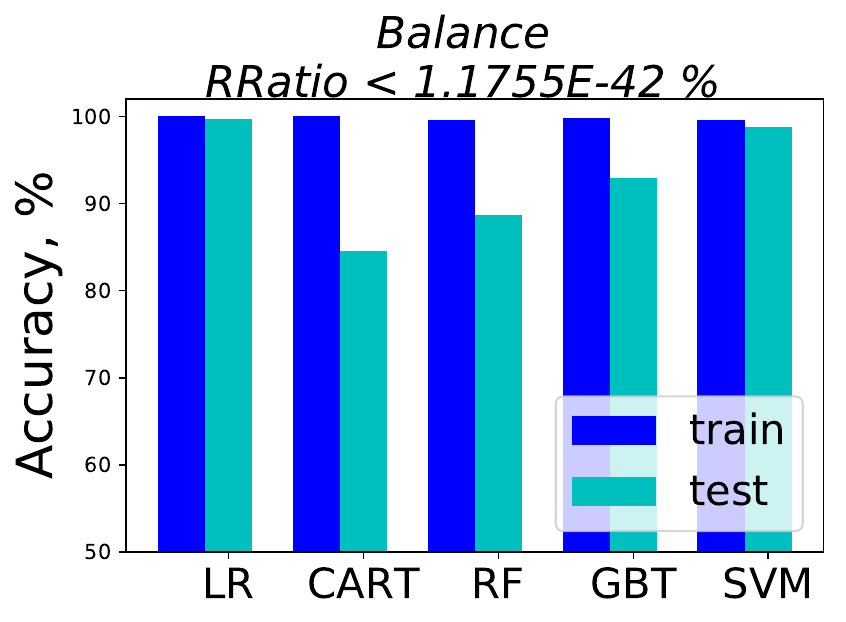}\quad	
			\includegraphics[width=0.22\textwidth]{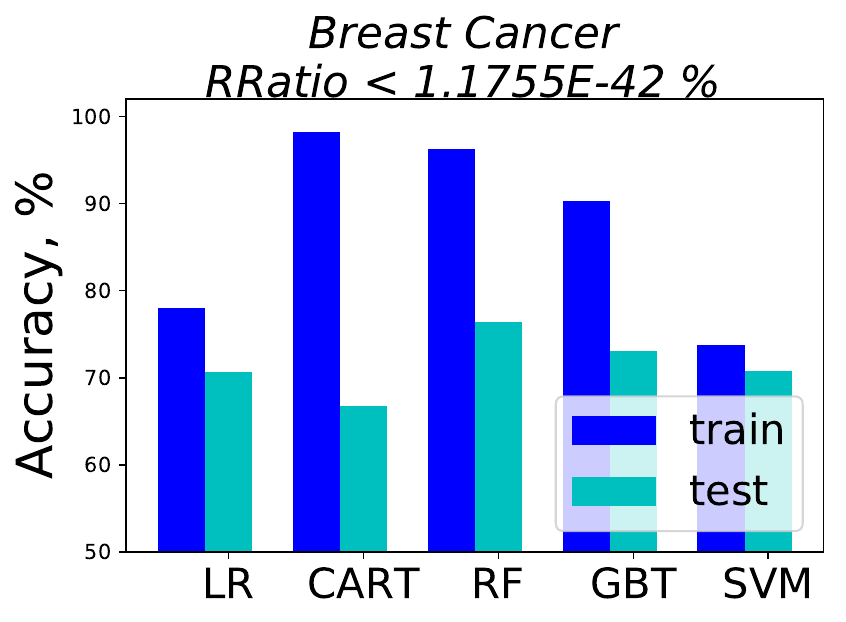}\\
			\includegraphics[width=0.22\textwidth]{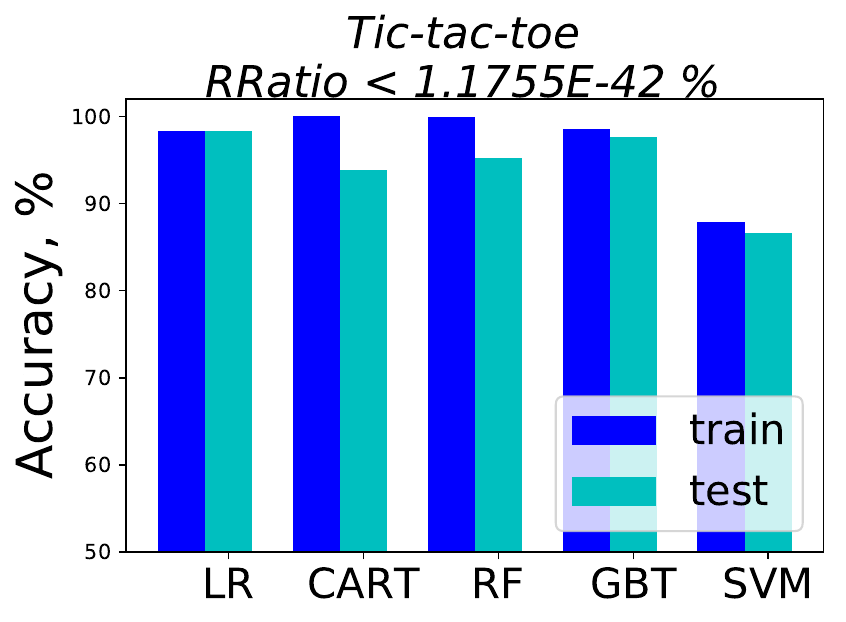}\quad
			\includegraphics[width=0.22\textwidth]{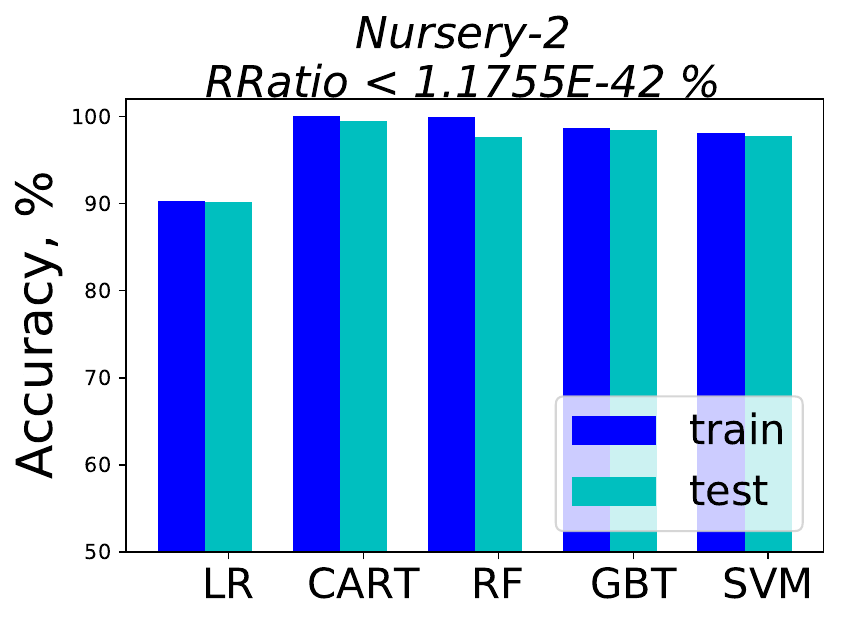}\quad	
			\includegraphics[width=0.22\textwidth]{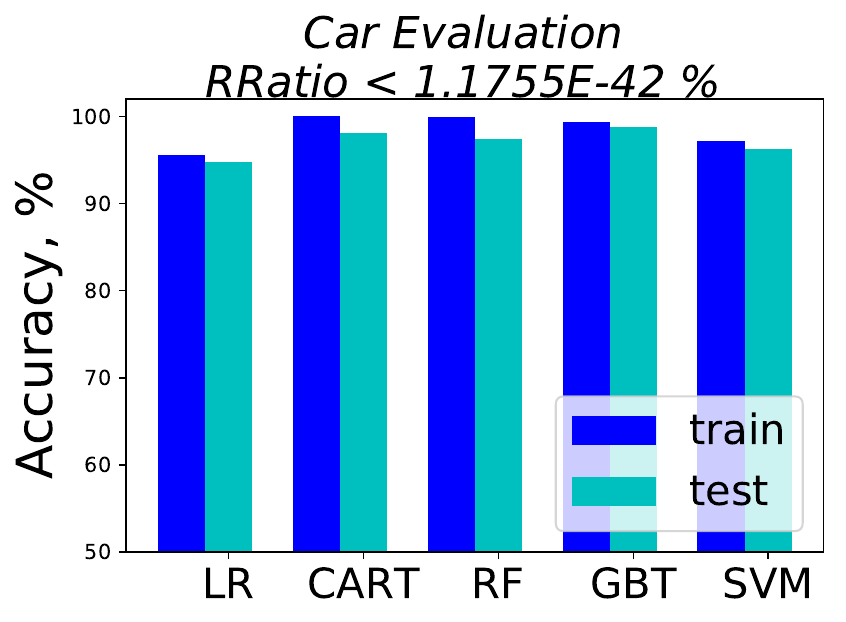}\quad
			\includegraphics[width=0.22\textwidth]{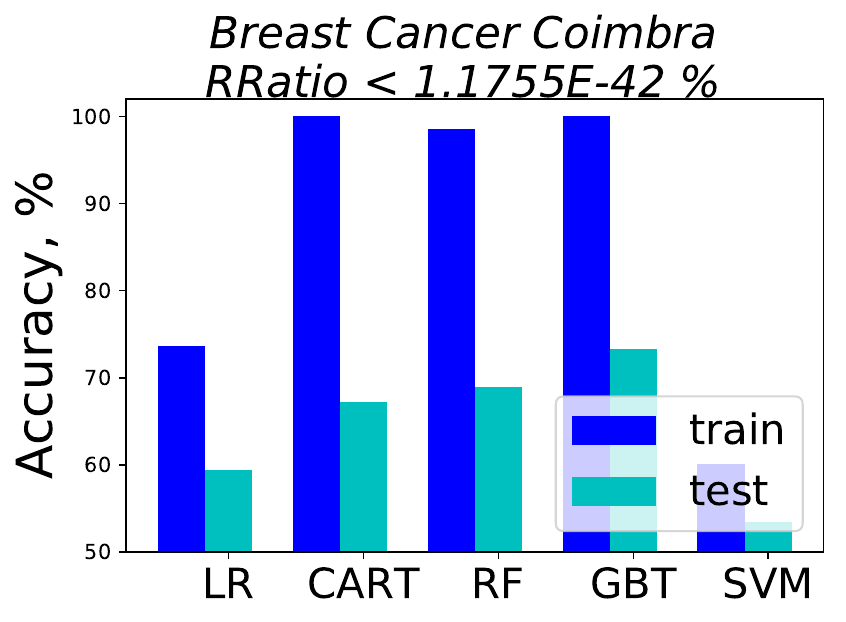}\\
			\includegraphics[width=0.22\textwidth]{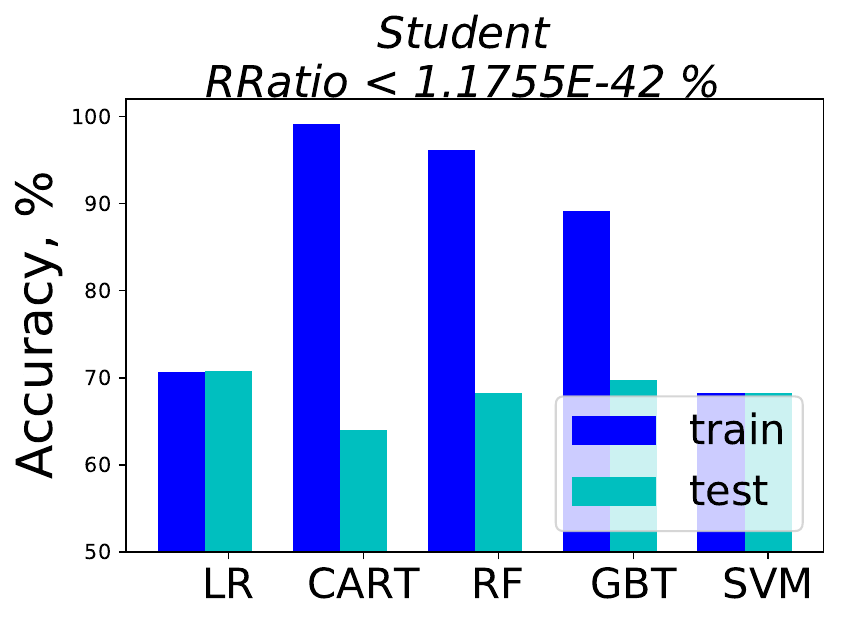}\quad
		    \includegraphics[width=0.22\textwidth]{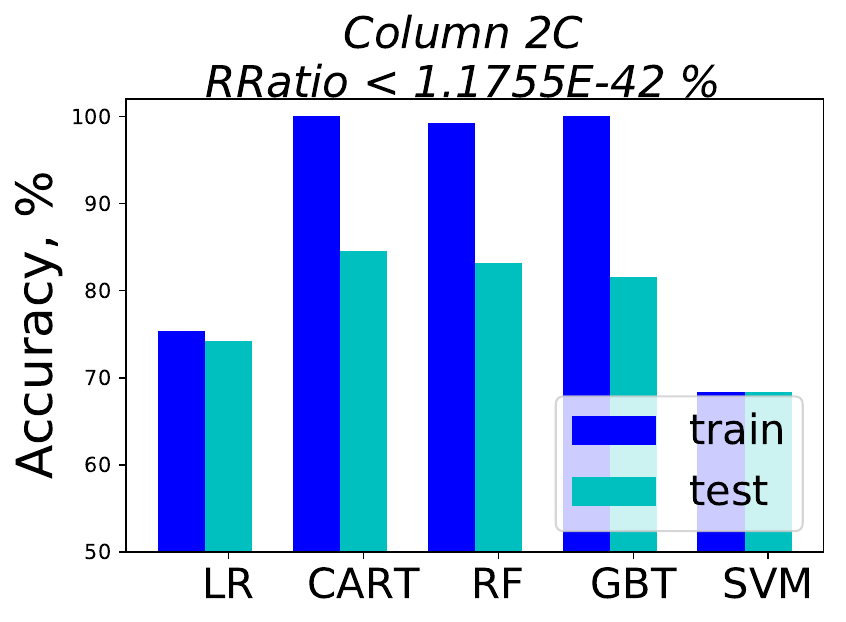}\quad	
			\includegraphics[width=0.22\textwidth]{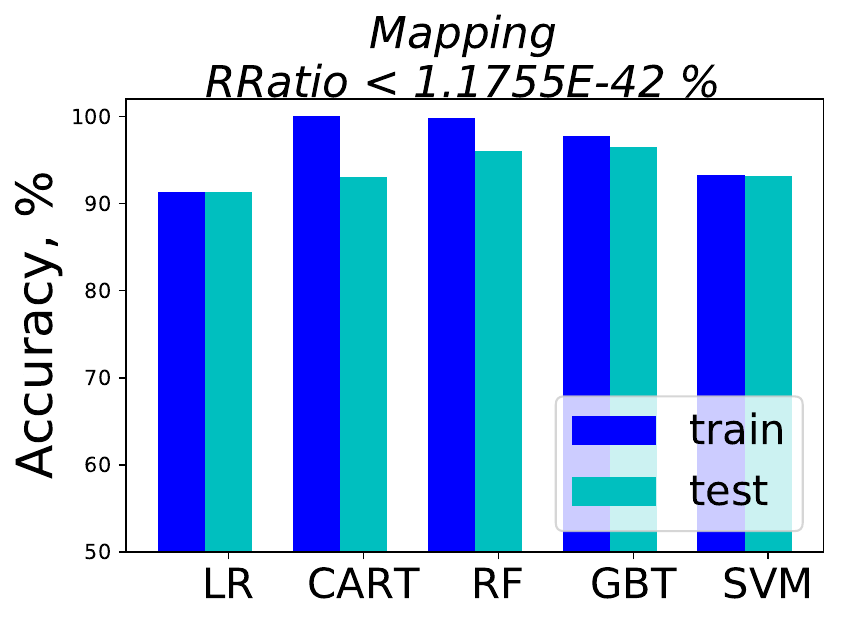}\quad
			\includegraphics[width=0.22\textwidth]{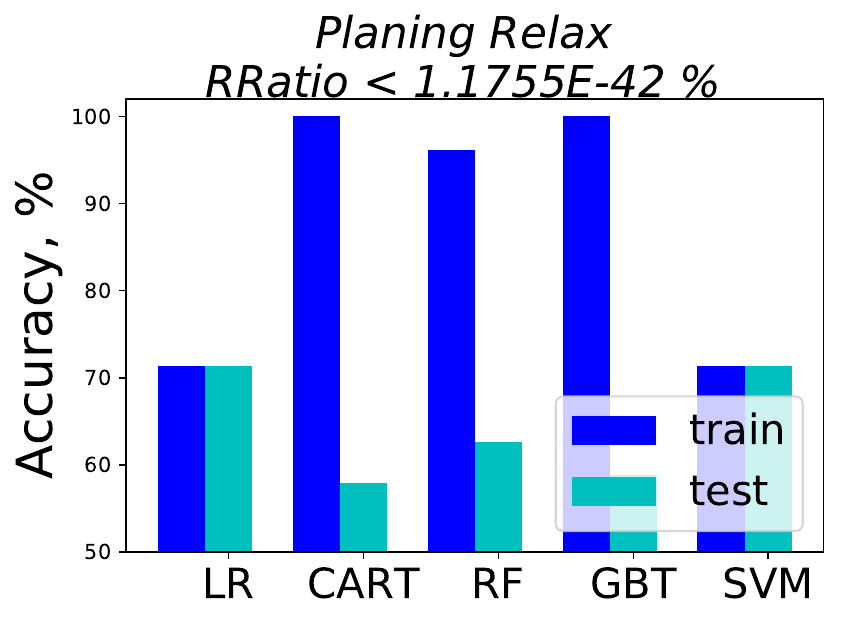}\\
			\includegraphics[width=0.22\textwidth]{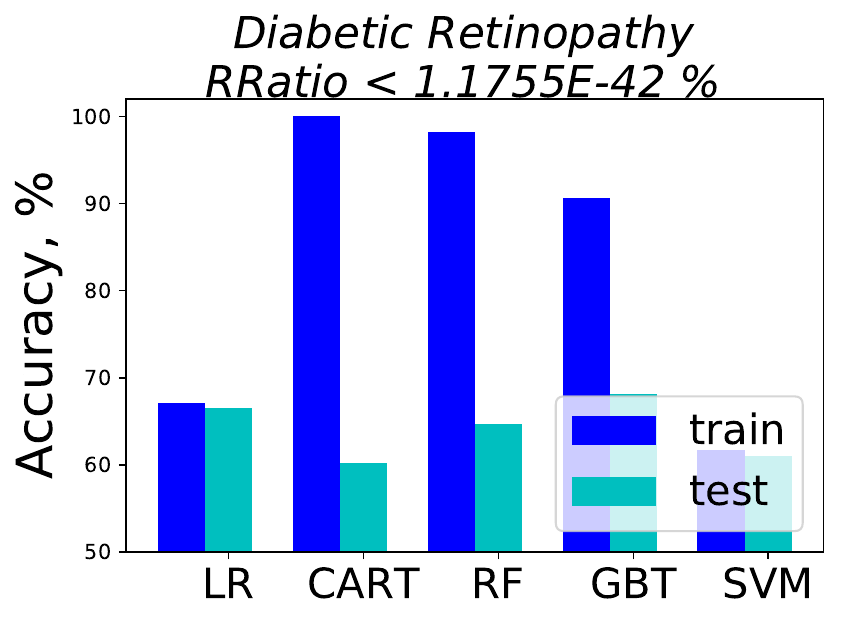}\quad
			\includegraphics[width=0.22\textwidth]{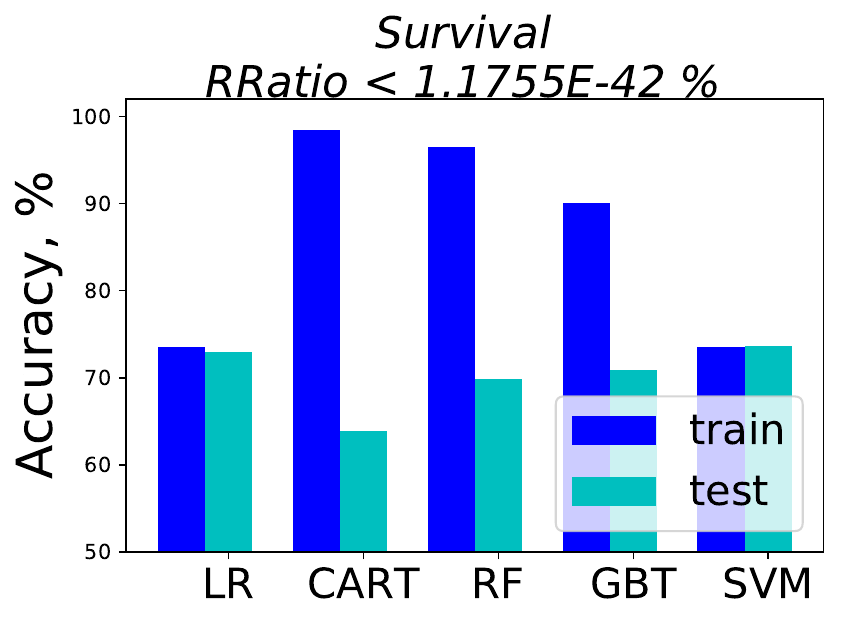}\quad	
			\includegraphics[width=0.22\textwidth]{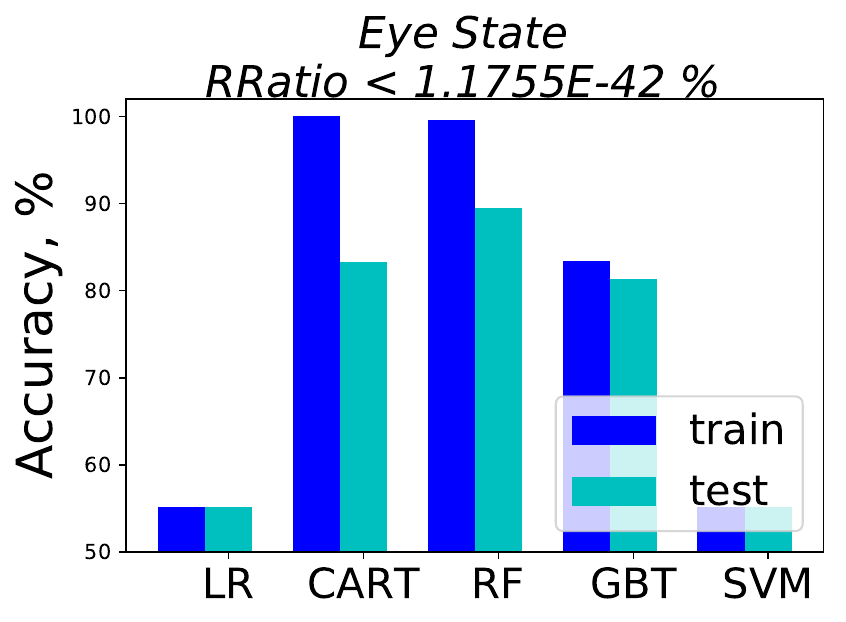}\quad
			\includegraphics[width=0.22\textwidth]{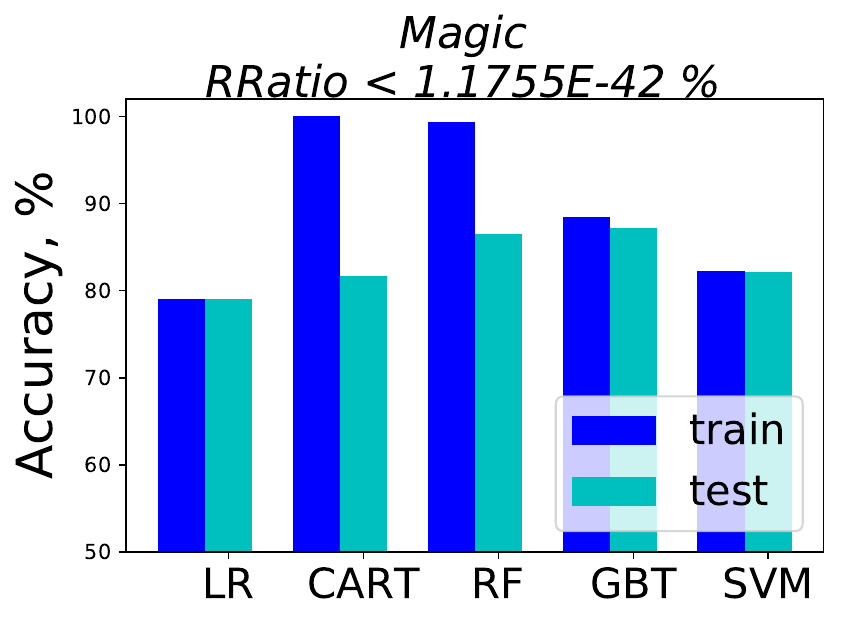}\\
			\includegraphics[width=0.22\textwidth]{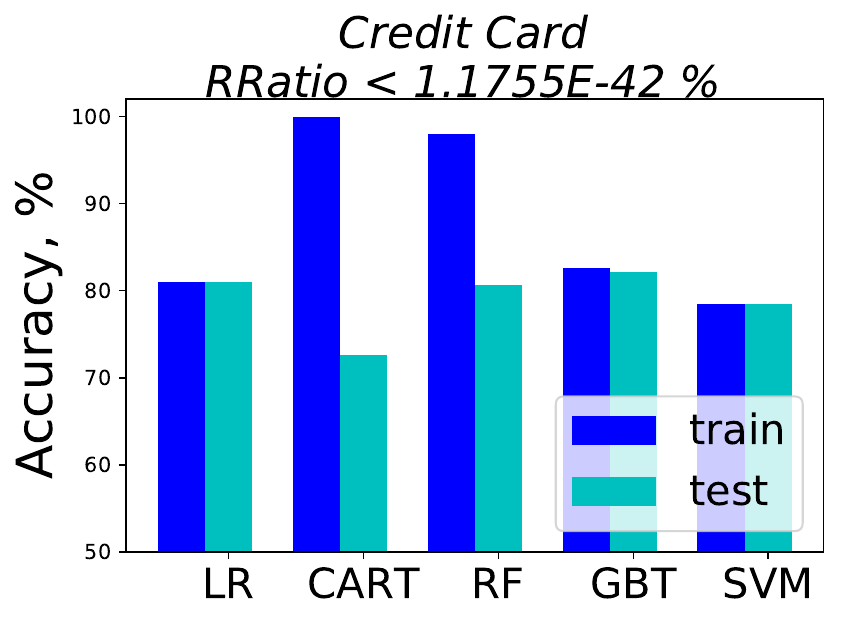}\quad	
			\includegraphics[width=0.22\textwidth]{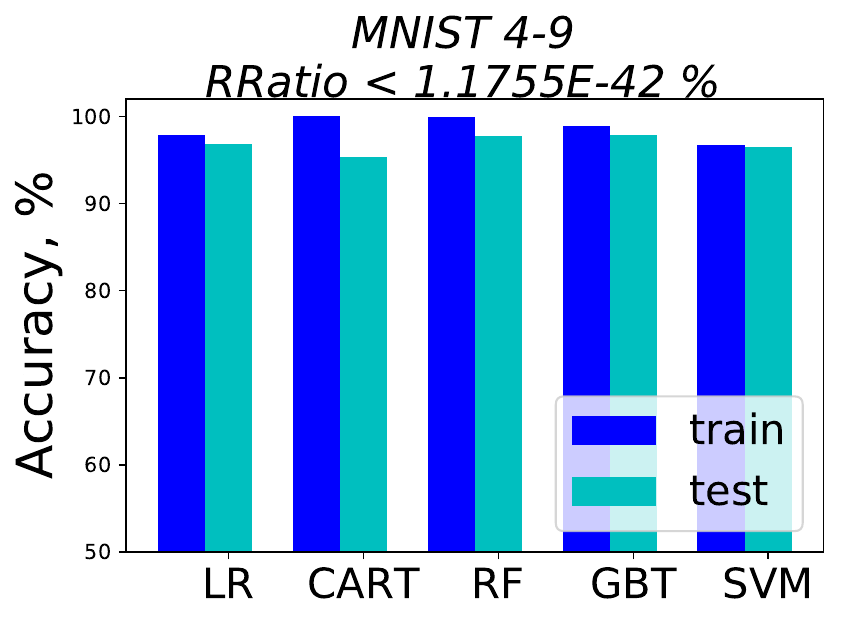}\\	
		\end{tabular}
		\caption{Performance of five machine learning algorithms without regularization for the UCI classification data sets. Data sets are listed in decreasing order of the Rashomon ratio continuing from  Figure \ref{fig:bar_plots_nr_1}. Rashomon ratios, train and test accuracies are averaged over ten folds for data sets with more than 200 points and over five folds for data sets with less than 200 points}.
		\label{fig:bar_plots_nr_2}
	\end{figure*}

		\begin{figure*}[t]
		\centering
		\begin{tabular}{cc}
			\includegraphics[width=0.22\textwidth]{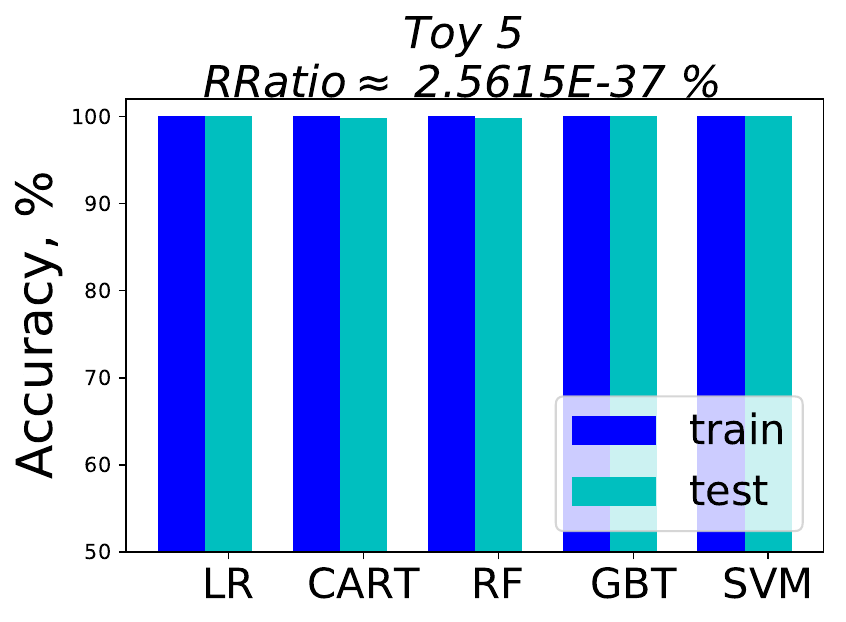}\quad	
			\includegraphics[width=0.22\textwidth]{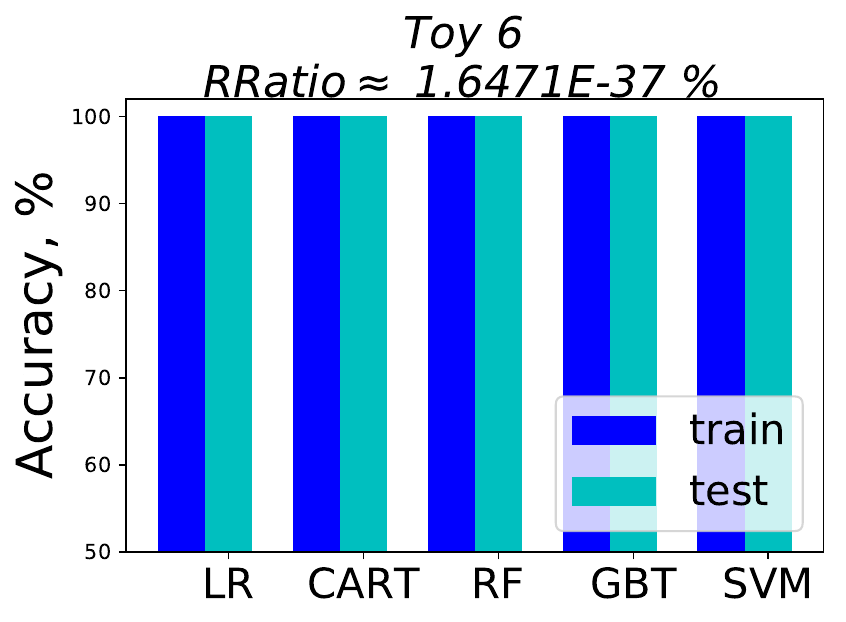}\quad
			\includegraphics[width=0.22\textwidth]{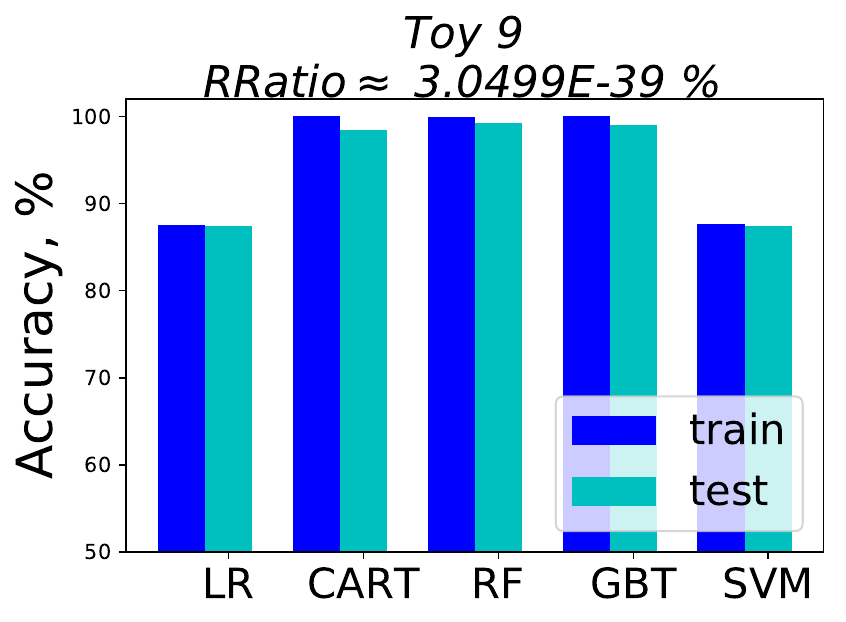}\quad
			\includegraphics[width=0.22\textwidth]{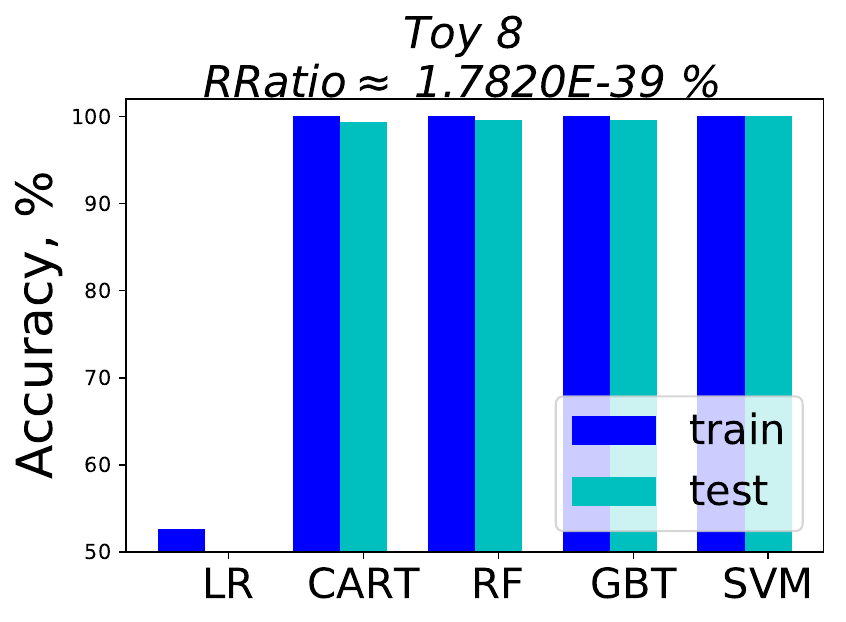}\\
			\includegraphics[width=0.22\textwidth]{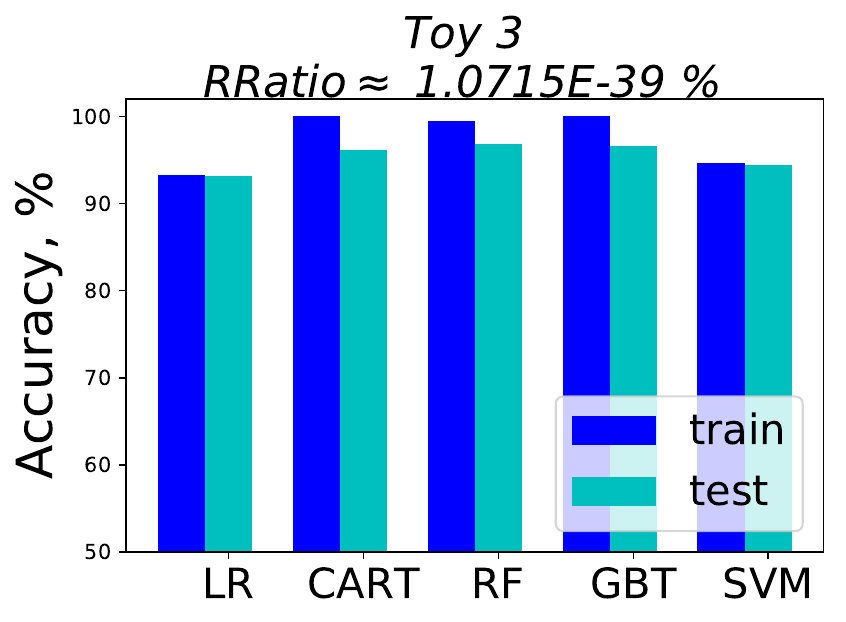}\quad	
			\includegraphics[width=0.22\textwidth]{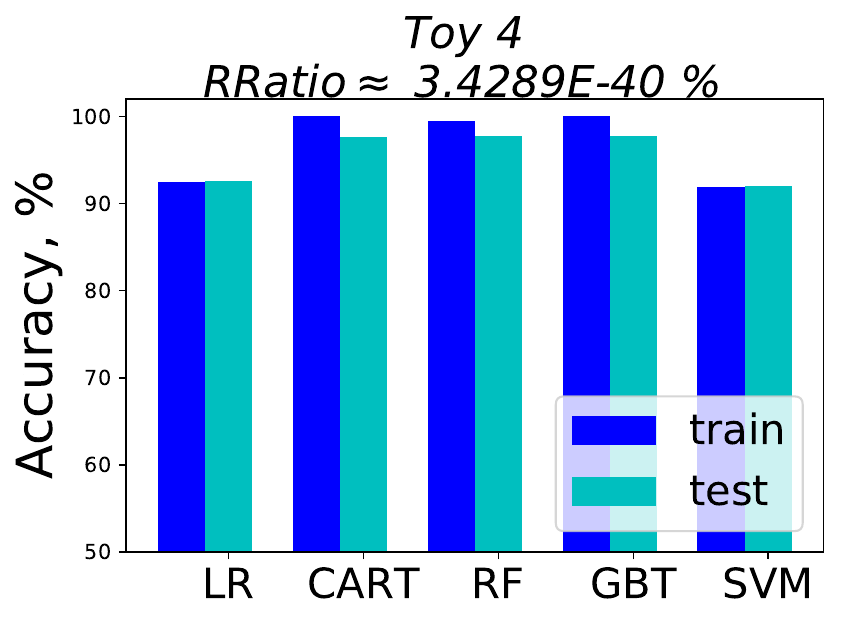}\quad
			\includegraphics[width=0.22\textwidth]{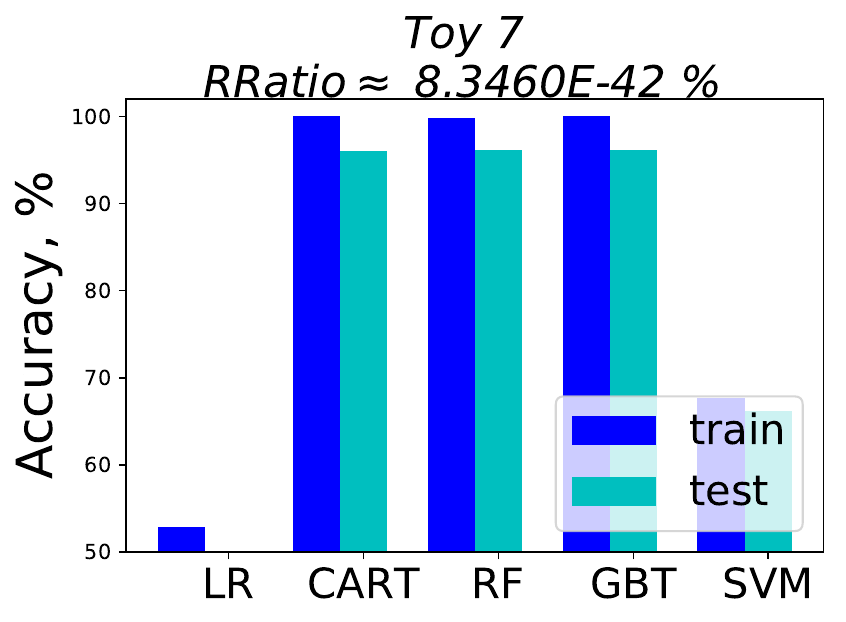}\quad
			\includegraphics[width=0.22\textwidth]{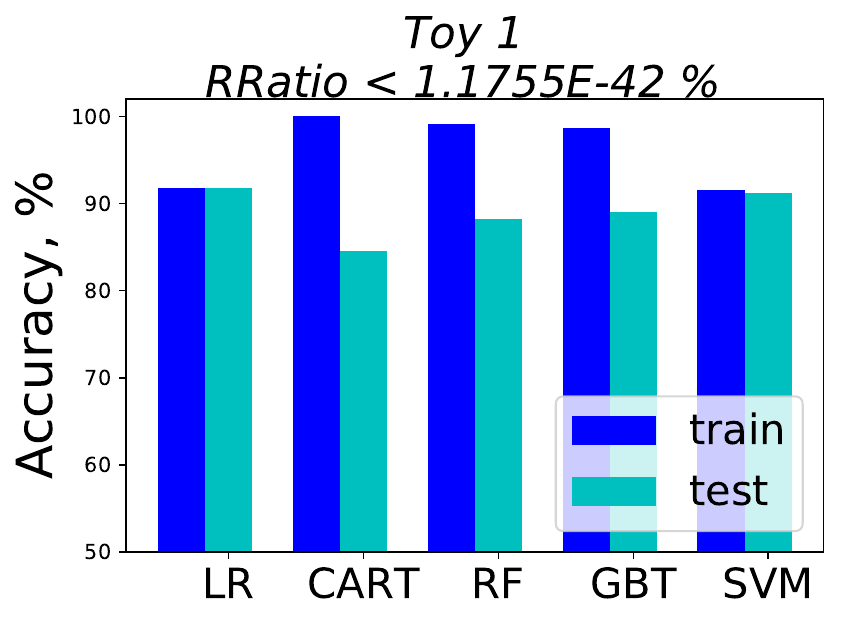}\\
		    \includegraphics[width=0.22\textwidth]{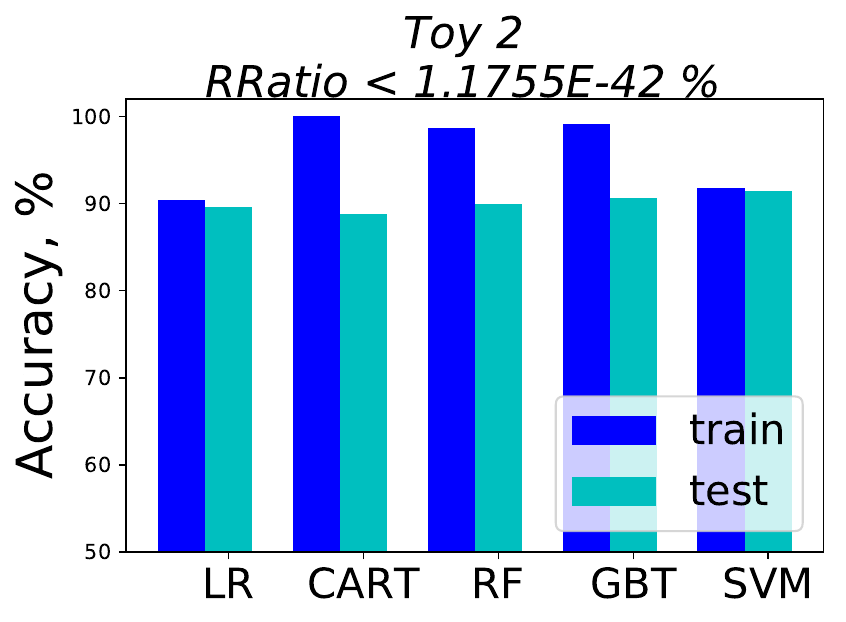}\quad	
			\includegraphics[width=0.22\textwidth]{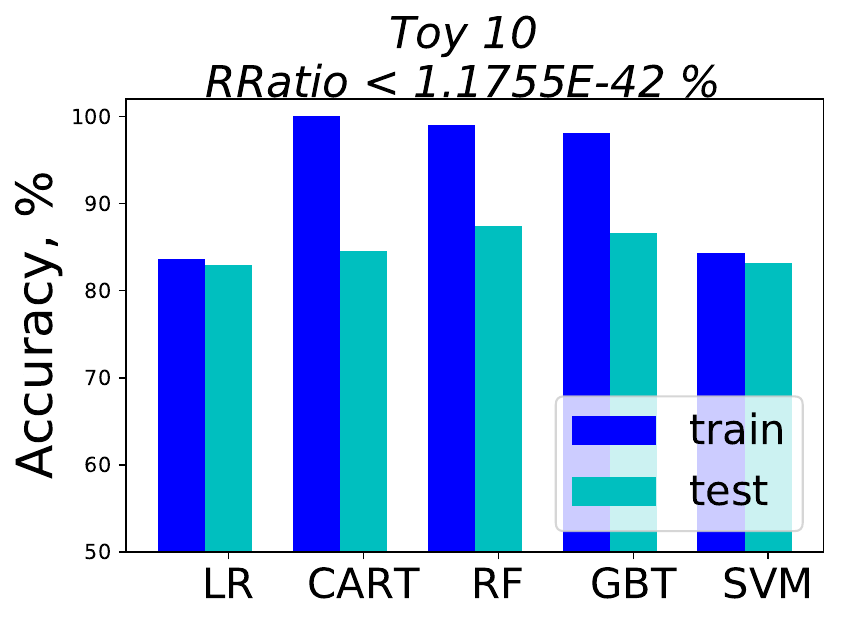}\quad
			\includegraphics[width=0.22\textwidth]{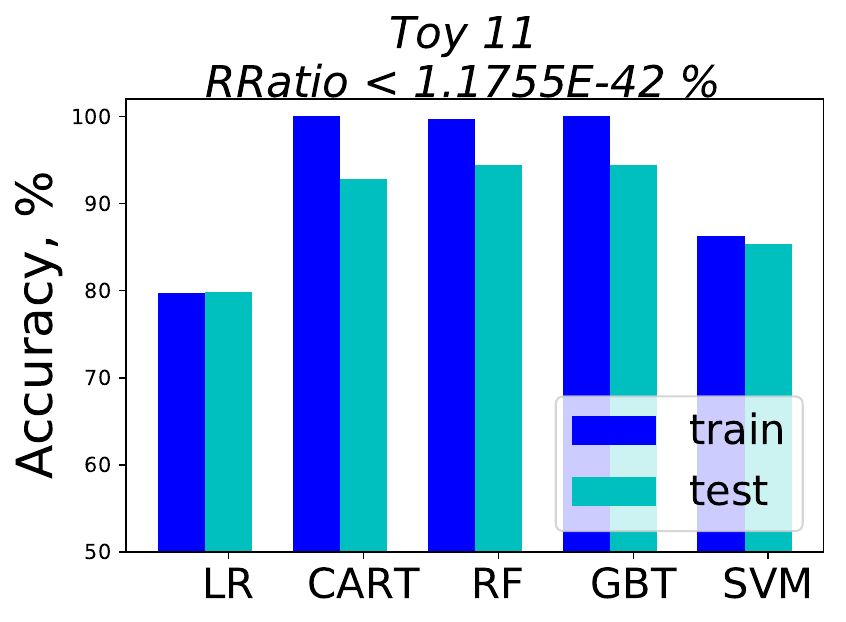}\quad
			\includegraphics[width=0.22\textwidth]{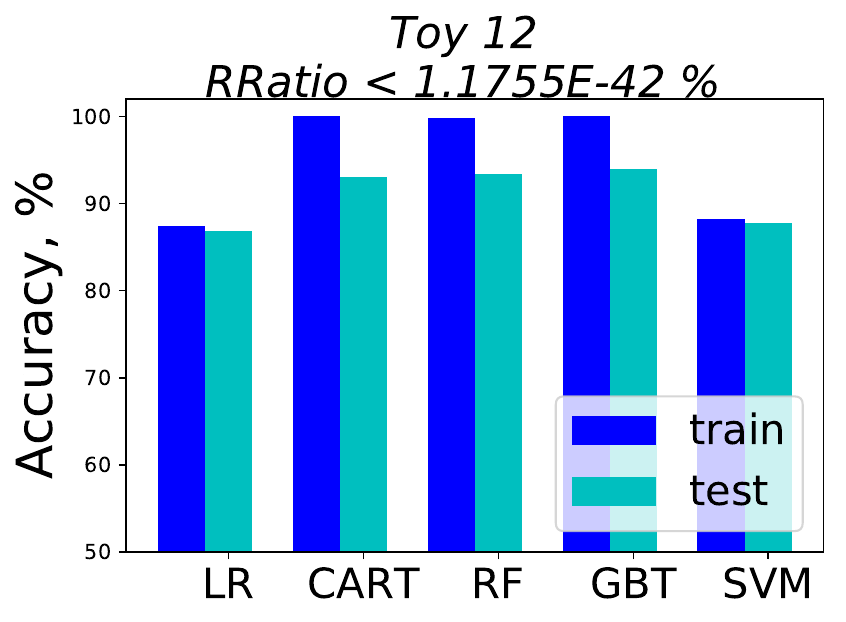}\\
		\end{tabular}
		\caption{Performance of five machine learning algorithms without regularization for the synthetic data sets with real-valued features. Data sets are listed in decreasing order of the Rashomon ratio. Rashomon ratios, train and test accuracies are averaged over ten folds. }
		\label{fig:bar_plots_nr_3}
	\end{figure*}

\end{document}